\documentclass[12pt]{UOthesis}

\usepackage{fancybox}
\usepackage{shadow}
\usepackage[all]{xy}
\usepackage{url}
\usepackage{rotating}
\usepackage{multicol}

\allowdisplaybreaks

\NoResume
\NoListofsymbols
\setUOname{Hubert Haoyang Duan}         % Name of the cnadidat
\setUOcpryear{2013}               % Year of graduation
\setUOtitle{Applying Supervised Learning Algorithms and a New Feature Selection Method to Predict Coronary Artery Disease}  % Title of the thesis

\msc               % Comment out for the M.Sc.

\setUOabstract{
From a fresh data science perspective, this thesis discusses the prediction of coronary artery disease based on genetic variations at the DNA base pair level, called Single-Nucleotide Polymorphisms (SNPs), collected from the Ontario Heart Genomics Study (OHGS). 

First, the thesis explains two commonly used supervised learning algorithms, the $k$-Nearest Neighbour ($k$-NN) and Random Forest classifiers, and includes a complete proof that the $k$-NN classifier is universally consistent in any finite dimensional normed vector space. Second, the thesis introduces two dimensionality reduction steps, Random Projections, a known feature extraction technique based on the Johnson-Lindenstrauss lemma, and a new method termed Mass Transportation Distance (MTD) Feature Selection for discrete domains. Then, this thesis compares the performance of Random Projections with the $k$-NN classifier against MTD Feature Selection and Random Forest, for predicting artery disease based on accuracy, the F-Measure, and area under the 
Receiver Operating Characteristic (ROC) curve.

The comparative results demonstrate that MTD Feature Selection with Random Forest is vastly superior to Random Projections and $k$-NN. The Random Forest classifier is able to obtain an accuracy of 0.6660 and an area under the ROC curve of 0.8562 on the OHGS genetic dataset, when 3335 SNPs are selected by MTD Feature Selection for classification. This area is considerably better than the previous high score of 0.608 obtained by Davies et al. in 2010 on the same dataset.

}

\setUOresume{
  Vous pouvez introduire le r\'esum\'e de votre th\`ese ici.
}

\setUOthanks{I would like to thank my supervisor Dr. Vladimir Pestov for all his guidance, excellent advice, and support throughout my years of undergraduate and Master's studies. I am extremely grateful to Dr. Pestov for first introducing me to machine learning in 2009, and I look forward to many more years of collaboration with him.

I would also like to thank my co-supervisor Dr. George Wells for many helpful discussions and suggestions in genetics and statistics.

I would like to express my gratitude for the NSERC Canada Graduate Scholarship and the Ontario Graduate Scholarship, and to the Department of Mathematics and Statistics, which all provided me with financial support. 

I would like to also thank Ga\"{e}l Giordano, for all his help and for all of our useful discussions, Varun Singla, Stan Hatko, and the rest of the University of Ottawa Data Science Group.

Finally, I would like to thank my wonderful parents for their heart-felt patience and encouragement.
}

\setUOdedicationsText{Dedicated to my late mom, Lin Mu (1958 - 2012).}

{
  \newtheoremstyle{remarkstyle}{\topsep}{\topsep}{\rm}{}{\bfseries}{.}{.5em}{}
  \theoremstyle{remarkstyle}

}

\begin{document}

%%%%%%%%%%%%%%%%%%%%%%%%%%%%%%%%%%%%%%%%%%%%%%%%%%%%%%%%%%%%%%%%%%%%%%
% We suggest to include your chapters, ... as illustrated below      %
% and leave this file with as Little of the contents of the thesis   %
% as it is possible.                                                 %
%%%%%%%%%%%%%%%%%%%%%%%%%%%%%%%%%%%%%%%%%%%%%%%%%%%%%%%%%%%%%%%%%%%%%%

\nonumchapter{Introduction}

{\em Data science} is a new, exciting, and interdisciplinary subject, which lies in the intersection of mathematics, statistics, and computer science. Its major focus is to extract valuable information from large and often high-dimensional datasets \cite{datascience}. Today, modern technology allows for the collection of large amounts of data, regularly in the magnitudes of terabytes and petabytes; for instance, the NASA Center for Climate Simulation stores over 30 petabytes of climate information on a series of supercomputers \cite{websitenasa}. These datasets, known as {\em big data}, have such immense size and complexity that traditional data analysis techniques cannot be applied  to them efficiently \cite{bigdata}. {\em Data scientists}, practitioners of data science, develop innovative computer algorithms, which are both effective and efficient, often through parallel and cloud computing, to study these big data \cite{datascientist}.

A subfield of data science is {\em supervised learning theory}, which formalizes the algorithmic notion of learning and building predictions from observed data. A classical example is detecting email spams; powerful supervised learning algorithms aim to accurately detect whether a newly received email is spam through the process of training from large amounts of past emails labeled as either spam or not spam \cite{guzella2009review}. In general, a labeled dataset with a specific predictive goal (e.g. predict email spam) is given, where each observation from the dataset comes equipped with a possible binary-valued label corresponding to the predictive goal (e.g. either ``spam" or ``not spam"). A learning algorithm, or classifier, would train on this information to predict the label for any new observation \cite{machinelearn}.  

Mathematically, as presented in \cite{devroyeprob} for example, a labeled dataset, also known as a training sample or observed data, is denoted as
\[
\mathcal{S}_\mathrm{lab} = \{(X_1,Y_1),(X_2,Y_2),\ldots,(X_n,Y_n)\},
\]
where the $i$'th observation $X_i\in\Omega$ is in some domain (or feature) space with $m$ dimensions (also known as features or coordinates), such as $\Omega = \mathbb{R}^m$, and $Y_i\in \{0,1\}$ is its label. A learning algorithm is then a function $g$ which maps a new observation $X\in \Omega$ to a label $g(X)\in\{0,1\}$, given the training information:
\begin{align*}
g: (\Omega \times \{0,1\})^n \times \Omega &\to \{0,1\}\\
\left(\mathcal{S}_\mathrm{lab}, X\right) &\longmapsto g(X).
\end{align*}

One of the major challenges for supervised learning theory is that big data often have high dimensions, and computational costs of running learning algorithms can be prohibitively expensive. Hence, techniques for dimensionality reduction, which map a labeled dataset onto a lower dimensional space, would have to be applied to simplify data complexity and allow these algorithms to run efficiently on the transformed data. Some commonly known learning algorithms in literature are the $k$-Nearest Neighbour classifier, Support Vector Machines (SVM), and Decision Trees. Widely used dimensionality reduction techniques include Principal Component Analysis (PCA), Linear Discriminant Analysis (LDA), and feature selection based on some measure of importance \cite{dmreview}.

\subsubsection{Thesis Objective}

The two main goals of this thesis are to formally introduce certain supervised learning algorithms and dimensionality reduction techniques, and to compare their applications on a high-dimensional genetic dataset for predicting coronary artery disease (CAD). CAD is a common type of cardiovascular disease that occurs when substances clog a heart's arteries, and severe cases can often lead to heart attacks \cite{cadweb}. It is a well-known fact that genetic variations play a role in the prevalence of CAD among individuals; in fact, studies have repeatedly shown that these variations account for approximately 40\% to 60\% of the risk for CAD \cite{cadroberts}. Consequently, the prediction of this disease using individuals' possible genetic variations at the DNA base pair level, called Single-Nucleotide Polymorphisms (SNPs), is an important supervised learning problem. Successful solutions can lead to more accurate diagnosis of CAD and better understanding of the genetics behind CAD. Since there are an immense number of base pairs with possible genetic variations among individuals, datasets containing SNP information are extremely high-dimensional and thus, dimensionality reduction steps have to be performed prior to running classification algorithms.

This thesis first explains in detail two widely used learning algorithms in literature, the $k$-Nearest Neighbour ($k$-NN) classifier and the Random Forest classifier, and two techniques for dimensionality reduction, one known method named Random Projections and one novel method based on the theory of Mass Transportation Distance (MTD), introduced in this thesis for the first time as {\em MTD Feature Selection}.

The $k$-NN classifier is one of the oldest and most recognized supervised learning algorithms in data science, and it is based on finding $k$ nearest neighbours in a metric space $\Omega = (\Omega, d)$ \cite{knnorig}. Given a training sample $\{(X_1,Y_1),(X_2,Y_2),\ldots,(X_n,Y_n)\}$ and a new observation $X\in\Omega$, this classifier arranges the training observations in increasing distance from $X$,
\[
d(X_{(1)},Y_{(1)}) \leq d(X_{(2)},Y_{(2)}) \leq \ldots \leq d(X_{(n)},Y_{(n)}),
\]
and the predicted label for $X$ is based on the majority vote of its $k$ closest neighbours: $\mathrm{label}\, (X) = \mathrm{mode}\{Y_{(1)},Y_{(2)},\ldots,Y_{(k)}\}$. The algorithm has the important theoretical property of being {\em universally consistent} when $\Omega = (\mathbb{R}^m,||\cdot||_2)$ is the finite dimensional Euclidean space \cite{stone77}. In a rough sense, this property signifies that as the number $n$ of training observations grows arbitrarily large, the predictive ability of $k$-NN will become approximately optimal.

The Random Forest classifier is a supervised learning algorithm that generalizes Decision Tree, a classifier which partitions a feature space into hyper-rectangles based on the training sample. Each rectangle corresponds to a classification label, and the learning algorithm predicts the label for a new observation according to the hyper-rectangle it belongs to, see e.g. \cite{cart}. Random Forest generalizes this classifier by considering multiple Decision Trees constructed from bootstrap samples of the training set, and predicts the label for a new observation according to the majority vote of the multiple Decision Trees \cite{rfbreiman}. 
Unlike $k$-NN however, the Random Forest learning algorithm is, in general, not universally consistent \cite{rfconsistent}, but it has excellent predictive abilities and is widely used in practice.

Random Projections is a dimensionality reduction technique for the Euclidean space $\Omega = (\mathbb{R}^m,||\cdot||_2)$, developed from the Johnson-Lindenstrauss Lemma:
\begin{theorem*}[Johnson-Lindenstrauss Lemma \cite{jllemmaorig}]\label{origjl}
Let $0<\epsilon<1/2$ and let $S\subseteq \mathbb{R}^m$ be any finite subset with $|S| = n$. If $m' \geq C\ln(n)/\epsilon^2$ is any integer, for some sufficiently large absolute constant $C$, there exists a linear map $T:\mathbb{R}^m\to\mathbb{R}^{m'}$ such that
\[
(1-\epsilon)||x - y||_2 \leq ||T(x) - T(y)||_2 \leq (1+\epsilon)||x - y||_2
\]
for all $x,y\in S$.
\end{theorem*}
\noindent The linear map $T$ can be chosen as matrix multiplication by a randomly generated matrix $(R_{ij})$, where $R_{ij}$ is a binary random variable taking values in $\{-1,1\}$ with equal probabilities \cite{variantjl}. In the field of supervised learning, if the domain is extremely high-dimensional, observations from a training sample, along with any new observations to be classified, can be projected to a lower dimensional space via matrix multiplication by $(R_{ij})$. A distance-based classifier, such as $k$-NN, can then be efficiently applied to the projected observations since all pairwise distances are preserved up to a factor of $1\pm \epsilon$, as guaranteed by the lemma.

The second dimensionality reduction method is new and based on the Mass Transportation Distance, which is defined on the space of finitely supported probability measures on a metric space $\Omega = (\Omega, d)$, see e.g \cite{mtdpaper} or \cite{supestov}. For predictive problems in learning theory with possible labels $\{0,1\}$, two probability measures $\mu_{0},\mu_{1}$, which model respectively the theoretical distributions of observations with labels $0$ and $1$, are natural to consider. The Mass Transportation Distance between $\mu_{0}$ and $\mu_{1}$ is defined to be
\[
{\hat d}(\mu_{0},\mu_{1}) = \inf_{\nu}\int_{\Omega\times \Omega} d(x,y) \,d\nu(x,y),
\]
where the infimum is taken over all probability measures $\nu$ on $\Omega\times \Omega$ such that the marginals of $\nu$ are $\mu_{0}$ and $\mu_{1}$ respectively. This distance provides a measure of separation between the two classes of observations; however, the infimum and integral are extremely difficult to compute exactly. In the specific instance that $(\Omega,d)$ is a finite and discrete metric space, where $d$ is the $\{0,1\}$-distance: $d(x,y) = 0$ if $x = y$ and $d(x,y) = 1$ otherwise, the Mass Transportation Distance ${\hat d}$ simplifies to the $\ell_1$ distance:
\[
{\hat d}(\mu_{0},\mu_{1}) = ||\mu_{0} - \mu_{1}||_1.
\]
Consequently, this easily computable distance can be used for dimensionality reduction of high-dimensional datasets with distances taking values in $\{0,1\}$. Let $\{(X_1,Y_1),(X_2,Y_2),\ldots,(X_n,Y_n)\}$, where $X_i = (X_{i1},X_{i2},\ldots,X_{im})$ and $Y_i\in\{0,1\}$, be a training sample such that for any fixed coordinate $j$, $X_{ij} \in \Omega$ for all $1\leq i \leq n$, and $(\Omega,d)$ is a finite metric space with the $\{0,1\}$-distance.
 Corresponding to the two labels, two empirical probability measures ${\hat \mu_{0}^j}$ and ${\hat \mu_{1}^j}$ can be defined on $\Omega$:
\[
{\hat \mu_{0}^j} = \sum_{\omega\in\Omega}p_0^j(\omega)\delta_\omega \quad \textnormal{and} \quad {\hat \mu_{1}^j} = \sum_{\omega\in\Omega}p_1^j(\omega)\delta_\omega,
\]
where $\delta_\omega$ is the Dirac measure, and $p_0^j(\omega)$ and $p_1^j(\omega)$ denote the relative occurrences of $\omega\in\Omega$, at the $j$'th coordinate, over all observations with labels $0$ and $1$, respectively:
\begin{align*}
p_0^j(\omega) &= \mathrm{card}\{i \,| \,(X_{ij}, Y_i) = (\omega, 0)\}/\mathrm{card}\{i \, | \, Y_i = 0\}\\
p_1^j(\omega) &= \mathrm{card}\{i \,| \,(X_{ij}, Y_i) = (\omega, 1)\}/\mathrm{card}\{i \, | \, Y_i = 1\}.
\end{align*}
Then, the Mass Transportation Distance
\begin{align*}
{\hat d}({\hat \mu_{0}^j},{\hat \mu_{1}^j}) = ||{\hat \mu_{0}^j} - {\hat \mu_{1}^j}||_1 =\sum_{\omega \in \Omega} |p_0^j(\omega) - p_1^j(\omega)|
\end{align*}
measures, in a heuristic sense, the degree of separation between the two classes with respect to the $j$'th coordinate. As a result, to reduce dimension, only coordinates $j$ with high Mass Transportation Distances should be considered, as the two class distributions would be most distinguishable at those coordinates, for classification purposes.

Following detailed explanations of these algorithms and techniques, this thesis provides a comparative study of two approaches for predicting coronary artery disease with a high-dimensional labeled dataset containing information on 865688 Single-Nucleotide Polymorphisms (SNPs) for 3907 observations, collected from the Ontario Heart Genomics Study, see e.g. \cite{robbiethesis}:

\begin{description}
\item[Approach 1] (Random Projections and $k$-NN): As a benchmark experiment, the first approach is to project the genetic dataset onto a lower dimensional space and apply the $k$-NN classifier for the prediction of coronary artery disease.
\item[Approach 2] (MTD Feature Selection and Random Forest): The second approach is to use Mass Transportation Distance Feature Selection to select the important coordinates, or equivalently SNPs, in the genetic dataset and apply the Random Forest classifier for prediction.
\end{description}
The predictive abilities of the two approaches are judged according to three common performance measures in learning theory: the accuracy, F-Measure, and area under the Receiver Operating Characteristic (ROC) curve \cite{rpfmeasure}. 

Results demonstrate that Approach 2 predicts coronary artery disease considerably better than Approach 1. Based on a subset of the genetic dataset containing information on 73571 SNPs from Chromosome 1 for the 3907 observations, Approach 1 achieves a maximum accuracy of 0.5554 and an area under the ROC of 0.5174 when the dataset is projected to a lower 5000-dimensional space and run by the $k$-NN classifier. On the other hand, Approach 2 can achieve a maximum accuracy of 0.6592 and an area under the ROC of 0.8392 when 287 important SNPs, out of the 73571 SNPs in Chromosome 1, are selected for classification by Random Forest. On the entire dataset with 865688 SNPs, Approach 2 can achieve an accuracy of 0.6660 and an area under the ROC of 0.8562. This area under the curve is the highest achieved on this dataset from the Ontario Heart Genomics Study, beating the previous high score of 0.608 from \cite{robbie12}, whose authors considered a panel of 12 previously identified SNPs associated to coronary artery disease and applied the Logistic Regression classifier.

\subsubsection*{Novel Contributions of Thesis}

The main theoretical contribution of this thesis is providing a complete proof that the $k$-Nearest Neighbour classifier is universally consistent in any finite dimensional normed vector space $\Omega = (\mathbb{R}^m, ||\cdot||)$. Although this result is already known, a complete and direct proof has not been found in the current literature, since the classical proof of the universal consistency of $k$-NN, e.g. seen in \cite{devroyeprob}, is specific for the Euclidean space $(\mathbb{R}^m, ||\cdot||_2)$. This classical proof involves Stone's Theorem, listing three conditions which together are sufficient for universal consistency, and Stone's Lemma in Euclidean space, required in proving that $k$-NN satisfies one of the three sufficient conditions \cite{stone77}. For the universal consistency of the $k$-NN classifier in any finite dimensional normed space, this thesis proves the classical Stone's Theorem and a generalized version of Stone's Lemma, whose proof has not been seen before. In addition, this thesis introduces a completely new feature selection method based on the Mass Transportation Distance. The thesis explains its application to reduce dimension of any dataset taking discrete values for prediction, and includes some theory on justifying the use of this distance in supervised learning theory.

Regarding new practical contributions, this thesis presents applications of the $k$-NN classifier, Random Projections, Random Forest, and the new MTD Feature Selection method on a genetic dataset from the Ontario Heart Genomics Study, for the prediction of coronary artery disease using Single-Nucleotide Polymorphisms information. All these algorithms and dimensionality reduction techniques are applied on this dataset for the first time, and another important practical contribution of this thesis is that the approach of applying MTD Feature Selection and Random Forest achieves the best area under the ROC curve ever obtained on this dataset.

\subsubsection*{Outline}

The thesis is structured as follows. Chapter \ref{prelim} first provides a mathematical foundation for supervised learning theory and explains the important concept of universal consistency. This chapter then discusses the biological background for the genetic dataset considered in this thesis. In particular, it explains coronary artery disease and the collection of genetic data, in the form of genotypes of Single Nucleotide Polymorphisms (SNPs), from a Genome Wide Association (GWA) Study. The chapter also surveys some past work in literature on this dataset and, more generally, on the application of supervised learning algorithms for datasets of GWA Studies.

Chapter \ref{supervisedalg} introduces the $k$-Nearest Neighbour classifier and provides a detailed proof that this classifier is universally consistent in any finite dimensional normed vector space. More specifically, the chapter covers a complete proof of Stone's Theorem and a new proof for a generalized version of Stone's Lemma. Chapter \ref{decisiontreesect} explains the Random Forest classifier and includes a section on the Decision Tree classifier, which Random Forest is generalized from. 

Chapter \ref{dimred} discusses two dimensionality reduction methods, Random Projections and the new feature selection technique based on the Mass Transportation Distance, considered in this thesis. An outline of a probabilistic proof of the Johnson-Lindenstrauss Lemma is given to motivate and justify the use of Random Projections, based on randomly generated matrices, in learning theory. A brief theoretical explanation for using the Mass Transportation Distance in data science is given as well.

Chapter \ref{validclass} explains three measures of predictive performance for supervised learning algorithms, namely accuracy, the F-Measure, and area under the Receiver Operating Characteristic curve. This chapter then discusses two methods, called the Holdout Method and cross validation, of dividing a labeled dataset to allow for the evaluation of a classifier and for estimations of optimal classification parameters on the same dataset.

Chapter \ref{datasetchap} provides information on the genetic dataset considered for this thesis, including its class information and the distribution of the 865688 SNPs across 22 chromosomes. This chapter also explains the methodology of applying Random Projections, $k$-NN, MTD Feature Selection, and Random Forest on the genetic dataset, and lists the classification parameters used by these techniques for dimensionality reduction and classification. In addition, the chapter introduces a framework for running Random Projections on any high-dimensional dataset in parallel.

Chapter \ref{resultschap} summarizes all the prediction results, obtained first from the approach of using Random Projections and the $k$-NN classifier, and then from the approach of MTD Feature Selection with Random Forest. This chapter also provides a brief discussion on the comparative results from the two approaches. Finally, Chapter \ref{conclude} concludes this thesis, addresses its limitations, and lists some directions for future research.

Throughout this thesis, knowledge in basic probability theory and real analysis is assumed from the reader. As the thesis author studies mathematics and data science, this thesis is mathematically oriented and its intended readers are mathematicians and data scientists. The main focus of the thesis is to explain dimensionality reduction techniques and classification algorithms in a rigorous setting, and to study the application of these techniques on a high-dimensional discrete (genetic) dataset. As a result, the biological background and explanations regarding biology and the genetic dataset itself are kept at a minimum. The excellent reference \cite{robbiethesis}, on the other hand, is a Master's thesis, on the prediction of coronary artery disease using the same genetic dataset, which is much more focused towards biology and the dataset in question.

\cleardoublepage

\chapter{Preliminaries}\label{prelim}

Chapter \ref{prelim} formalizes the mathematical setting for supervised learning theory and provides a biological introduction to the genetic dataset considered for this thesis. Section \ref{probintrolearning} defines a supervised learning classifier, the process of training and predicting based on a labeled dataset, and the important notion of universal consistency. Section \ref{bioback} introduces coronary artery disease, genetics from the DNA level to the Chromosome level, and Genome Wide Association Studies, where the genetic dataset is collected from.

\section{Introduction to Supervised Learning} \label{probintrolearning}

The goal of this section is formalize the algorithmic notion of training on a sample of labeled observations and applying this information to predict the label for a new observation. In Section \ref{uniconsist}, the property of universal consistency is defined, which makes exact the concept of optimal predictability for a learning algorithm. A good reference on supervised learning theory, and on which Section \ref{probintrolearning} is based, is \cite{devroyeprob}.

Suppose $(X,Y)$ is a pair of random variables taking values in $\Omega\times \{0,1\}$, where $\Omega$ is called the {\em domain} or {\em feature space}. The variable $X \in \Omega$ is called an {\em observation}, possibly containing information in $m$ dimensions (also known as coordinates or features), with {\em label} $Y \in \{0,1\}$. The distribution of this pair is given by $\mu$, a probability measure on $\Omega$, and a regression function of $Y$, $\eta:\Omega\to\left[0,1\right]$ defined by
\[
\eta(x) = \mathrm{Pr}(Y=1|X=x) = \mathbb{E}(Y|X = x).
\]
The function $\eta$ is simply the expectation of the label variable $Y$ when a particular instance of $X$ is observed. 

A {\em training sample} $\mathcal{S}_\mathrm{lab}$, with sample size $n$, is a collection of independent labeled observations:
\begin{equation}\label{trainsamp}
\mathcal{S}_\mathrm{lab} = \{(X_1,Y_1),(X_2,Y_2),\ldots,(X_n,Y_n)\},
\end{equation}
drawn from the same paired distribution as $(X,Y)$. For a fixed training sample $\mathcal{S}_\mathrm{lab}$, a {\em classifier} is any function $g:\Omega\to\{0,1\}$, and a {\em prediction} for a pair $(X,Y)$ is merely $g(X)$, and it need not equal $Y$. One particular classifier is defined below.
\begin{defn}
Given a random pair $(X,Y)$ with distribution $\mu$ and regression function $\eta$, the {\em Bayes classifier} $g^*$ is defined by
\[
g^*(x) = \begin{cases}
1 &\quad\textnormal{ if $\eta(x)>1/2$}\\
0 &\quad\textnormal{ otherwise}.
\end{cases}
\]
\end{defn}
\noindent The Bayes classifier is the optimal classifier, satisfying the important property of having the lowest probability of making a false prediction, out of all classifiers from $\Omega$ to $\{0,1\}$.
\begin{theo}\label{lowesterror}
Given any classifier $g:\Omega\to\{0,1\}$,
\[
\mathrm{Pr}(g^*(X)\neq Y) \leq \mathrm{Pr}(g(X) \neq Y).
\]
\end{theo}

\begin{proof}
It suffices to prove that
\[
\mathrm{Pr}(g^*(X)\neq Y|X = x) \leq \mathrm{Pr}(g(X)\neq Y| X = x),
\]
for all $x\in \Omega$. For any classifier $g$, the following holds:
\begin{align*}
\mathrm{Pr}(g(X)\neq Y|X = x) &= 1 - (\mathrm{Pr}(Y = 1, g(x)=1|X = x) + \mathrm{Pr}(Y = 0,g(x) = 0|X = x) )\\
& = \begin{cases}
1 - \mathrm{Pr}(Y = 1|X = x)& \quad\textnormal{ if }\quad g(x) = 1\\
1 - \mathrm{Pr}(Y = 0|X = x)& \quad\textnormal{ if }\quad g(x) = 0\\
\end{cases}\\
& = 1 - \eta(x)^{g(x)}(1-\eta(x))^{1-g(x)}.
\end{align*}

As a result,
\begin{align*}
\mathrm{Pr}(g(X)\neq Y|&X = x) - \mathrm{Pr}(g^*(X)\neq Y| X = x)\\
& = 1 - \eta(x)^{g(x)}(1-\eta(x))^{1-g(x)} - 1 + \eta(x)^{g^*(x)}(1-\eta(x))^{1-g^*(x)}\\
& = \eta(x)^{g^*(x)}(1-\eta(x))^{1-g^*(x)}- \eta(x)^{g(x)}(1-\eta(x))^{1-g(x)}\\
& = \begin{cases}
0 & \quad\textnormal{ if }\quad g(x) = g^*(x)\\
2\eta(x) - 1 & \quad\textnormal{ if }\quad g^*(x) = 1\textnormal{ and } g(x) = 0\\
1 - 2\eta(x) & \quad\textnormal{ if }\quad g^*(x) = 0\textnormal{ and } g(x) = 1 .
\end{cases}
\end{align*}
Since $g^*(x) = 1$ if and only if $\eta(x) > 1/2$, we must have $2\eta(x) - 1 >0$ when $g^*(x) = 1$ and $1-2\eta(x) \geq 0$ when $g^*(x) = 0$. Therefore,
\[
\mathrm{Pr}(g(X)\neq Y|X = x) - \mathrm{Pr}(g^*(X)\neq Y|X = x)\geq 0,
\]
and the claim is proved.
\end{proof}
In real-life however, the Bayes classifier cannot be constructed because the function $\eta$ is not known. Nevertheless, this classifier is important in learning theory, as it is used to define the notion of universal consistency, a theoretical formalization of how well a classifier can predict.

\subsection{Universal Consistency}\label{uniconsist}

Universal consistency is an important concept in supervised learning theory, and it involves studying the prediction error of a classifier as the training sample size increases to infinity.

\begin{defn}
Let $g:\Omega\to\{0,1\}$ be a classifier and $\mathcal{S}_\mathrm{lab}$ be a training sample as defined in (\ref{trainsamp}). The {\em learning error of $g$} is defined by
\[
L(g) = \mathrm{Pr}(g(X)\neq Y|\mathcal{S}_\mathrm{lab}).
\]
If $g = g^*$ is the Bayes classifier, its learning error is called the {\em Bayes error} and is denoted by $L^* = L(g^*)$.
\end{defn}
\noindent From Theorem \ref{lowesterror}, $L^*$ is the smallest error that can be achieved:
\[
L^* = \inf_{g:\Omega\to\{0,1\}} L(g) = \inf_{g:\Omega\to\{0,1\}} \mathrm{Pr}(g(X) \neq Y | \mathcal{S}_\mathrm{lab}).
\]

The sample size and the training sample have thus far been fixed for a classifier, but in order to define universal consistency, the size $n$ is allowed to become arbitrarily large: $\mathcal{S}_\mathrm{lab} = \mathcal{S}_{\mathrm{lab},n}$, where $n\to\infty$. As a result, a more general terminology for learning must be introduced.
\begin{defn}
A {\em learning rule} is a family of classifiers $\{g_n\}$, where
\[
g_n:(\Omega\times\{0,1\})^n\times \Omega\to\{0,1\}.
\]
\end{defn}
\noindent A learning rule takes in a training sample $\mathcal{S}_{\mathrm{lab},n} \in (\Omega\times\{0,1\})^n$ and a new observation $X\in  \Omega$, and outputs a label $g_n(X)\in \{0,1\}$, while at the same time, the training size $n$ can vary. Universal consistency of a learning rule is defined as follows.

\begin{defn}
A learning rule $\{g_n\}$ is {\em consistent} for $(X,Y)$ with distribution $\mu$ and regression function $\eta$ if
\[
\mathbb{E}(L(g_n))\to L^*,\quad\textnormal{ as $n\to\infty$,}
\]
where the expectation $\mathbb{E}(\cdot)$ is taken over all possible training samples $\mathcal{S}_{\mathrm{lab},n}$. The rule $\{g_n\}$ is {\em universally consistent} if it is consistent for any distribution and regression function of $(X,Y)$. 
\end{defn}

In practice and for the remaining chapters of this thesis, a learning rule is often referred to as a classifier or a (supervised) learning algorithm, with the implicit knowledge that the training sample, simply denoted as $\mathcal{S}_\mathrm{lab}$, and its size $n$ can always vary. Furthermore, a learning rule is commonly defined in two stages: 1) a training step on how the rule trains on a sample $\mathcal{S}_\mathrm{lab} \in (\Omega\times\{0,1\})^n$ and 2) a prediction, or classification, step on how this rule predicts the label for a new observation $X\in \Omega$. Chapters \ref{supervisedalg} and \ref{decisiontreesect} explain two well-known learning rules in the literature today.

\section{Biological Background}\label{bioback}
This section provides the necessary background to understand the biology behind the genetic dataset studied in this thesis. In Section \ref{cad}, coronary artery disease and its relevance in health care are explained. In Section \ref{genetics}, a brief introduction to genetics is given and in particular, DNA sequences and Single-Nucleotide Polymorphisms are explained. Genome Wide Association Studies, from which the genetic dataset is collected, are introduced in Section \ref{gwas}, and some past work in the literature involving supervised learning based on these studies are surveyed in Section \ref{refgwas}.

\subsection{Coronary Artery Disease}\label{cad}

It is a well-known fact that in high-income countries, such as Canada and the United States of America, the number one leading cause of death is from cardiovascular diseases. This is a class of diseases involving the heart and its surrounding blood vessels, including arteries, and the most common type of cardiovascular disease is {\em coronary artery disease} (CAD). It occurs when cholesterol and other fatty substances build up and clog the arteries of an individual's heart, slowing down the flow of blood. Significant clogs in the arteries can often lead to heart attacks and even death. As a result, preventing and understanding CAD are important problems in health care today, especially in countries with aging populations as the prevalence of CAD increases with age \cite{cadweb}.

Medical research has shown that age, sex, smoking, diabetes, hypertension, and high cholesterol levels are important physical risk factors for developing CAD \cite{robbie12}, while exercising regularly, reducing stress, and eating low-fat and low-salt diet can prevent CAD \cite{cadweb}. It is also widely accepted in biology that genetics play a role in the prevalence of CAD; however, the extent is heavily debated today, see e.g. \cite{cadroberts} for a discussion regarding genetics and CAD. It is not known whether coronary artery disease can be predicted with high probability purely based on the genetic information of an individual.  This thesis attempts to answer this question in terms of the genetic information of Single-Nucleotide Polymorphisms, as explained in Section \ref{genetics}.

\subsection{Genetics}\label{genetics}

{\em Genetics} is the biological study of heredity, the process of passing traits from a parent to an offspring, also known as trait inheritance. At the molecular level, this process involves {\em Deoxyribonucleic Acid} (DNA), a molecule that encodes information for the trait development and functionality of any living organism and is passed down to its offsprings. For humans, (double-stranded) DNA is a pair of molecules, containing repeated units of four possible nucleotides, {\em Cytosine} (C), {\em Guanine} (G), {\em Adenine} (A), and {\em Thymine} (T). The pair of molecules are tightly held together and have complementary nucleotides, called base pairs: the corresponding pairs are Cytosine with Guanine and Adenine with Thymine \cite{genehandbook}. Commonly, DNA is visualized as a pair of strings containing the base pairs, and often the second string is omitted as it is completely determined by the first:
\begin{align*}
\,\ldots\mathrm{A\,T\,G\,C\,T\,T\,C\,G\,G\,C\,A\,A\,G\,A\,C\,T\,C\,A\,A\,A\,A}\ldots\,\,\\
(\ldots\mathrm{T\,A\,C\,G\,A\,A\,G\,C\,C\,G\,T\,T\,C\,T\,G\,A\,G\,T\,T\,T\,T}\ldots)
\end{align*}

An individual has 46 double-stranded DNA paired together to form 23 thread-like structures called {\em chromosomes}, where 23 double-stranded DNA come from the mother and the other 23 from the father. There are an estimated 3.2 billion total base pairs in human DNA, with Chromosome 1 having the most number of base pairs (approximately 240 million) and Chromosome 22 having the least (approximately 40 million) \cite{dnacount}. The level of organization for a person's genetic information can be visualized as in Figure \ref{dnapic}, where for each chromosome, genetic information is given as two double-stranded DNA strings.

As genetic information is passed from a parent to an offspring, variations or mutations sometimes occur in certain segments of DNA. These DNA differences result precisely in trait variations among individuals \cite{genehandbook}. For instance, the following pairs of doubled-stranded DNA for two individuals in a chromosome have a variation at the bolded nucleotide:
\begin{align*}
\textnormal{Individual 1:}  \quad\ldots\mathrm{A\,T\,G\,C\,T\,T\,C\,G\,{\bf T}\,C\,A\,A\,G\,A\,C\,T\,C\,A\,A\,A\,A}\ldots\\
\quad\ldots\mathrm{A\,T\,G\,C\,T\,T\,C\,G\,{\bf G}\,C\,A\,A\,G\,A\,C\,T\,C\,A\,A\,A\,A}\ldots\\
\textnormal{Individual 2:}  \quad\ldots\mathrm{A\,T\,G\,C\,T\,T\,C\,G\,{\bf G}\,C\,A\,A\,G\,A\,C\,T\,C\,A\,A\,A\,A}\ldots\\
\quad\ldots\mathrm{A\,T\,G\,C\,T\,T\,C\,G\,{\bf G}\,C\,A\,A\,G\,A\,C\,T\,C\,A\,A\,A\,A}\ldots
\end{align*}
When a variation occurs at a base pair position in 1\% or more of a population, it is known as a {\em Single-Nucleotide Polymorphism} (SNP); normally, only one possible variation at a position can occur \cite{cadroberts}. The nucleotide at a SNP position that occurs more often in a population is called the major allele, while the less frequent one is termed the minor allele. For a SNP, since a person has two DNA in each chromosome, there are three possible combinations, or {\em genotypes}, of alleles. The genotype is known as {\em homozygous major} if two copies of the major allele are present, {\em homozygous minor} if two copies of the minor allele are present, and {\em heterozygous} if one copy of each allele is present \cite{robbiethesis}.

Continuing the DNA example above, suppose that the variation at the bolded base pair occurs in 5\% of a population and that Thymine (T) is the most frequent nucleotide. Then, this variation is a SNP and the SNP genotype for Individual 1 is heterozygous while it is is homozygous minor for Individual 2. Across all 23 chromosomes, there are over an estimated 3 million SNPs \cite{robbiethesis}. Since SNPs record all genetic variation among humans at the DNA base pair level, they are extremely important to study biologically and may be useful in the genetic prediction of certain diseases, such as coronary artery disease. Section \ref{gwas} below explains Genome Wide Association Studies which collect SNP information from individuals in order to study possible genetic associations to physical traits.

\begin{figure}
\begin{center}
\includegraphics[scale = 0.5]{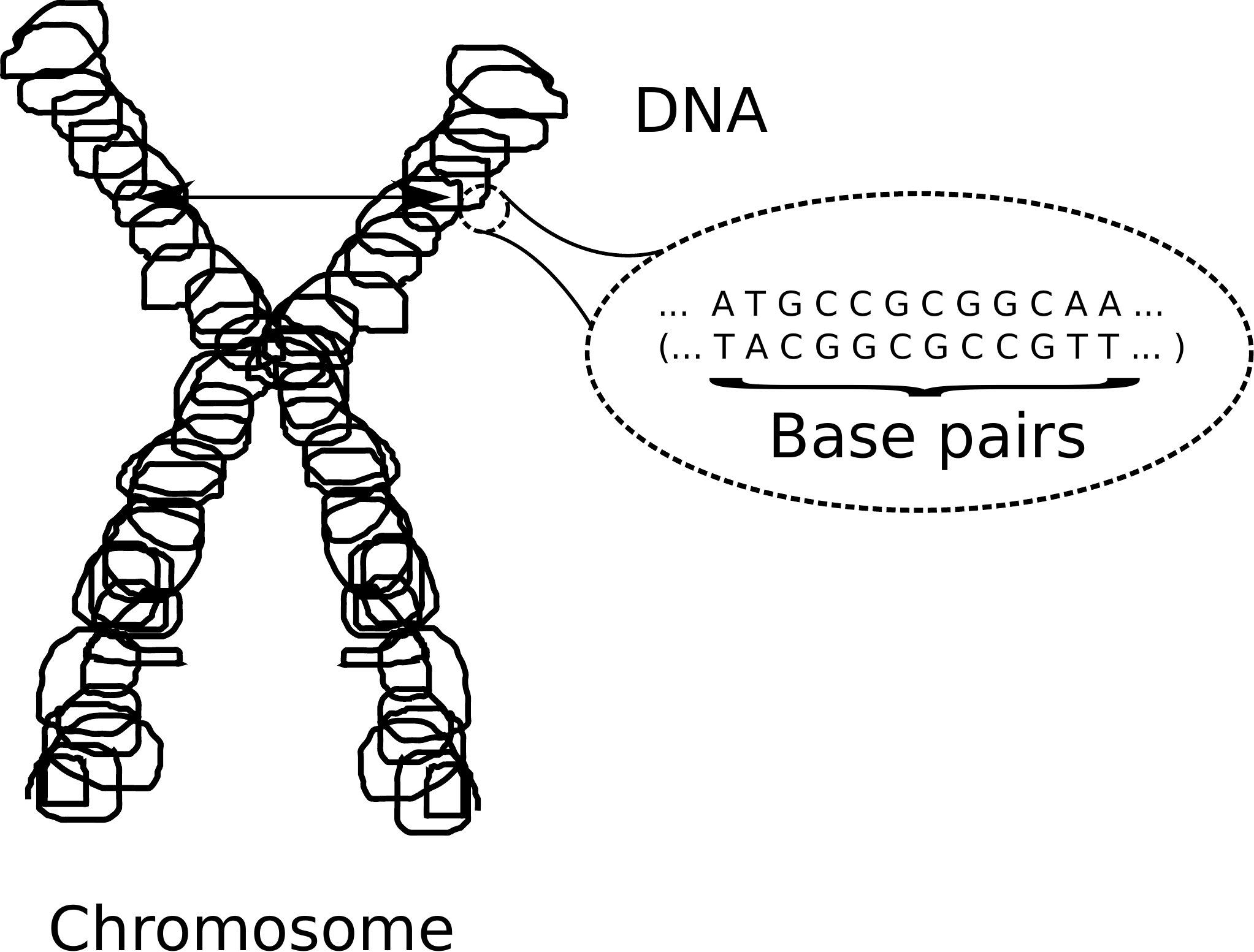}
\caption[Shape of a chromosome]{Two double-stranded DNA are present in a chromosome.}
\label{dnapic}
\end{center}
\end{figure}

\subsection{Genome Wide Association Studies}\label{gwas}

In genetics, a {\em Genome Wide Association} (GWA) {\em Study} is a study that examines Single-Nucleotide Polymorphisms and their associations to a certain trait, such as coronary artery disease. Such a study commonly considers SNPs from two groups of randomly selected individuals, those with the trait, known as {\em cases}, and those without it, known as {\em controls}. For each individual of the study, a blood sample is drawn and a DNA based microarray chip is used to determine the individual's genotype at each SNP position from the blood sample. In addition, information on other physical traits of interest, such as sex, age, and cholesterol level, could be collected for case-control sampling corrections or for future association studies \cite{robbiethesis}.

A specific GWA Study is the {\em Ontario Heart Genomics Study} (OHGS) which collects SNP information, mostly from individuals residing in Ottawa, Canada to study genetic association to coronary artery disease. Genotypes at approximately 900,000 SNPs (the coordinates of the dataset) from around 4000 individuals, labeled as controls or cases, are collected for this study. The microarray chip used to determine these genotypes is the Affymetrix GeneChip 6.0. However, due to microarray limitations, the exact genotype at a SNP cannot be determine by this chip, rather a probability is assigned to each of the three possible genotypes. A good reference, which explains GWA Studies and the OHGS dataset in much more detail, including a section on the precise genotyping procedures, is the Master's thesis \cite{robbiethesis}. 

This thesis considers the genetic dataset from OHGS, with the main goal of predicting coronary artery disease based on individuals' SNP genotypes using data science techniques. See Section \ref{inforaw} for the exact size and format of the genetic dataset used in the thesis. In general, applying data science methods for GWA Studies, with the goal of predicting a physical trait, is an extremely new research direction, where only a handful of results have been published in the literature. Section \ref{refgwas} below briefly surveys some of these past results in this field of research.

\subsection{Past Work on GWA Studies}\label{refgwas}

Past research done on the OHGS dataset for predicting coronary artery disease (CAD) include \cite{robbiethesis} and \cite{robbie12}, where the respective authors considered the Naive Bayes and Logistic Regression classifiers. The current best classification performance, evaluated according to the area under the Receiver Operating Characteristic (ROC) curve (see Section \ref{rocarea} for an explanation of this curve) on this dataset, is 0.608 obtained in \cite{robbie12}. Davies et al. in this paper considered the Logistic Regression classifier using 12 SNPs that had been previously shown to be associated with CAD.

Other than coronary artery disease, in 2009, Wei et al. published one of the first papers \cite{weisvm} of applying supervised learning algorithms on a GWA Study. The authors compared the predictive performance of the Logistic Regression and Support Vector Machines (SVM) classifiers for predicting type 1 diabetes, and were able to obtain an area under the ROC curve of approximately 0.84 with SVM and 409 SNPs.

 Goldstein et al., in 2010, considered a GWA Study involving a multiple sclerosis (MS) dataset, with 325807 SNPs and 3362 observations, and applied the Random Forest classifier to predict MS. An accuracy of 0.65 was obtained by Random Forest \cite{rfms}. Then, in 2011, the paper \cite{rfasthma} reported the application of Random Forest on a GWA Study, with 417 observations, to predict asthma exacerbations in children. With 320 SNPs, Random Forest was able to achieve an area under the ROC curve of 0.66 according to an independent sample for evaluation.

Mao and Kelly in \cite{maosnps} compared the ability of a modified version of Random Forest against common learning algorithms in literature, including the nearest neighbour classifier and Support Vector Machines (SVM), on two genetic datasets to predict Crohn's disease and autoimmune disorder. On the genetic dataset of 103 SNPs and 387 observations for predicting Crohn's disease, the modified Random Forest classifier was able to achieve an accuracy of 0.744; on the dataset of 108 SNPs and 1036 observations for autoimmune disorder, this classifier was able to obtain an accuracy of 0.721. Both accuracies were considerably higher than those from the other classifiers. Kooperberg et al. in \cite{crohnagain} also studied the prediction of Crohn's disease on two genetic datasets, using the Logistic Regression classifier, and obtained a best area under the ROC curve of 0.637 with 177 SNPs and 4686 observations.

In summary, the applications of learning algorithms, such as Random Forest, Logistic Regression, and SVM, on GWA Studies in the past have been fairly straightforward. At the same time, the selection of SNPs for classification, since considering all possible SNPs is prohibitively expensive, has been based on statistical approaches or on previous findings in the genetics community. It can be argued that, thus far, trait predictions using learning algorithms based on GWA Studies have been primarily investigated by geneticists and statisticians. This thesis, on the other hand, aims to study the prediction of coronary artery disease from a complete data science perspective, with the main focus of introducing certain supervised learning algorithms and dimensionality reduction techniques and applying them on the high-dimensional OHGS genetic dataset. Chapters \ref{supervisedalg} and \ref{decisiontreesect} explain the $k$-Nearest Neighbour and Random Forest classifiers this thesis considers for predicting artery disease.

\cleardoublepage

\chapter{The $k$-Nearest Neighbour Classifier}\label{supervisedalg}

Chapter \ref{supervisedalg} explains a well known classifier in data science, called the $k$-Nearest Neighbour ($k$-NN) classifier, and provides a complete proof that it is universally consistent in any finite dimensional normed vector space. In particular, Section \ref{knndef} introduces the $k$-NN classifier and Section \ref{consistknn} states and proves Stone's Theorem and a generalized version of Stone's Lemma, and the universal consistency of $k$-NN follows as a corollary.

Recall that, from e.g. \cite{normref}, for the finite dimensional vector space $\mathbb{R}^m$, a norm $||\cdot ||:\mathbb{R}^m\to \mathbb{R}$ satisfies
\begin{multicols}{2}
\begin{enumerate}
\item $||x|| = 0$ implies $x = 0$
\item $||ax|| = |a|\!\cdot\!||x||$
\item $||x + y|| \leq ||x|| + ||y||$,
\end{enumerate}
\end{multicols}
\noindent for all $x,y\in\mathbb{R}^m$ and $a \in\mathbb{R}$. For example, a commonly used norm is the $\ell_p$ norm $||\cdot||_p$ defined by
\[
||x||_p = \begin{cases}
\left(\displaystyle\sum_{i = 1}^m |x_i|^p\right)^{1/p} & \quad\textnormal{ if $ 1 \leq p < \infty$}\\
\displaystyle\max_{i = 1}^m |x_i| &\quad\textnormal{ if $p = \infty$},
\end{cases}
\]
for $x = (x_1,x_2,\ldots,x_m)\in\mathbb{R}^m$. Equipped with a norm, $(\mathbb{R}^m,||\cdot||)$ is a metric space with distance defined by $|| x- y||$ for $x,y\in\mathbb{R}^m$. 

\section{Definition of the $k$-NN Classifier}\label{knndef}

The $k$-Nearest Neighbour ($k$-NN) classifier \cite{knnorig} is a well-known learning rule in literature, based on searching for $k$-nearest neighbours in a metric space. Section \ref{knndef} defines the $k$-NN classifier for the domain being a finite dimensional normed vector space $\Omega = (\mathbb{R}^m,||\cdot||)$ (although in general, it can be defined in the same manner for any metric space).
 
Suppose the domain is a finite dimensional normed vector space $\Omega = (\mathbb{R}^m,||\cdot||)$ and a random pair $(X,Y)$ takes values in $\mathbb{R}^m\times\{0,1\}$, with distribution $\mu$ and regression function $\eta$. Given a training sample $\mathcal{S}_\mathrm{lab} = \{(X_1,Y_1),(X_2,Y_2),\ldots,(X_n,Y_n)\}$, and a new observation $X = x\in\mathbb{R}^m$, known as the {\em query}, the $k$-NN classifier arranges the training observations in the order of increasing distances to $x$:
\begin{equation}\label{sortdist}
(X_{(1)},Y_{(1)}),\ldots,(X_{(n)},Y_{(n)}),
\end{equation}
where $||X_{(1)} - x||\leq ||X_{(2)} - x||\leq \ldots \leq ||X_{(n)} - x||$. Then, it considers only the first $k$ observations $(X_{(1)},Y_{(1)}),\ldots,(X_{(k)},Y_{(k)})$, which are the $k$-nearest neighbours of $x$, and predicts the label of $x$ to be $\mathrm{mode}\{Y_{(1)},\ldots,Y_{(k)}\}$ \cite{knnorig}. The value $k$ is normally taken to be odd to avoid ties, and multiple values are used in real-life applications in order to determine the optimal one for classification. Figure \ref{knnpic} is an illustration of the $k$-NN classifier for $k = 5$. 

\begin{figure}
\begin{center}
\includegraphics[scale = 0.75]{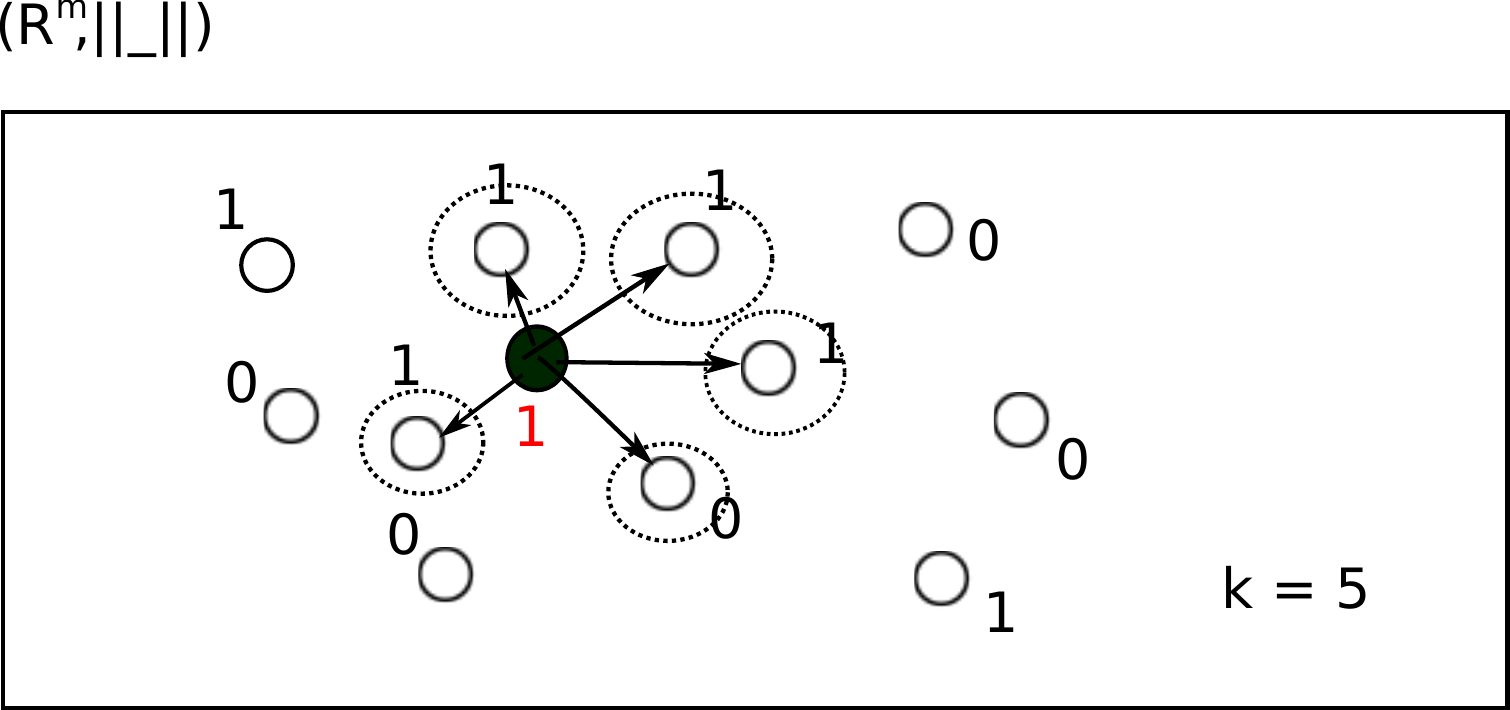}
\caption[Diagram of $k$-nearest neighbours]{The mode of the labels for the $k = 5$ closest neighbours (open, with dotted circles) to the query point (black) is $Y = 1$; thus, this query point is predicted to have label 1.}
\label{knnpic}
\end{center}
\end{figure}

More generally, according to \cite{devroyeprob}, the $k$-NN classifier belongs to a family of learning rules with the common approach of defining an estimate $\eta_n:\mathbb{R}^m\to\left[0,1\right]$, for the true regression function $\eta$, by
\[
\eta_n(x) = \sum_{i = 1}^n Y_i W_{ni}(x),
\]
where $W_{ni}(x) = W_{ni}(x,\mathcal{S}_\mathrm{lab})$ are some weights that are non-negative and add up to one. Then, a classifier $g_n$ is defined by
\begin{equation}\label{gndef}
g_n(x) = \begin{cases}
1 &\quad\textnormal{ if $\eta_n(x)=\displaystyle\sum_{i = 1}^n Y_i W_{ni}(x)>1/2$}\\
0 &\quad\textnormal{ otherwise.}
\end{cases}
\end{equation}
For the $k$-NN classifier, from the sorted training observations in (\ref{sortdist}), the weights are given as
\[
W_{ni}(x) = 1/k \quad\textnormal{ if and only if } \quad X_i \in\{X_{(1)},X_{(2)},\ldots,X_{(k)}\},
\]
and as a result, an equivalent definition for the $k$-NN classifier, denoted here as $g_n$, is
\begin{align*}
g_n(x) &=\begin{cases}
1 &\quad\textnormal{ if $\eta_n(x)= \displaystyle\frac{1}{k}\sum_{i = 1}^k Y_{(i)} >1/2$}\\
0 &\quad\textnormal{ otherwise.}
\end{cases}
\end{align*}

The $k$-NN classifier, introduced in 1967 by Cover and Hart \cite{knnorig}, is one of the first, and most intuitive, supervised learning algorithms to be developed. Stone in 1977 published the famous result that this classifier is universally consistent in the finite dimensional Euclidean space \cite{stone77}. The proof is based on Stone's Theorem, which lists three conditions that together imply universal consistency, and Stone's Lemma, which is used to prove that $k$-NN satisfies one of these conditions in the Euclidean case. More generally, in any finite dimensional normed vector space, this classifier is still universally consistent, a strong property that justifies its use in data science. Section \ref{consistknn} explains a proof of this property for the $k$-NN classifier. In particular, the classical Stone's Theorem and a generalized version of Stone's Lemma, whose proof has not been found in the literature before, are stated and proved.

\section{Universal Consistency of $k$-NN}\label{consistknn}

The goal of Section \ref{consistknn} is to prove Theorem \ref{consistencyknn} found below. The training sample $\mathcal{S}_\mathrm{lab} = \{(X_1,Y_1),(X_2,Y_2),\ldots,(X_n,Y_n)\}$ of size $n$, drawn from distribution $\mu$ and regression function $\eta$, exists but may not be explicitly mentioned throughout this section.

\begin{theo}[Consistency of the $k$-NN classifier]\label{consistencyknn}
If $k\to\infty$ and $k/n\to 0$, then the $k$-NN Classifier is universally consistent in any finite dimensional normed vector space $(\mathbb{R}^m,||\cdot||)$.
\end{theo}
\noindent The proof of this result requires two crucial theorems: Stone's Theorem and a generalized version of Stone's Lemma, and consequently, it is structured as follows.
\begin{enumerate}
\item Stone's Theorem, listing three conditions that together imply universal consistency, is stated and proved. 
\item A generalized version of Stone's Lemma, required to demonstrate that the $k$-NN classifier satisfies one of the conditions in Stone's Theorem, and its proof are given. 
\item The universal consistency of the $k$-NN classifier follows as a corollary from these two theorems.
\end{enumerate}

\subsubsection{1. Stone's Theorem}

Stone proved the following theorem in 1977 which gives sufficient conditions for a learning rule to be universally consistent \cite{stone77}.

\begin{theo}[Stone's Theorem \cite{stone77}]\label{stonetheorem}
Let a learning rule $\{g_n\}$ be defined as in (\ref{gndef}). Suppose the following for any distribution $\mu$ and regression function $\eta$:
\begin{enumerate}
\item There exists $c\in\mathbb{R}$ such that for every non-negative measurable function $f$ on $\mathbb{R}^m$ with finite expectation,
\[
\mathbb{E}\left( \sum_{i = 1}^n W_{ni}(X) f(X_i)        \right) \leq c \mathbb{E}(f(X))
\]
\item For all $a>0$,
\[
\mathbb{E}\left(\sum_{i = 1}^n W_{ni}(X) 1_{\{ ||X_i  - X||>a\}}      \right)\to 0,
\]
where $1_{\{ ||X_i  - X||>a\}} = 1$ if and only if $||X_i  - X||>a$
\item
\[
\mathbb{E}\left(\max_{i = 1}^n W_{ni}(X)      \right)\to 0.
\]
\end{enumerate}
Then $\{g_n\}$ is universally consistent in $(\mathbb{R}^m,||\cdot||)$.
\end{theo}

\subsubsection{1. Proof of Stone's Theorem}

The proof of Stone's Theorem requires the notion of convexity and the finite version of Jensen's inequality, see e.g. \cite{jensenfinite}, along with a few lemmas. Note that the proof for Stone's Theorem presented here is entirely based on \cite{devroyeprob}. For a random pair $(X,Y)\in\mathbb{R}^m\times\{0,1\}$ with distribution $\mu$ and regression function $\eta$, Lemma \ref{regestimate} relates the difference in the learning error for a classifier defined according to the estimated regression function $\eta_n$ and the Bayes error to the expectation of $(\eta(X) - \eta_n(X))^2$. Lemmas \ref{squarecorollary} and \ref{squareinequality} are two quick observations about scalars in $\mathbb{R}$.

\begin{defn}
A function $\varphi:\mathbb{R}\to\mathbb{R}$ is {\em convex} if for all $t\in\left[0,1\right]$ and for all $x,y\in\mathbb{R}$,
\[
\varphi(tx + (1-t)y)\leq t\varphi(x) + (1-t)\varphi(y).
\]
\end{defn}

\begin{theo}[Jensen's inequality - finite version]\label{jensen}
For a convex function $\varphi$, for weights $a_1,a_2,\ldots,a_n\in\mathbb{R}^+$ such that $\sum_{i = 1}^n a_i = 1$, and for any real numbers $b_1,\ldots,b_n$,
\[
\varphi\left(      \sum_{i = 1}^n a_ib_i  \right)\leq \sum_{i = 1}^n a_i \varphi(b_i).
\]
\end{theo}
\begin{proof}
We proceed by induction. For $n = 1$, the inequality holds trivially since $a_1 = 1$ by assumption. Suppose the inequality holds for $n$ and we must show that
\[
\varphi\left(      \sum_{i = 1}^{n+1} a_ib_i  \right)\leq \sum_{i = 1}^{n+1} a_i \varphi(b_i).
\]
Indeed,
\begin{align*}
\varphi\left(      \sum_{i = 1}^{n+1} a_ib_i  \right) &= \varphi\left(      \sum_{i = 1}^{n} a_ib_i +a_{n+1}b_{n+1} \right)\\
&=\varphi\left(      \sum_{i = 1}^{n} a_ib_i +\left(  1 - \sum_{i = 1}^n a_i   \right) b_{n+1} \right)\\
& =\varphi\left(    \left( \sum_{i = 1}^n a_i    \right)  \left(\frac{\sum_{i = 1}^{n} a_ib_i}{\sum_{i = 1}^n a_i}\right) +\left(  1 - \sum_{i = 1}^n a_i   \right) b_{n+1} \right)\\
\textnormal{(by convexity of $\varphi$)}\quad&\leq \left( \sum_{i = 1}^n a_i  \right)\varphi\left(  \frac{\sum_{i = 1}^{n} a_ib_i}{\sum_{i = 1}^n a_i} \right) + \left(  1 - \sum_{i = 1}^n a_i   \right) \varphi(b_{n+1})\\
& = \left( \sum_{i = 1}^n a_i  \right)\varphi\left( \sum_{i = 1}^{n} \left(\frac{a_i}{\sum_{i = 1}^n a_i}\right)b_i \right) + \left(  1 - \sum_{i = 1}^n a_i   \right) \varphi(b_{n+1})\\
\textnormal{(by the induction hypothesis)}\quad&\leq\left( \sum_{i = 1}^n a_i  \right)\left( \sum_{i = 1}^{n} \left(\frac{a_i}{\sum_{i = 1}^n a_i}\right)\varphi(b_i) \right) + \left(  1 - \sum_{i = 1}^n a_i   \right) \varphi(b_{n+1})\\
& = \sum_{i = 1}^n a_i\varphi(b_i) + a_{n +1}\varphi(b_{n+1})\\
&=\sum_{i = 1}^{n+1} a_i \varphi(b_i).
\end{align*}
\end{proof}
\begin{lem}\label{regestimate}
Given any estimate $\eta_n$ of the regression function $\eta$ of $Y$, if a classifier $g_n:\mathbb{R}^m\to\{0,1\}$ is defined by $\eta_n$, i.e.
\[
g_n(x) = \begin{cases}
1 &\quad\textnormal{ if $\eta_n(x)=\displaystyle\sum_{i = 1}^n Y_i W_{ni}(x)>1/2$}\\
0 &\quad\textnormal{ otherwise,}
\end{cases}
\]
then the following holds:
\[
\mathrm{Pr}(g_n(X)\neq Y) - L^*\leq 2\sqrt{\mathbb{E}((\eta(X) - \eta_n(X))^2)}
\]
\end{lem}

\begin{proof}
Based on the proof of Theorem \ref{lowesterror},
\[
\mathrm{Pr}(g_n(X)\neq Y|X = x) - \mathrm{Pr}(g^*(X)\neq Y|X = x) =  |2\eta(x) - 1|1_{\{g_n(x) \neq g^*(x)\}}.
\]
Hence,
\begin{align*}
\mathrm{Pr}(g_n(X)\neq Y) - L^* & = \mathrm{Pr}(g_n(X)\neq Y) - \mathrm{Pr}(g^*(X)\neq Y)\\
& = 2\int_{\mathbb{R}^m} |\eta(x) - 1/2|1_{\{g_n(x)\neq g^*(x)\}} d\mu(x).
\end{align*}
Note that if $g_n(x) = 1$ and $g^*(x) = 0$, then $\eta(x) \leq 1/2$ and $\eta_n(x)> 1/2$, so
\begin{equation}\label{etahalf}
|\eta(x) - 1/2| \leq |\eta(x) - \eta_n(x)| 
\end{equation}
and similarly, if $g_n(x) = 0$ and $g^*(x) = 1$, the same inequality as in (\ref{etahalf}) holds. As a result,
\begin{align*}
\mathrm{Pr}(g_n(X)\neq Y) - L^* & \leq  2\int_{\mathbb{R}^m} |\eta(x) - \eta_n(x)| d\mu(x)\\
& = 2\mathbb{E}(|\eta(X) - \eta_n(X)|)\\
& \leq 2 \sqrt{\mathbb{E}((\eta(X) - \eta_n(X))^2)}.
\end{align*}
The last inequality follows from the measure-theoretic version of Jensen's inequality (Theorem \ref{jensen}), referenced from e.g. \cite{rudin}.
\end{proof}
\begin{lem}\label{squarecorollary}
For weights $a_1,a_2,\ldots,a_n\in\mathbb{R}^+$ such that $\sum_{i = 1}^n a_i = 1$, and for any real numbers $b_1,b_2,\ldots,b_n$,
\[
\left(      \sum_{i = 1}^n a_ib_i  \right)^2\leq \sum_{i = 1}^n a_i b_i^2.
\]
\end{lem}
\begin{proof}
Again, by Jensen's inequality, Theorem \ref{jensen}.
\end{proof}
\begin{lem}\label{squareinequality}
For any $a,b,c\in\mathbb{R}$, the following holds.
\begin{enumerate}
\item $(a+b)^2 \leq 2(a^2 + b^2)$
\item $(a+b+c)^2 \leq 3(a^2 + b^2 + c^2)$
\end{enumerate}
\end{lem}

With these lemmas in hand, Stone's Theorem can now be proved.

\begin{proof}[Stone's Theorem]
By Lemma \ref{regestimate}, it suffices to prove that
\[
\mathbb{E}((\eta(X) - \eta_n(X))^2)\to 0.
\]
Let $\epsilon>0$. Denote
\[
{\hat \eta_n(x)} = \sum_{i = 1}^n \eta(X_i) W_{ni}(x)
\]
and we have that
\begin{align*}
\mathbb{E}((\eta(X) - \eta_n(X))^2) & = \mathbb{E}((\eta(X) - {\hat \eta_n(X)}+ {\hat \eta_n(X)}- \eta_n(X))^2)\\
\textnormal{(by Lemma \ref{squareinequality})}\quad& \leq 2\mathbb{E}((\eta(X) - {\hat \eta_n(X)})^2) + 2\mathbb{E}(    ({\hat \eta_n(X)}- \eta_n(X))^2   )\\
&=2\mathbb{E}\left(\left(\sum_{i = 1}^nW_{ni}(X)(\eta(X) - \eta(X_i))\right)^2\right) + 2\mathbb{E}(    ({\hat \eta_n(X)}- \eta_n(X))^2   )\\
\textnormal{(by Lemma \ref{squarecorollary})}\quad& \leq 2\mathbb{E}\left(\sum_{i = 1}^nW_{ni}(X)(\eta(X) - \eta(X_i))^2\right) + 2\mathbb{E}(    ({\hat \eta_n(X)}- \eta_n(X))^2   )\\
\end{align*}

We would now like to show that both terms of the right-hand expression can be made arbitrarily small. Since continuous functions with bounded support are uniformly continuous and dense in $L_2(\mu)$, there exists some uniformly continuous function $\eta^*$ such that
\[
\mathbb{E}(     (\eta(X) - \eta^*(X))^2  )<\epsilon.
\]
Thus,
\begin{align*}
&\hspace{-30mm}\mathbb{E}\left(\sum_{i = 1}^nW_{ni}(X)(\eta(X) - \eta(X_i))^2\right)\\
&\hspace{-3mm} = \mathbb{E}\left(   \sum_{i = 1}^n W_{ni}(X) ( \eta(X) - \eta^*(X) + \eta^*(X) - \eta^*(X_i) +  \eta^*(X_i) - \eta(X_i)             )^2           \right)\\
\textnormal{(by Lemma \ref{squareinequality})}\quad& \hspace{-3mm}\leq 3 \mathbb{E}\left(  \sum_{i = 1}^n W_{ni}(X) ( \eta(X) - \eta^*(X))^2  \right)\\
&\hspace{-3mm}+3 \mathbb{E}\left(  \sum_{i = 1}^n W_{ni}(X) (\eta^*(X) - \eta^*(X_i))^2  \right)\\
&\hspace{-3mm}+3 \mathbb{E}\left(   \sum_{i = 1}^n W_{ni}(X) (\eta^*(X_i) - \eta(X_i))^2 \right)\\
\textnormal{(by assumption 1)}\quad&\hspace{-3mm}<3\epsilon + 3 \mathbb{E}\left(  \sum_{i = 1}^n W_{ni}(X) (\eta^*(X) - \eta^*(X_i))^2  \right) + 3c\epsilon.
\end{align*}
For the middle term, there exists $a>0$ such that if $||X - X_i||\leq a$, then $|\eta^*(X)- \eta^*(X_i)|<\sqrt{\epsilon}$ because $\eta^*$ is uniformly continuous. Consequently,
\begin{align*}
 \mathbb{E}\left(  \sum_{i = 1}^n W_{ni}(X) (\eta^*(X) - \eta^*(X_i))^2  \right) & =  \mathbb{E}\left(  \sum_{i = 1}^n W_{ni}(X) (\eta^*(X) - \eta^*(X_i))^2 1_{\{||X_i - X||>a\}}\right) \\
 &+  \mathbb{E}\left(  \sum_{i = 1}^n W_{ni}(X) (\eta^*(X) - \eta^*(X_i))^2 1_{\{||X_i - X||\leq a\}} \right)\\
 &< \mathbb{E}\left(  \sum_{i = 1}^n W_{ni}(X) 1_{\{||X_i - X||>a\}}\right) + \epsilon\\
\textnormal{(by assumption 2)}\quad & < \epsilon + \epsilon = 2\epsilon.
\end{align*}
The first inequality above is due to the fact that $|\eta^*(X) - \eta^*(X_i)|\leq 1$. As a result,
\begin{align*}
\mathbb{E}\left(\sum_{i = 1}^nW_{ni}(X)(\eta(X) - \eta(X_i))^2\right)&<3\epsilon + 3 \mathbb{E}\left(  \sum_{i = 1}^n W_{ni}(X) (\eta^*(X) - \eta^*(X_i))^2  \right) + 3c\epsilon\\
&<3\epsilon + 6\epsilon + 3c\epsilon = 3\epsilon(3+c)
\end{align*}

For the second term,
\begin{align*}
\mathbb{E}(    (   {\hat \eta_n(X)} - \eta_n(X)     )^2       ) & = \mathbb{E}\left(           \left(   \sum_{i = 1}^n\eta(X_i)W_{ni}(X) - \sum_{i = 1}^n Y_i W_{ni}(X)         \right)^2     \right)\\
& = \mathbb{E}\left(           \left(   \sum_{i = 1}^n W_{ni}(X) ( \eta(X_i) - Y_i)         \right)^2     \right)\\
& = \mathbb{E}\left(    \sum_{i = 1}^n \sum_{j = 1}^n W_{ni}(X)W_{nj}(X)(   \eta(X_i) - Y_i)  (   \eta(X_j) - Y_j)     \right)\\
& =     \sum_{i = 1}^n \sum_{j = 1}^n \mathbb{E}(W_{ni}(X)W_{nj}(X)(   \eta(X_i) - Y_i)  (   \eta(X_j) - Y_j))  \\ 
& =\mathbb{E}\left( \sum_{i = 1}^n ( W_{ni}(X))^2 (\eta(X_i) - Y_i)^2\right),
\end{align*}
since the sample $\mathcal{S}_\mathrm{lab} = \{(X_1,Y_1),(X_2,Y_2),\ldots,(X_n,Y_n)\}$ is assumed to be independent. As a result,
\begin{align*}
\mathbb{E}(    (   {\hat \eta_n(X)} - \eta_n(X)     )^2       ) & = \mathbb{E}\left( \sum_{i = 1}^n ( W_{ni}(X))^2 (\eta(X_i) - Y_i)^2\right)\\
& \leq \mathbb{E}\left(  \sum_{i = 1}^n ( W_{ni}(X)   )^2    \right)\\
& \leq \mathbb{E}\left(  \sum_{i = 1}^n W_{ni}(X)  \left(\max_{i = 1}^n W_{ni}(X)\right)    \right)\\
& = \mathbb{E}\left(    \max_{i = 1}^n W_{ni}(X)  \right)<\epsilon,
\end{align*}
by assumption 3.

Altogether, we have that
\begin{align*}
\mathbb{E}((\eta(X) - \eta_n(X))^2) &\leq 2\mathbb{E}\left(\sum_{i = 1}^nW_{ni}(X)(\eta(X) - \eta(X_i))^2\right) + 2\mathbb{E}(    ({\hat \eta_n(X)}- \eta_n(X))^2   )\\
&<6\epsilon(3+c) + 2\epsilon \to 0,
\end{align*}
which concludes this proof.
\end{proof}

\subsubsection{2. A Generalized Version of Stone's Lemma}

To prove the universal consistency of the $k$-NN classifier, by Stone's Theorem, it is now sufficient to prove that this classifier satisfies the three conditions in Theorem \ref{stonetheorem}. However, its first condition requires a generalized version of Stone's Lemma. The original Stone's Lemma was first stated and proved by Stone in \cite{stone77} for the Euclidean space $(\mathbb{R}^m,||\cdot||) = (\mathbb{R}^m,||\cdot||_2)$, and was used to prove that the $k$-NN classifier is universally consistency in this space. The generalized version of this lemma, whose proof has not been found in literature before, is required when the domain is assumed to be any finite dimensional normed vector space, not necessarily Euclidean.

\begin{theo}[Generalized version of Stone's Lemma]
In a finite dimensional normed vector space $(\mathbb{R}^m,||\cdot||)$, let $f:\mathbb{R}^m\to\mathbb{R}$ be any measurable function and let $k\leq n$. Suppose $X,X_1,X_2,\ldots,X_n$ are independent and identically distributed from $\mu$. Then,
\[
\frac{1}{k}\sum_{i = 1}^k \mathbb{E}(|f(X_{(i)})|)\leq c\mathbb{E}(|f(X)|),
\]
where $X_{(1)},\ldots,X_{(n)}$ are ordered with respect to the distance induced by the norm $||\cdot||$ from $X = x$, as in (\ref{sortdist}), and $c$ is an absolute constant, which does not depend on $n$ nor $k$ (but may depend on the norm).
\end{theo}

\subsubsection{2. Proof of the Generalized Version of Stone's Lemma}

A covering lemma for $(\mathbb{R}^m,||\cdot||)$ is required in the proof of the generalized version of Stone's Lemma, and some additional notations used in this section are given below. For $x \in \mathbb{R}^m$ and radius $r\geq 0$, define respectively the closed ball, sphere, and open ball, centred at $x$ with radius $r$, by
\begin{align*}
B_{x,r} &= \{y \in \mathbb{R}^m: ||x - y||\leq r\}\\
S_{x,r} &= \{y \in \mathbb{R}^m: ||x - y|| = r\}\\
U_{x,r} &= \{y \in \mathbb{R}^m: ||x - y|| < r\}.
\end{align*}

\begin{lem}\label{finitedimcovering}
Let $(\mathbb{R}^m,||\cdot||)$ be a finite dimensional normed vector space. Then, we can write
\[
\mathbb{R}^m = \bigcup_{i = 1}^c A_i,
\]
so that for every $i = 1,2,\ldots,c$ and for all $x,y\in A_i$,
\[
||y||> ||x|| \Rightarrow ||y - x|| < ||y||.
\]
\end{lem}

\begin{proof}
Since $(\mathbb{R}^m,||\cdot||)$ has finite dimension, the unit sphere $S_{0,1} = \{y\in\mathbb{R}^m: ||y||= 1\}$ is compact. Therefore, we can cover $S_{0,1}$ by finitely many open balls:
\[
S_{0,1} \subseteq \bigcup_{i = 1}^c U_{x_i,1/2}.
\]
For each open ball $U_{x_i,1/2}$, define the set $A_i$ as follows: if $x\in\mathbb{R}^m\setminus\{0\}$, then
\[
x\in A_i \Leftrightarrow \frac{x}{||x||} \in U_{x_i,1/2},
\]
and we require that $0\in A_i$ for all $i = 1,2,\ldots,c$. It is clear that $\mathbb{R}^m = \displaystyle\bigcup_{i = 1}^c A_i$. Figure \ref{spherecover} gives a visualization for the construction of $A_i$ from $U_{x_i,1/2}$, but note that for a general finite dimensional normed space, the unit sphere and open balls need not be ``circular" (as in Euclidean space).
\begin{figure}
\begin{center}
\includegraphics[scale = 0.75]{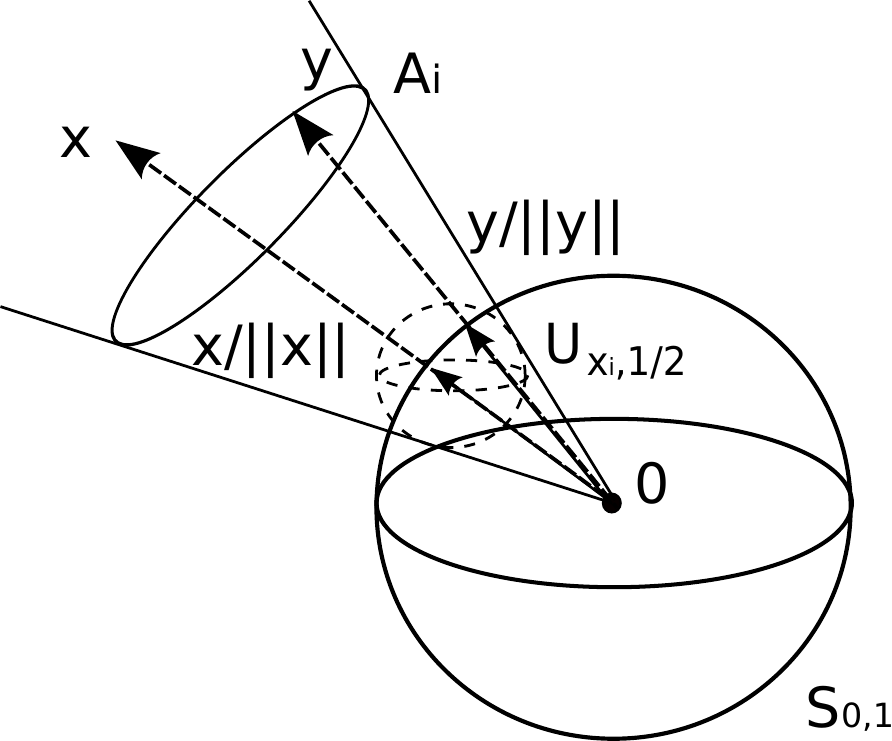}
\caption[Normalization of vectors in $(\mathbb{R}^m,||\cdot||)$]{Vectors $x,y\in\mathbb{R}^m$ are in $A_i$ if and only if their normalizations with respect to $||\cdot||$ are in $U_{x_i,1/2}$.}
\label{spherecover}
\end{center}
\end{figure}
For each $i = 1,2,\ldots,c$ and $x,y\in A_i$,
\[
\left|\left|\frac{x}{||x||} - \frac{y}{||y||}\right|\right|< 1,
\]
because if $x,y\in A_i$, then $x/||x||,y/||y||\in U_{x_i,1/2}$ and
\[
\left|\left|\frac{x}{||x||} - \frac{y}{||y||}\right|\right|\leq \left|\left|\frac{x}{||x||} -x_i\right|\right|+\left|\left| x_i-\frac{y}{||y||}\right|\right|<1/2 + 1/2 = 1.
\]
Now, we have to prove that for any $x,y\in A_i$, $||y|| > ||x||$ implies $||y - x|| < ||y||$. Suppose $||y||> ||x||$ and consider the vector $y||x||/||y||$. We have that
\begin{align*}
\left|\left| \frac{y||x||}{||y||} - x       \right|\right| &= \left|\left|||x||\left( \frac{y}{||y||} - \frac{x}{||x||}\right)    \right|\right|\\
&= ||x||\left|\left|   \frac{x}{||x||} - \frac{y}{||y||}\right|\right|\\
&< ||x||,
\end{align*}
since $||(x/||x|| - y/||y||)||<1$. Therefore,
\begin{align*}
||y - x|| &= \left|\left| y - \frac{y||x||}{||y||} + \frac{y||x||}{||y||} - x\right|\right|\\
&\leq \left|\left| y - \frac{y||x||}{||y||}\right|\right| + \left|\left|\frac{y||x||}{||y||} - x\right|\right|\\
&< \left|\left| y - \frac{y||x||}{||y||}\right|\right| + ||x||\\
& = \left|\left| y - \frac{y||x||}{||y||}\right|\right| + \left|\left|\frac{y||x||}{||y||}\right|\right|\\
& = ||y||,
\end{align*}
since $y||x||/||y||$ is on the line spanned by $y$. Figure \ref{triangulardia} provides a triangular diagram for illustrating that $||y|| > ||x||$ implies $||y - x|| < ||y||$.

\begin{figure}
\begin{center}
\includegraphics[scale = 1]{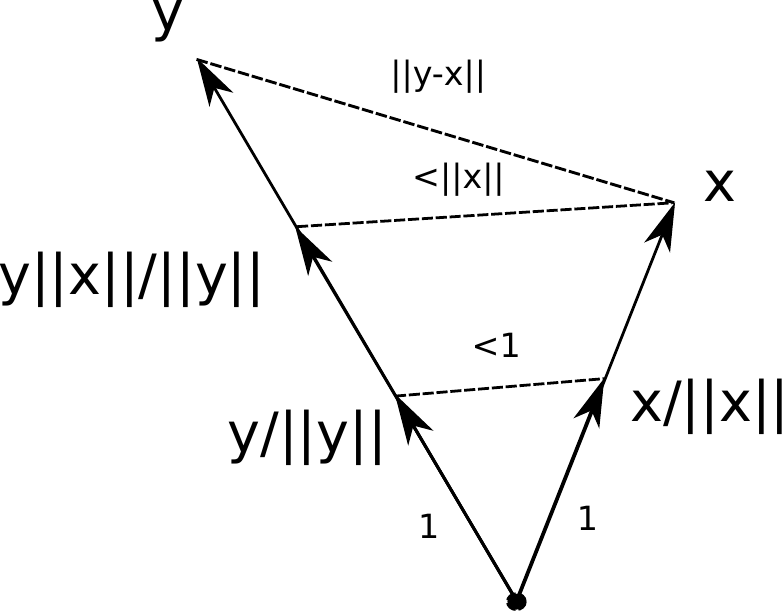}
\caption[Triangular illustration for a covering lemma for $(\mathbb{R}^m,||\cdot||)$]{This diagram illustrates the reason that $||y-x||< ||x||$ if $||y||>||x||$.}
\label{triangulardia}
\end{center}
\end{figure}

\end{proof}
The generalized version of Stone's Lemma now follows from Lemma \ref{finitedimcovering}.

\begin{proof}[Generalized version of Stone's Lemma]
First, define a function $\epsilon_{\textnormal{$k$-NN}}:\mathbb{R}^m \to \mathbb{R}^+$ by
\[
\epsilon_{\textnormal{$k$-NN}}(x) = \min\{r\geq 0: |\{i : X_i \in B_{x,r}\setminus\{x\}\}|\geq k\}.
\]
Let $k\in\mathbb{N}$ and let $X,X_1,X_2,\ldots,X_n\in\mathbb{R}^m$. By Theorem \ref{finitedimcovering}, we can write
\[
\mathbb{R}^m = \bigcup_{i = 1}^c A_i,
\]
such that for every $i = 1,2,\ldots,c$ and for all $x,y\in A_i$,
\[
||y||>||x|| \Rightarrow ||y - x|| < ||y||,
\]
where the constant $c$ does not depend on $n$ (since any two unit spheres are isometric between themselves). Translate the sets $A_i$ by $X$ and it is clear that $\{X+A_i: i = 1,\ldots,c\}$ still covers $\mathbb{R}^m$:
\[
\mathbb{R}^m = \bigcup_{i = 1}^c X + A_i.
\]

Furthermore, it is clear that for every $x,y\in X+A_i$, as $i = 1,2,\ldots,c$,
\begin{equation}\label{stonelemmaeq}
||y-X||> ||x-X|| \Rightarrow ||y - x||< ||y-X||
\end{equation}
The proof is exactly the same as Theorem \ref{finitedimcovering}, except with a translation by the vector $X$.

For each set $X+A_j$, mark the $k$-nearest neighbours of $X$ from the training sample $\{X_1,X_2,\ldots,X_n\}$ and if there are not enough such elements in $X+A_j$, mark all of them. Then, we have that
\[
|\{i: X\in B_{X_i,\epsilon_{\textnormal{$k$-NN}}(X_i)}\}|\leq |\{X_i:\textnormal{$X_i$ is marked}\}|\leq ck.
\]
The first inequality holds since if $X_i$ is not marked, then there are already at least $k$ elements in the set $\{X_1,\ldots,X_{i -1},X_{i+1},\ldots,X_n\}\cap(X +A_j)$ closer to $X$ than $X_i$. Hence, the distance between $X_i$ and each of those elements is less than the distance between $X_i$ and $X$ by (\ref{stonelemmaeq}). Then,
\begin{align*}
\sum_{i = 1}^k \mathbb{E}\left(|f(X_{(i)})|\right)&=\sum_{i = 1}^n \mathbb{E}\left(1_{\{X_i \textnormal{ is among the $k$ nearest neighbours of $X$ in $\{X_1,\ldots,X_n\}$}\}}|f(X_i)|\right)\\
& = \sum_{i = 1}^n \mathbb{E}\left(1_{\{X \textnormal{ is among the $k$ nearest neighbours of $X_i$ in $\{X_1,\ldots,X_{i-1},X,X_{i+1},\ldots,X_n\}$}\}}|f(X)|\right)\\
& \leq \sum_{i = 1}^n \mathbb{E}\left(1_{\{X \in B_{X_i,\epsilon_{\textnormal{$k$-NN}}(X_i)}\}}|f(X)|\right)\\
& = \mathbb{E}\left(|f(X)|\sum_{i = 1}^n 1_{\{X \in B_{X_i,\epsilon_{\textnormal{$k$-NN}}(X_i)}\}}\right)\\
& \leq \mathbb{E}\left(|f(X)| ck\right)\\
&=ck\mathbb{E}\left(|f(X)|\right).
\end{align*}
Hence,
\[
\frac{1}{k}\sum_{i = 1}^k \mathbb{E}\left(|f(X_{(i)})|\right) \leq c\mathbb{E}\left(|f(X)|\right).
\]
\end{proof}

\subsubsection{3. Proof of Consistency of the $k$-NN Classifier}

The universal consistency of the $k$-NN classifier follows from Stone's Theorem and the generalized version of Stone's Lemma.

\begin{proof}[Theorem \ref{consistencyknn}]
It suffices to prove that the assumptions for Stone's Theorem are valid for the $k$-NN classifier. Assumption 3 is true since $k\to \infty$:
\[
\max_{i = 1}^n W_{ni}(X) \leq 1/k \to 0.
\]
For assumption 2, we would like to prove that for all $a>0$,
\[
\mathbb{E}\left(  \frac{1}{k} \sum_{i = 1}^k 1_{\{||X_{(i)} - X||>a\}}     \right)\to 0.
\]
We have that
\begin{align*}
\mathbb{E}\left(  \frac{1}{k} \sum_{i = 1}^k 1_{\{||X_{(i)} - X||>a\}}     \right)& = \frac{1}{k}\sum_{i = 1}^k \mathbb{E}\left(1_{\{||X_{(i)} - X||>a\}} \right) \\
& = \frac{1}{k}\sum_{i = 1}^k \mathrm{Pr}(||X_{(i)} - X||>a)\\
& \leq \mathrm{Pr}(||X_{(k)} - X||>a).
\end{align*}
We note that $||X_{(k)} - X||>a$ if and only if
\[
\sum_{i = 1}^n 1_{\{X_i\in B_{X,a}\}} < k,
\]
or equivalently,
\[
\frac{1}{n}\sum_{i = 1}^n 1_{\{X_i\in B_{X,a}\}} < k/n\to 0,
\]
by assumption, but $\frac{1}{n}\sum_{i = 1}^n 1_{\{X_i\in B_{X,a}\}}$ converges to $\mu(B_{X,a})>0$ almost surely. As a result,
\[
\mathbb{E}\left(  \frac{1}{k} \sum_{i = 1}^k 1_{\{||X_{(i)} - X||>a\}}     \right)\leq \mathrm{Pr}(||X_{(k)} - X||>a)\to 0.
\]

Finally, assumption 1 is valid due to the generalized version of Stone's Lemma since
\begin{align*}
\mathbb{E}\left( \sum_{i = 1}^n W_{ni}(X)f(X_i)    \right) & = \left(\frac{1}{k}\right)\mathbb{E}\left( \sum_{i = 1}^k f(X_{(i)})    \right)\\
& \leq c\mathbb{E}(f(X)).
\end{align*}
\end{proof}

Altogether, Chapter \ref{supervisedalg} has introduced the $k$-NN classifier and demonstrated that it is universally consistent in any finite dimensional normed space. However, in general metric spaces, this classifier need not be consistent. For instance, C\'erou and Guyader in \cite{frenchguys} provide an example of a Gaussian distribution on an infinite-dimensional Hilbert space  where the $k$-NN classifier is not consistent. Chapter \ref{decisiontreesect} below explains another learning algorithm named the Random Forest classifier, based on Decision Trees.

\cleardoublepage

\chapter{The Random Forest Classifier}\label{decisiontreesect}

The goal of Chapter \ref{decisiontreesect} is to introduce Random Forest, one of the most popular supervised learning algorithms in data science \cite{rfbreiman}. Since Random Forest is a generalization of the Decision Tree classifier, the construction of a Decision Tree is explained as well. Section \ref{truedtree} discusses how Decision Tree trains on a labeled sample to predict labels for new observations. Section \ref{rfsection} then explains how the Random Forest classifier generalizes Decision Tree, by constructing multiples Decision Trees and building predictions based on the majority vote of these trees.

\section{The Decision Tree Classifier}\label{truedtree}

The Decision Tree classifier is a supervised learning algorithm, see e.g. \cite{machinelearn}, which builds a binary tree-like predictor based on a training sample. Section \ref{truedtree} explains this classifier in line with one of its most common implementations today, called the Classification and Regression Tree (CART) algorithm, developed by Breiman et al. in 1984 \cite{cart}.

 Given a labeled sample, the main idea of the Decision Tree classifier is to divide the sample into two disjoint subsets using an optimal binary splitting decision, at a specific coordinate (or feature) of the domain, according to some class homogeneity condition. This process then repeats for each of the two subsets in a recursive manner, and the recursion terminates when each divided subset contains only observations of the same class or when further divisions no longer improve class homogeneities. Each final subset is associated to a class label based on the mode of its observations' labels. For a new observation, the binary splitting decisions force it to one of the subsets, and its associated label becomes the observation's prediction.

Geometrically, the divisions of the training sample correspond to partitioning the domain into disjoint hyper-rectangles (possibly discrete or with sides of infinite length), parallel to the axes, based on the training sample, where each hyper-rectangle corresponds to a class label. The Decision Tree classifier predicts the label for a new observation by considering which hyper-rectangle the observation falls in.

\subsubsection{Formal Definition of a Decision Tree}

The following provides a complete and formal explanation of the Decision Tree classifier. First, the notion of class homogeneity is defined in terms of {\em entropy} and optimality of the binary splitting decision, at a coordinate $j$, is defined in terms of maximal {\em information gain}, see e.g. \cite{entropy}. Then, a complete procedure for the classifier is given, along with an example of a Decision Tree constructed from a real-life dataset. Because the Decision Tree classifier can be applied for observations from a domain with either discrete or continuous coordinates (or both), this section assumes that an observation $(X_i,Y_i)$ from a training sample $\mathcal{S}_\mathrm{lab} = \{(X_1,Y_1),(X_2,Y_2),\ldots,(X_n,Y_n)\}$ has coordinates
\[
X_i = (X_{i1},X_{i2},\ldots,X_{im}),
\]
where $X_{ij} \in  \mathbb{R}$ (continuous) or $X_{ij} \in \Omega$ for $|\Omega|<\infty$ (discrete). 

\begin{defn}
Given a set of labeled observations
\begin{align*}
\mathcal{S}_\mathrm{lab} &= \{(X_1,Y_1),(X_2,Y_2),\ldots,(X_n,Y_n)\},
\end{align*}
the {\em entropy $I_E$} of $\mathcal{S}_\mathrm{lab}$ is defined as
\begin{align*}
\quad\quad I_E(\mathcal{S}_\mathrm{lab}) &= -(f_0\log(f_0)+f_1\log(f_1)),
\end{align*}
where 
\begin{align*}
f_0 & = \frac{\mathrm{card}\{(X_i,Y_i)\in\mathcal{S}_\mathrm{lab} : Y_i = 0\}}{n}\\
f_1 & = \frac{\mathrm{card}\{(X_i,Y_i)\in\mathcal{S}_\mathrm{lab} : Y_i = 1\}}{n}.\\
\end{align*}
\end{defn}
Note that entropy is a measure of class homogeneity since it is minimal when a training sample only contains observations from a single class, and is maximal when the sample contains an equal number of observations from both classes. The Decision Tree classifier aims to recursively split a training sample into two subsets in a way that the weighted entropies of the two subsets are minimized, with respect to the entropy of the entire sample. The information gain for such a splitting condition measures exactly this change in entropy.

\begin{defn}
Given a set of labeled observations
\begin{align*}
\mathcal{S}_\mathrm{lab} &= \{(X_1,Y_1),(X_2,Y_2),\ldots,(X_n,Y_n)\},
\end{align*}
For $1 \leq j \leq m$ and $a\in\mathbb{R}$ (the continuous case) or $a\subseteq \Omega$ (the discrete case), write
\[
\mathcal{U}_{j,a} = \begin{cases}
\{(X_i,Y_i)\in\mathcal{S}_\mathrm{lab} : X_{ij} \leq a\} & \quad\textnormal{if $X_{ij}\in\mathbb{R}$}\\
\{(X_i,Y_i)\in\mathcal{S}_\mathrm{lab} : X_{ij} \in a\} &\quad\textnormal{if $X_{ij}\in\Omega$}
\end{cases}
\]
and
\[
\mathcal{V}_{j,a} = \begin{cases}
\{(X_i,Y_i)\in\mathcal{S}_\mathrm{lab} : X_{ij} > a\} & \quad\textnormal{if $X_{ij}\in\mathbb{R}$}\\
\{(X_i,Y_i)\in\mathcal{S}_\mathrm{lab} : X_{ij} \notin a\} &\quad\textnormal{if $X_{ij}\in\Omega$}.
\end{cases}
\]
Then, the {\em information gain $IG_{j,a}$} for $\mathcal{S}_\mathrm{lab}$, at the coordinate $j$ with splitting decision $a$, is defined by
 \[
IG_{j,a}(\mathcal{S}_\mathrm{lab}) =  I_E(\mathcal{S}_\mathrm{lab}) - \left[\left(\frac{|\mathcal{U}_{j,a}|}{n}\right)I_E(\mathcal{U}_{j,a}) + \left(\frac{|\mathcal{V}_{j,a}|}{n}\right)I_E(\mathcal{V}_{j,a})\right].
\]
\end{defn}

Here are the complete procedural steps of the Decision Tree classifier for both training and predicting:

\vspace{3.4mm}

\noindent{\em Training}
\vspace{-4mm}
\begin{enumerate}
\item Fix a minimal information gain threshold $\alpha\geq 0$.
\item Calculate the entropy $I_E(\mathcal{S}_\mathrm{lab})$ for a training sample $\mathcal{S}_\mathrm{lab}$.
\item Exhaustively determine the best splits $1\leq j \leq m$ and $a\in\mathbb{R}$ or $a\subseteq \Omega$, depending on whether the $j$'th coordinate (or feature) is continuous or discrete, such that $IG_{j,a}(\mathcal{S}_\mathrm{lab})$ is maximal.
\item Divide
\[
\mathcal{S}_\mathrm{lab} = \mathcal{S}_{\mathrm{lab},+} \cup \mathcal{S}_{\mathrm{lab},-}
\]
for $\mathcal{S}_{\mathrm{lab},+} = \mathcal{U}_{j,a}$ and $\mathcal{S}_{\mathrm{lab},-} = \mathcal{V}_{j,a}$.
\item Repeat Steps 2 to 5 for each of $\mathcal{S}_{\mathrm{lab},+}$ and $\mathcal{S}_{\mathrm{lab},-}$ recursively, until either termination conditions is met:
\begin{enumerate}
\item Both $\mathcal{S}_{\mathrm{lab},+}$ and $\mathcal{S}_{\mathrm{lab},-}$ respectively contain only observations from a single class.
\item The information gains $IG_{j,a}(\mathcal{S}_{\mathrm{lab},+})$ and $IG_{j,a}(\mathcal{S}_{\mathrm{lab},+})$ are less or equal to $\alpha$, for any $1\leq j\leq m$ and $a \in\mathbb{R}$ or $a\subseteq \Omega$.
\end{enumerate}

\item Upon termination, the training sample would have been partitioned into $w$ disjoint subsets
\[
\mathcal{S}_\mathrm{lab} = \mathcal{S}_{\mathrm{lab},1}\cup\mathcal{S}_{\mathrm{lab},2}\cup \ldots \cup \mathcal{S}_{\mathrm{lab},w},
\]
with best split conditions
\[
\{(j_1,a_1),(j_2,a_2),\ldots,(j_{w-1},a_{w-1})\}.
\]
Associate each subset $\mathcal{S}_{\mathrm{lab},i}$ to its most frequent class label.

\end{enumerate}
\noindent{\em Predicting}
\vspace{-4mm}
\begin{enumerate}
\setcounter{enumi}{6}
\item Given a new observation, determine which subset $\mathcal{S}_{\mathrm{lab},i}$ the observation would belong to, according to the split decisions $\{(j_1,a_1),(j_2,a_2),\ldots,(j_{w-1},a_{w-1})\}$ calculated in training, and predict it to have the subset's associated class label.
\end{enumerate}

Due to the recursive nature of the Decision Tree classifier, its procedure can be easily visualized and modelled as a tree, called the Decision Tree. The entire training sample starts at the parent node and is divided into two child nodes, corresponding to subsets of the training samples, based on the optimal binary split condition. Each child node is then divided into two further nodes and this process is repeated recursively until either of the termination conditions is satisfied for all the bottom nodes. Upon termination, each node at the final level becomes a leaf corresponding to the dominating class label.

For predicting the label of a new observation, the Decision Tree classifier simply passes the observation down the tree and the predicted class label is the one associated to the leaf this observation falls in. The next example is a Decision Tree constructed using the {\tt rpart} package \cite{rpart} with the statistical programming language R \cite{rcite} on a real-life voice recognition dataset.

\begin{figure}
\begin{center}
\includegraphics[scale = 1]{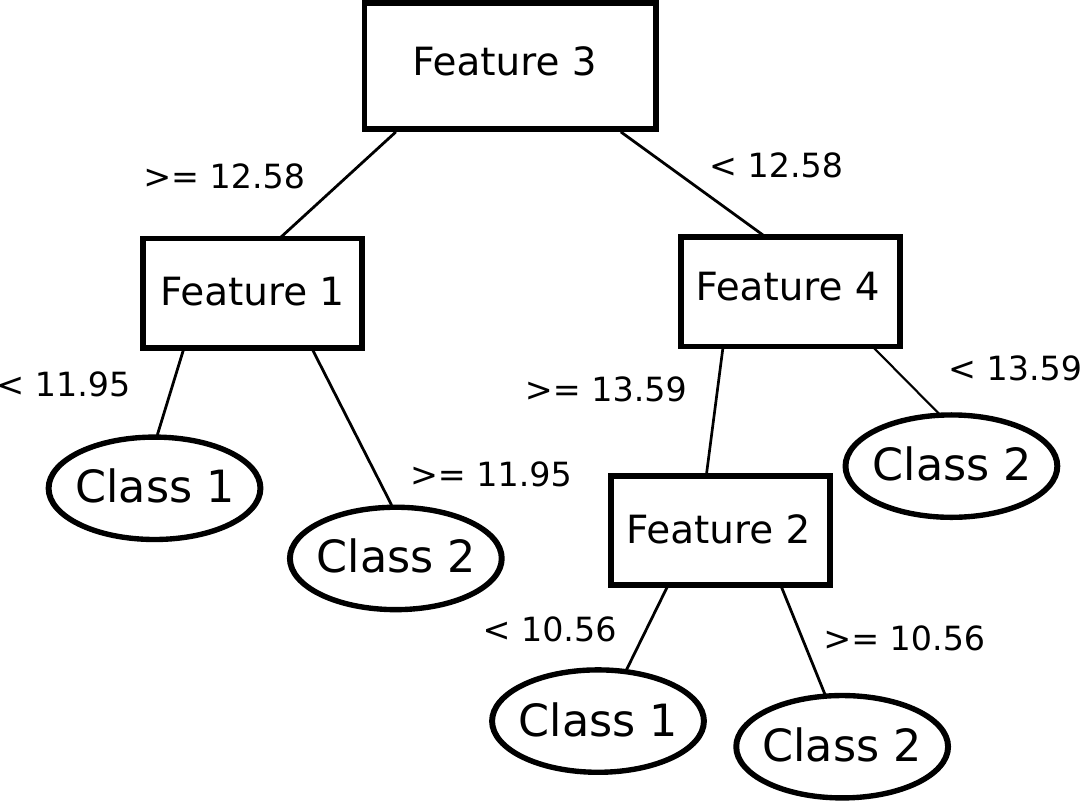}
\caption[Sample Decision Tree from phoneme dataset]{Decision Tree built from the example in Section \ref{truedtree}. The tree has 1 parent node (coordinate 3) with 3 child nodes (coordinates 1, 4, and 2) and 5 final leaves for the two possible classes.}
\label{dtdiag}
\end{center}
\end{figure}

\subsubsection{Example of a Decision Tree}

This example of constructing a Decision Tree is based on a subset of the real-life phoneme dataset studied in \cite{phoneme}. The dataset is used to predict phoneme sounds based on their discretized representations as vectors in $\mathbb{R}^{256}$. For simplicity of the example, only the first 6 dimensions of the representation vectors and 2 types of phoneme sounds, ``dcl" (Class 1) and ``sh" (Class 2), are considered for the example. Table \ref{dectreeexampletable} lists the class information for this dataset to build a Decision Tree, and Figure \ref{dtdiag} provides the actual Decision Tree constructed by {\tt rpart} in R. If a new observation is given by the vector
\[
(13.32, 2.31, 16.37, 18.20, 9.13, 11.00),
\]
it would be predicted to be from Class 2, corresponding to the ``sh" sound, because its third feature is greater than 12.58 while its first feature is greater than 11.95.

\begin{table}
\begin{center}
\caption[Phoneme dataset for Decision Tree]{Phoneme dataset's class information for constructing a Decision Tree, with two possible labels: ``dcl" and ``sh".}
\vspace{4mm}
\begin{tabular}{|c|c|}
\hline
Class & \# of observations\\
\hline\hline
1 (``dcl") & 757\\
\hline
2 (``sh") & 872\\
\hline
\end{tabular}
\label{dectreeexampletable}
\end{center}
\end{table}

\section{Random Forest from Multiple Decision Trees}\label{rfsection}

The Random Forest classifier, developed by Breiman in 1994 \cite{rfbreiman}, generalizes the Decision Tree classifier in two important steps. First, Random Forest uses the method of {\em bootstrap aggregation}, or {\em bagging} \cite{bagging}, by taking $t$ bootstrap samples with replacement of the training sample $\mathcal{S}_\mathrm{lab} = \{(X_1,Y_1),(X_2,Y_2),\ldots,(X_n,Y_n)\}$ and constructing a Decision Tree for each bootstrap sample:
\begin{align*}
\mathcal{S}_{\mathrm{boot},1} & = \{(X_1^{*1},Y_1^{*1}),(X_2^{*1},Y_2^{*1}),\ldots,(X_n^{*1},Y_n^{*1})\} &\longmapsto \quad\quad\textnormal{Decision Tree 1}\\
\mathcal{S}_{\mathrm{boot},2} & = \{(X_1^{*2},Y_1^{*2}),(X_2^{*2},Y_2^{*2}),\ldots,(X_n^{*2},Y_n^{*2})\}&\longmapsto \quad\quad\textnormal{Decision Tree 2}\\ 
&\ldots\\
\mathcal{S}_{\mathrm{boot},t} & = \{(X_1^{*t},Y_1^{*t}),(X_2^{*t},Y_2^{*t}),\ldots,(X_n^{*t},Y_n^{*t})\}&\longmapsto \quad\quad\textnormal{Decision Tree $t$}
\end{align*}

\begin{figure}
\begin{center}
\includegraphics[scale = 0.5]{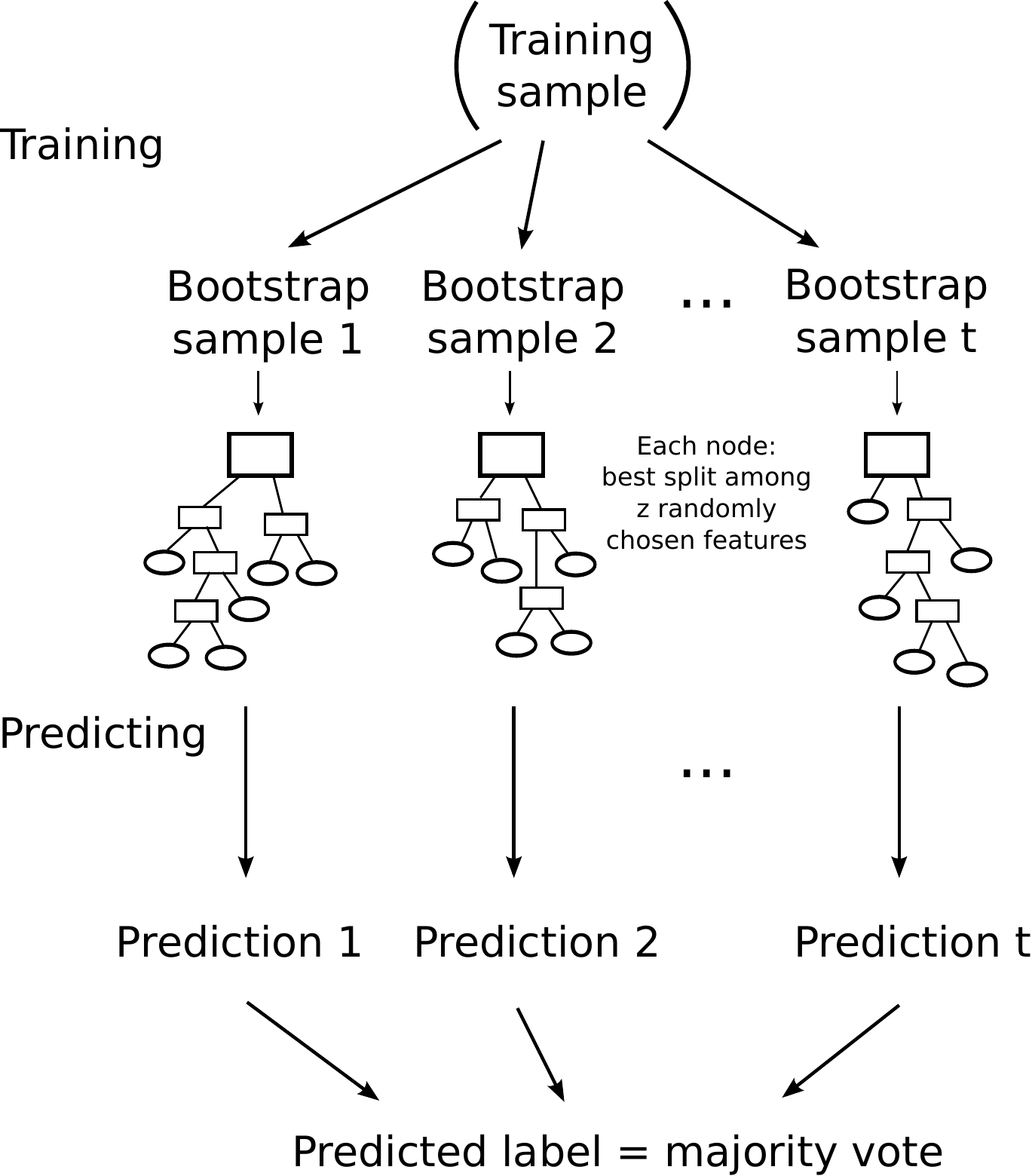}
\caption[Visualization of the Random Forest classifier]{This diagram visualizes the training step for the Random Forest classifier, of taking bootstrap samples to generate multiple Decision Trees, and the predicting step for a new observation.}
\label{rfdiag}
\end{center}
\end{figure}

Second, for the construction of each Decision Tree, at each iteration or node, only a random selection of $z$ coordinates (instead of all the coordinates), for some fixed $z$, are considered for finding the best split of the training sample from the previous step. In other words, regarding Step 3 of the procedure for building a Decision Tree, a subset $A\subseteq \{1,2,\ldots,m\}$ of cardinality $|A| = z$ is randomly selected and the best split among $j\in A$ (instead of $1 \leq j \leq m$) and $a\in \mathbb{R}$ or $a\subseteq\Omega$ are chosen. For a new observation $X$, each Decision Tree makes a prediction, denoted as $T_i(X)$, and the overall prediction for Random Forest is the mode of these individual predictions:
\[
\mathrm{label}\,(X) = \mathrm{mode}\{T_i(X):i = 1,2,\ldots,t\}.
\]
Figure \ref{rfdiag} illustrates the Random Forest classifier for both training and predicting. In total, the Random Forest classifier depends on two main parameters in training: the number $t$ of Decision Trees to generate and the number $z$ of randomly selected features, whose default value is $z=\sqrt{m}$, used for each splitting decision. Optimal values of these parameters are usually found through a validation process explained in Section \ref{crossvalid}.

The Random Forest classifier, in practice, has excellent predictive abilities and has demonstrated its effectiveness in many fields and applications. Moreover, it has the ability to easily handle training observations from continuous or discrete (or a mixture of both) domains, since it is a generalization of the Decision Tree classifier. Theoretically, however, Random Forest's predictive abilities are not justified. In fact, Devroye et al. in \cite{rfconsistent} showed that this classifier is, in general, not universally consistent. This seemingly bizarre fact illustrates a common distinction between theory and practice for supervised learning algorithms. Often, a classifier with proven theoretical properties, such as being universally consistent, may not actually produce accurate predictions in real-life, or its computational complexity may be too high in practice. On the other hand, a simple and intuitive classifier, with absolutely no theory supporting it, may have excellent predictive abilities and can run efficiently.

In short, Chapter \ref{decisiontreesect} has explained two supervised learning algorithms, the Decision Tree and Random Forest, which generalizes the former, classifiers in detail. The next chapter introduces dimensionality reduction, which is an important field of study in data science as big data can be extremely high-dimensional and learning algorithms often struggle to train and predict efficiently on such data.

\cleardoublepage

\chapter{Dimensionality Reduction}\label{dimred}

Chapter \ref{dimred} discusses the important concept of dimensionality reduction for a dataset prior to classification. Section \ref{featintro} introduces this concept and explains the difference between feature selection and extraction. Sections \ref{rpsect} and \ref{mtdsect} discuss two such techniques, Random Projections, a known method based on the popular Johnson-Lindenstrauss Lemma, and MTD Feature Selection, a novel method developed from the theory of Mass Transportation Distance.

\section{Introduction to Feature Extraction and Selection}\label{featintro}

Big data today often have high dimensions; for instance, the genetic dataset considered for this thesis includes 3907 observations, each having 865688 coordinates, corresponding to the number of Single-Nucleotide Polymorphisms. Running a supervised learning algorithm on a high-dimensional dataset is computationally expensive. Even worse, such an algorithm may not be capable of running at all, due to memory and storage constraints. Techniques that reduce dimension are often run on a high-dimensional dataset prior to classification, to simplify complexity and save computational costs, while at the same time attempt to preserve the predictive information of the original dataset \cite{dmreview}. In other words, when the dimension (i.e. the number of coordinates or features) $m$ of the domain $\Omega$ is high, a function $T:\Omega \to \Omega'$, called a {\em dimensionality reduction}, is defined from $\Omega$ to another domain $\Omega'$ with a much lower dimension $m' << m$. A classifier $g:\Omega'\to\{0,1\}$ is then applied on the reduced space:
\[
g\circ T: \Omega \to \Omega' \to \{0,1\}.
\]
For a training sample $\mathcal{S}_\mathrm{lab} = \{(X_1,Y_1),(X_2,Y_2),\ldots,(X_n,Y_n)\}$ and a new observation $X$, where $X,X_i \in \Omega$, the classifier $g$ would consider $T(X)$ and the image of $\mathcal{S}_\mathrm{lab}$ under $T$,
\[
\mathcal{S}_\mathrm{lab}^\mathrm{red} = \{(X_1',Y_1),(X_2',Y_2),\ldots,(X_n',Y_n)\},
\]
for $X_i' = T(X_i) \in\Omega'$, in order to predict the label for $X$:
\[
\left(\mathcal{S}_\mathrm{lab}, X\right)\longmapsto\left(\mathcal{S}_\mathrm{lab}^\mathrm{red}, T(X)\right) \longmapsto g(T(X)) = \textnormal{label of $X$}.
\]

There are generally two approaches for defining the function $T$, called {\em feature selection} and {\em feature extraction} \cite{dmreview}. For feature selection, a score of importance $S(j)$ is first assigned to each coordinate $j$ of the domain $\Omega$. The dimension reduction map $T$ then projects observations from $\Omega$ onto the highest scored coordinates, according to some threshold $\alpha$. If $\Omega$ has $m$ dimensions and coordinates $J \subseteq \{1,2,\ldots,m\}$ are determined to be important, for $J = \{j \,| \,S(j)\geq \alpha\}$, the map $T: \Omega\to\Omega'$, where $\Omega'$ has $|J|$ coordinates, would be defined by
\[
\Omega \ni X = (X_1,X_2,\ldots,X_m) \longmapsto (X_j \,|\, j \in J) \in \Omega',
\]
For instance, if $m = 5$ and $J = \{1,4,5\}$, then $X = (X_1,X_2,X_3,X_4,X_5) \longmapsto (X_1,X_4,X_5)$. The identifying property of feature selection is that the coordinates in the reduced domain form a subset of the original coordinates.

On the other hand, a feature extraction map $T: \Omega \to \Omega'$ reduces dimension by transforming observations from $\Omega$ in more complicated ways than simple coordinate projections. For example, the map $T: (0,1)^2 \to (0,1)$ defined by interchanging decimal expansions,
\[
T(0.a_1a_2a_3\ldots,0.b_1b_2b_3\ldots) = (0.a_1b_1a_2b_2\ldots),
\]
is a feature extraction method, known as the {\em Borel Isomorphic Dimensionality Reduction Method}, introduced in \cite{knncurse} and further studied in \cite{stan}. This method can be easily generalized to a feature extraction map from $(0,1)^m$ to $(0,1)^{m'}$. Other feature extraction techniques include Principal Component Analysis (PCA) and Linear Discriminant Analysis (LDA), which are dimension reduction methods based on linear transformations \cite{dmreview}.

Section \ref{rpsect} explains a feature extraction method called Random Projections where the reduction map is a linear function from $\mathbb{R}^m$ to $\mathbb{R}^{m'}$ defined by matrix multiplication. Section \ref{mtdsect} introduces a new feature selection method for a discrete domain $\Omega$, where each coordinate is assigned an importance score based on the Mass Transportation Distance.

\section{Random Projections}\label{rpsect}

This section explains Random Projections, which is a known feature extraction method, see e.g. \cite{rpcite}, based on the following theorem proved by Johnson and Lindenstrauss in 1984 \cite{jllemmaorig}.

\begin{theo}[Johnson-Lindenstrauss Lemma \cite{jllemmaorig}]\label{origjl}
Let $0<\epsilon<1/2$ and let $S\subseteq \mathbb{R}^m$ be any finite subset with $|S| = n$. If $m' \geq C\ln(n)/\epsilon^2$ is any integer, for some sufficiently large absolute constant $C$, there exists a linear map $T:\mathbb{R}^m\to\mathbb{R}^{m'}$ such that
\[
(1-\epsilon)||x - y||_2 \leq ||T(x) - T(y)||_2 \leq (1+\epsilon)||x - y||_2
\]
for all $x,y\in S$.
\end{theo}

\begin{figure}
\begin{center}
\includegraphics[scale =0.75]{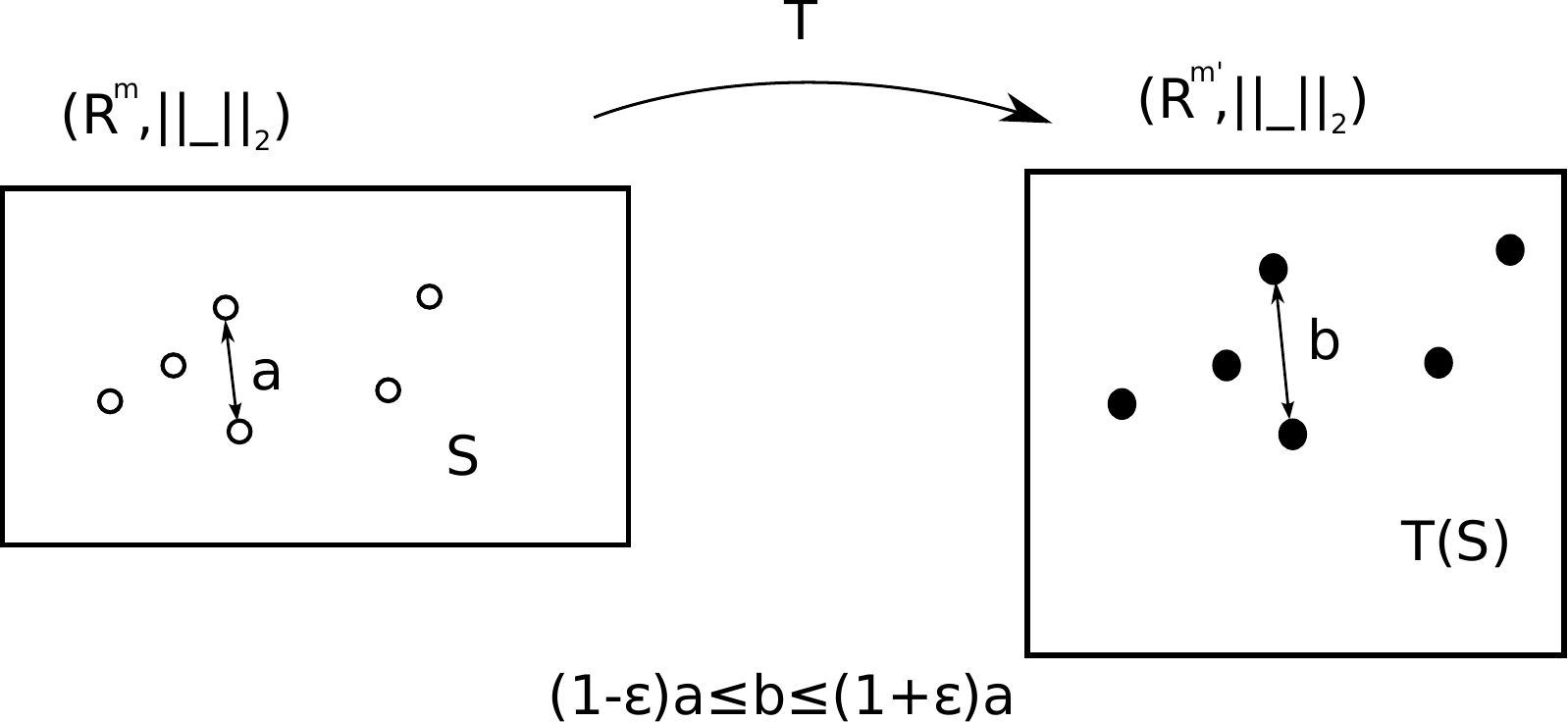}
\caption[Preservation of pairwise distances in $\mathbb{R}^m$]{This figure illustrates a map $T$ that can only distort distances between elements of $S$ up to $1\pm\epsilon$.}
\label{jllemmapic}
\end{center}
\end{figure}

Figure \ref{jllemmapic} provides a visualization of such a linear map and its preservation of distances. In the context of learning theory, distance-based classifiers, such as $k$-NN, have high computational costs for calculating distances between observations in $\mathbb{R}^m$, if $m$ is very large. A linear map that could project these observations to a lower dimensional space $\mathbb{R}^{m'}$, while still preserves their pairwise $\ell_2$ distances up to a factor of $1\pm \epsilon$ as guaranteed by Theorem \ref{origjl}, would allow these classifiers to run more efficiently on the simpler space. In fact, a probabilistic proof of Theorem \ref{origjl}, found in  \cite{variantjl}, provides a constructive method of finding such a map $T$, defined by multiplication via a randomly generated matrix $\frac{1}{\sqrt{m'}}(R_{ij})$, whose entries $R_{ij}$ follow a binary distribution, taking on values in $\{-1,1\}$ with equal probabilities.

Altogether, Random Projections as a feature extraction method works as follows.  Consider the domain $(\mathbb{R}^m,||\cdot||_2)$, pick some lower dimension $m'<< m$, and generate a random matrix $\frac{1}{\sqrt{m'}}(R_{ij})$ from the required binary distribution. Define the dimension reduction map $T: \mathbb{R}^m\to\mathbb{R}^{m'}$ by matrix multiplication by $\frac{1}{\sqrt{m'}}(R_{ij})$:
\[
T(x) = \frac{1}{\sqrt{m'}}(R_{ij}) \cdot x.
\]
Then, map a training sample $\mathcal{S}_\mathrm{lab} = \{(X_1,Y_1),(X_2,Y_2),\ldots,(X_n,Y_n)\}$, along with any new observation $X$, from $\mathbb{R}^m$ to the reduced space $\mathbb{R}^{m'}$, and denote their images by $\mathcal{S}_\mathrm{lab}^\mathrm{red}$ and $T(X)$, respectively. A distance-based classifier for $(\mathbb{R}^{m'},||\cdot||_2)$ can now train on $\mathcal{S}_\mathrm{lab}^\mathrm{red}$ and predict a label $g(T(X))\in\{0,1\}$ for the original new observation $X\in\mathbb{R}^m$.

As found in \cite{variantjl}, Section \ref{jllemmaoutline} concentrates on the outline of a proof for the Johnson-Lindenstrauss Lemma, which is included in the thesis to justify the use of Random Projections via random matrix multiplication in supervised learning theory and data science.

\subsection{Probabilistic Proof of the Johnson-Lindenstrauss Lemma}\label{jllemmaoutline}

Section \ref{jllemmaoutline} outlines a probabilistic proof of the Johnson-Lindenstrauss Lemma, as provided by Matousek in \cite{variantjl} as the main result. The importance of this proof is that it provides a constructive method of finding the linear map $T$, as multiplication by a randomly generated matrix, for dimensionality reduction. In his paper, Matousek first defines the notion of {\em sub-gaussian tails}.

\begin{defn}
Let $X$ be a real random variable with $\mathbb{E}(X) = 0$.
\begin{enumerate}
\item The variable $X$ has a {\em sub-gaussian upper tail} if there exists a constant $a>0$ such that for all $\lambda>0$,
\begin{equation}\label{subgtail}
\mathrm{Pr}(X>\lambda) \leq e^{-a\lambda^2}.
\end{equation}
\item The variable $X$ has a {\em sub-gaussian upper tail up to $\lambda_0$} if (\ref{subgtail}) holds for all $\lambda\leq \lambda_0$.
\item The variable $X$ has a {\em sub-gaussian tail} if $X$ and $-X$ both have sub-gaussian upper tails.
\end{enumerate}
A collection $X_1,X_2,\ldots,X_m$ has a {\em uniform sub-gaussian tail} if each random variable $X_i$ has a sub-gaussian tail with the same constant $a$.
\end{defn}

The main result of Matousek's paper \cite{variantjl}, implying the Johnson-Lindenstrauss Lemma, is that the randomly generated matrix for defining the distance-preservation map $T$ can have entries from a random variable with uniform sub-gaussian tail.

\begin{theo}\label{probjltheo}
Let $m\in\mathbb{N}$, $0<\epsilon<1/2$, $0<\delta<1$, and $m' \geq C\ln(\delta/2)/\epsilon^2$, where $C$ is some absolute constant. Define a random linear map $T:\mathbb{R}^m\to\mathbb{R}^{m'}$ by
\begin{align*}
T(x) = \frac{1}{\sqrt{m'}}(R_{ij})\cdot x,
\end{align*}
where $(R_{ij})$ consists of independent binary random variables $R_{ij}$ such that 
\begin{multicols}{2}
\begin{enumerate}
\item Every $R_{ij}$ has a uniform sub-gaussian tail;
\item $\mathbb{E}(R_{ij}) = 0$;
\item $\mathrm{Var}(R_{ij}) = 1$.
\end{enumerate}
\end{multicols}
Then, for each $x\in\mathbb{R}^m$,
\[
\mathrm{Pr}[(1-\epsilon)||x||_2\leq ||T(x)||_2\leq (1+\epsilon)||x||_2]\geq 1 - \delta.
\]
\end{theo}
\noindent In particular, random variables $R_{ij}$ taking binary values $\{-1,1\}$, with equal probabilities, satisfy the three requirements in Theorem \ref{probjltheo}. From this probabilistic result, the Johnson-Lindenstrauss Lemma follows as an easy corollary.

\begin{proof}[Johnson-Lindenstrauss Lemma]
Suppose $T$ is any random linear map defined as matrix multiplication by $(\frac{1}{\sqrt{m'}})(R_{ij})$, whose entries satisfy the conditions from Theorem \ref{probjltheo}. Let $\delta = 1/n^2$ and for any $x,y\in S$, the probability, over all such possible maps $T$, that the next equation does not hold is at most $1/n^2$:
\[
(1-\epsilon)||x-y||_2\leq ||T(x-y)||_2= ||T(x) - T(y)||_2\leq (1+\epsilon)||x-y||_2.
\]
Since there are a total of ${\bigg(}{\small \begin{array}{c}n\\2\end{array}}{\bigg)}$ distinct choices of $x$ and $y$ from $S$, the probability that $T$  preserve all pairwise distances in $S$ is at least
\[
1 - (1/n^2){\bigg(}{\small \begin{array}{c}n\\2\end{array}}{\bigg)} > 1/2 > 0,
\]
so such a linear map $T$ exists.
\end{proof}

Because the lower dimension $m'$ depends on the distance preservation constant $\epsilon$ and a sufficiently large constant $C$ in Theorems \ref{origjl} and \ref{probjltheo}, it can of course be hard to calculate in practice. In addition, multiplication by a randomly generated matrix $(\frac{1}{\sqrt{m'}})(R_{ij})$ is only guaranteed to preserve all pairwise distances up to $1\pm\epsilon$ with some high enough probability (greater than $1/2$), so not all such matrices would be satisfactory. Multiple values of $m'$ and generated matrices are hence usually considered in practice, and a validation process on the training sample is run to determine the optimal value $m'$ and matrix $(\frac{1}{\sqrt{m'}})(R_{ij})$ for classification purposes. See Section \ref{crossvalid} for more details on the validation of a classifier and parameter selection. The remaining part of Section \ref{jllemmaoutline} is devoted to outlining the proof of Theorem \ref{probjltheo}, as done by Matousek in \cite{variantjl}.

To prove Theorem \ref{probjltheo}, Matousek first relates the concept of sub-gaussian tails with the moment-generating function.

\begin{lem}\label{momimplysub}
Let $X$ be a random variable with $\mathbb{E}(X) = 0$.
\begin{enumerate}
\item If
\begin{equation}\label{momentgen}
\mathbb{E}(e^{uX}) \leq e^{Cu^2} 
\end{equation}
for some constant $C$ and all $u>0$, then $X$ has a sub-gaussian upper tail.
\item If (\ref{momentgen}) holds for all $0< u\leq u_0$, then $X$ has a sub-gaussian upper tail up to $2Cu_0$.
\end{enumerate}
Conversely, if $X$ is a random variable with $\mathbb{E}(X) = 0$ and $\mathrm{Var}(X) = \mathbb{E}(X^2) = 1$, and suppose $X$ has a sub-gaussian upper tail, with constant $a$. Then,
\[
\mathbb{E}(e^{uX}) \leq e^{Cu^2} 
\]
for all $u>0$, where $C$ is some constant depending only on $a$.
\end{lem}

By Lemma \ref{momimplysub}, certain linear combinations of random variables with a uniform sub-gaussian tail also have sub-gaussian tails.

\begin{lem}\label{sumlemma}
Let $X_1,X_2,\ldots,X_m$ be independent random variables such that $\mathbb{E}(X_i) = 0$ and $\mathrm{Var}(X_i) = 1$, for all $i = 1,2,\ldots,m$, and suppose $X_1,X_2,\ldots,X_m$ have a uniform sub-gaussian tail. Suppose $\alpha_1^2 + \alpha_2^2 + \ldots + \alpha_m^2 = 1$ are real constants, then
\[
Y = \alpha_1X_1 + \alpha_2X_2 + \ldots + \alpha_mX_m
\]
has $\mathbb{E}(Y) = 0$, $\mathrm{Var}(Y) = 1$, and a sub-gaussian tail.
\end{lem}
\begin{proof}
It is clear that $Y$ satisfies $\mathbb{E}(Y) = 0$ and $\mathrm{Var}(Y) = 1$, since $X_1,X_2,\ldots,X_m$ do and are independent. As $X_1,X_2,\ldots,X_m$ have a uniform sub-gaussian tail, $\mathbb{E}(e^{uX_i}) \leq e^{Cu^2}$ for each $i = 1,2,\ldots,m$, by Lemma \ref{momimplysub}; hence,
\[
\mathbb{E}(e^{uY}) = \prod_{i = 1}^m \mathbb{E}(e^{u\alpha_iX_i})\leq \prod_{i = 1}^m e^{Cu^2\alpha_i^2} = e^{Cu^2},
\]
so $Y$ has a sub-gaussian tail, also by Lemma \ref{momimplysub}.
\end{proof}

Then, Matousek proves the following, which implies Theorem \ref{probjltheo}.
\begin{lem}\label{mainlemma}
Suppose $m'\geq 1$ and $Y_1,Y_2,\ldots,Y_{m'}$ are independent random variables, where $\mathbb{E}(Y_i) = 0$, $\mathrm{Var}(Y_i) = 1$, and each $Y_i$ has a uniform sub-gaussian tail. Then
\[
Z = \frac{1}{\sqrt{m'}}(Y_1^2 + Y_2^2 + \ldots + Y_{m'}^2 - m')
\]
has a sub-gaussian tail up to $\sqrt{m'}$.
\end{lem}

\begin{proof}[Theorem \ref{probjltheo}]
Suppose $x\in\mathbb{R}^m$ has unit Euclidean length, $\sum_{i = 1}^m x_i^2 = 1$, and write
\[
Y_i = \sum_{j = 1}^m R_{ij}x_j.
\]
By Lemma \ref{sumlemma}, $Y_i$ satisfies $\mathbb{E}(Y_i) = 0$, $\mathrm{Var}(Y_i) = 1$, and each $Y_i$ has a sub-gaussian tail. Then, by Lemma \ref{mainlemma}, $Z = \frac{1}{\sqrt{m'}}(Y_1^2 + Y_2^2 + \ldots + Y_{m'}^2 - m')$ has a sub-gaussian tail up to $\sqrt{m'}$. Hence,
\begin{align*}
\mathrm{Pr}[||T(x)||_2 \geq 1 + \epsilon] &\leq \mathrm{Pr}[||T(x)||^2_2 \geq 1 + 2\epsilon]\\
& \leq \mathrm{Pr}[Z \geq 2\epsilon \sqrt{m'}]\\
& \leq e^{-4a\epsilon^2 C\epsilon^{-2}\log(2/\delta)}  \leq \delta / 2,
\end{align*}
as long as $C$ is sufficiently large. Similarly, $\mathrm{Pr}[||T(x)||_2 \leq 1 - \epsilon]\leq \delta/2$ and Theorem \ref{probjltheo} is proved.
\end{proof}

In summary, Random Projections is a feature extraction technique for the domain $(\mathbb{R}^m,||\cdot||_2)$, where a training sample and any new observations for classification are transformed to a lower dimensional space $(\mathbb{R}^{m'},||\cdot||_2)$ via matrix multiplication. A distance-based classifier can then be efficiently applied on the reduced space to save computational costs. Analogous results of the Johnson-Lindenstrauss Lemma and Theorem \ref{probjltheo} exist as well for the $\ell_1$ norm in $\mathbb{R}^{m'}$, also proved by Matousek in \cite{variantjl}. Thus, a distance-based classifier in the reduced space can use either the $\ell_1$ or $\ell_2$ norm for classification. 

Section \ref{mtdsect} below introduces a new feature selection method for discrete domains.

\section{Mass Transportation Distance}\label{mtdsect}

This section explains a new feature selection technique called {\em Mass Transportation Distance Feature Selection}, or {\em MTD Feature Selection}. The method is based on the popular {\em Mass Transportation Distance}, also known as the {\em Earth Mover's distance} and the {\em Wasserstein distance}, originally introduced by Kantorovich in 1942 \cite{mtdorig}. Sections \ref{mtdsect} and \ref{mtdlip} introduce this distance in a general context and then explain its relevance as a feature selection method in data science.

Given a metric space $(\Omega,d)$, the Mass Transportation Distance is a metric defined on the space of finitely-supported probability measures on $\Omega$.  If $\mu$ and $\mu'$ are any two such probability measures, then the Mass Transportation Distance ${\hat d}$ of $\mu$ and $\mu'$ is defined as follows \cite{supestov}:
\begin{equation}\label{mtdequ}
{\hat d}(\mu,\mu') = \inf_{\nu}\int_{\Omega\times \Omega} d(x,y) \, d\nu,
\end{equation} 
where the infimum is taken over all probability measures $\nu$ on $\Omega\times \Omega$ such that the marginals of $\nu$ are $\mu$ and $\mu'$ respectively. The distance ${\hat d}(\mu,\mu')$ can be thought of as the minimal cost of moving the supported ``masses" of the probability measure $\mu$ to $\mu'$, with respect to the underlying distance $d$. 

In general, the infimum in (\ref{mtdequ}) is extremely difficult to compute exactly, but in the specific case where $\Omega = \{\omega_1,\omega,\ldots,\omega_w\}$ is a finite set with $w$ elements and $d$ is the discrete $\{0,1\}$-distance, where $d(x,y) = 0$ if $x= y$ and $d(x,y) =1 $ else, the equation (\ref{mtdequ}) has a much simpler form. On such a space $\Omega$, known as a {\em discrete (metric) space}, any probability measure $\mu$ is discrete:
\[
\mu = \sum_{i = 1}^w p_i\delta_{\omega_i},
\]
where
\[
\delta_{\omega_i}(C) = \begin{cases}
1 & \quad\textnormal{ if $\omega_i \in C$;}\\
0 & \quad\textnormal{ otherwise,}
\end{cases} 
\]
for $C\subseteq \Omega$, and $p_i\geq 0$ satisfying $\sum_{i = 1}^w p_i = 1$. Then, the Mass Transportation Distance between two probability measures $\mu,\mu'$ simplifies to the $\ell_1$ distance:
\[
{\hat d}(\mu,\mu') = || \mu - \mu'||_1 = \sum_{i = 1}^w |p_i - q_i|
\]
for $\mu = \sum_{i = 1}^w p_i\delta_{\omega_i}$ and $\mu' = \sum_{i = 1}^w q_i\delta_{\omega_i}$, see e.g. \cite{supestov}.

As a result of this simplification, the Mass Transportation Distance can be used as a feature selection technique in data science, when the domain in question is a product of finite discrete spaces. Suppose $\Omega^m$, where $(\Omega = \{\omega_1,\omega_2,\ldots,\omega_w\},d)$ is a finite discrete space, is the domain with some product distance (e.g. the Hamming distance). Two probability measures $\mu_0$ and $\mu_1$ model the theoretical distributions of observations $X\in\Omega^m$ with label 0 and 1, respectively, but note that these measures only exist in theory and cannot be computed exactly in practice. For a fixed coordinate $ 1 \leq j \leq m$ and a training sample $\mathcal{S}_\mathrm{lab} = \{(X_1,Y_1),(X_2,Y_2),\ldots,(X_n,Y_n)\}$, containing observations 
\[
X_i  = (X_{i1},X_{i2},\ldots,X_{im})\in\Omega^m
\]
for $X_{ij} \in\Omega$ with label $Y_i \in \{0,1\}$, the projection map $\pi_j:\Omega^m\to \Omega$ maps each $X_i$ to its $j$'th coordinate:
\[
\pi_j(X_i) = X_{ij} \in \Omega.
\]
Consequently, the projection of the training sample onto the $j$'th coordinate induces two empirical probability measures ${\hat \mu_0^j}$ and ${\hat \mu_1^j}$ on $\Omega$ defined by
\[
{\hat \mu_{0}^j} = \sum_{i = 1}^w p_0(\omega_i)\delta_{\omega_i} \quad \textnormal{and} \quad {\hat \mu_{1}^j} = \sum_{i = 1}^w p_1(\omega_i)\delta_{\omega_i},
\]
where
\begin{align*}
p_0(\omega) &= \frac{\mathrm{card}\{i \,| \,(\pi_j(X_i), Y_i) = (\omega, 0)\}}{\mathrm{card}\{i \, | \, Y_i = 0\}} = \frac{\mathrm{card}\{i \,| \,(X_{ij}, Y_i) = (\omega, 0)\}}{\mathrm{card}\{i \, | \, Y_i = 0\}}\\
p_1(\omega) &= \frac{\mathrm{card}\{i \,| \,(\pi_j(X_i), Y_i) = (\omega, 1)\}}{\mathrm{card}\{i \, | \, Y_i = 1\}}= \frac{\mathrm{card}\{i \,| \,(X_{ij}, Y_i) = (\omega, 1)\}}{\mathrm{card}\{i \, | \, Y_i = 1\}}.
\end{align*}
These probability measures ${\hat \mu_0^j}$ and ${\hat \mu_1^j}$ can be thought as estimations of $\mu_0$ and $\mu_1$, with respect to the $j$'th coordinate, and their Mass Transportation Distance can be calculated exactly:
\begin{align*}
{\hat d({\hat \mu_{0}^j},{\hat \mu_{1}^j})} &= || {\hat \mu_{0}^j} - {\hat \mu_{1}^j}||_1 \\
& = \sum_{i = 1}^w |p_0(\omega_i) - p_1(\omega_i)|.
\end{align*}

Figure \ref{mtdpic} provides an illustration of the Mass Transportation Distance when the two empirical probability measures are interpreted as histograms supported on elements of the finite discrete domain $(\Omega,d)$. The Mass Transportation Distance ${\hat d({\hat \mu_{0}^j},{\hat \mu_{1}^j})}$ provides an estimate of the separation between the true underlying distributions $\mu_0$ and $\mu_1$, projected to the $j$'th coordinate. Hence, only coordinates with high Mass Transportation Distances should be deemed as important, in terms of sufficient separation between the observations from class 0 and 1, and be used for classification.

\begin{figure}
\begin{center}
\includegraphics[scale = 1]{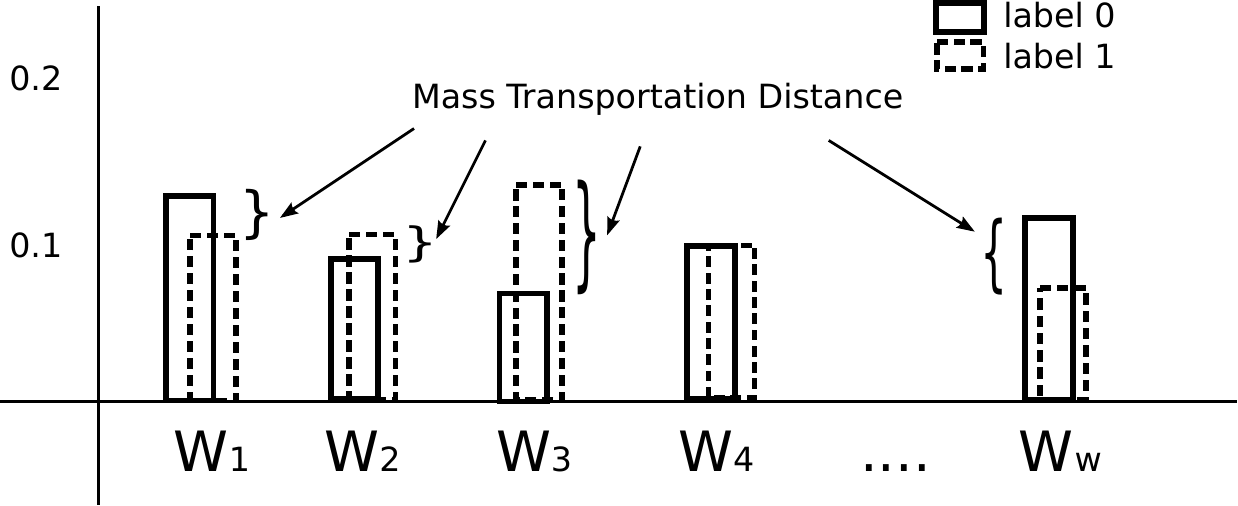}
\caption[Mass Transportation Distance for discrete space]{Visualization of the Mass Transportation Distance for a finite discrete space $\Omega = \{\omega_1,\omega_2,\ldots,\omega_w\}$. This distance is simply the sum of the absolute height differences between the bars of the two histograms.}
\label{mtdpic}
\end{center}
\end{figure}

Together, MTD Feature Selection works as follows. Given a training sample $\mathcal{S}_\mathrm{lab} = \{(X_1,Y_1),(X_2,Y_2),\ldots,(X_n,Y_n)\}$, for each coordinate (or feature) $1\leq j\leq m$, compute the importance score for $j$ as 
\[
S(j) = {\hat d({\hat \mu_{0}^j},{\hat \mu_{1}^j})}.
\]
Fix a threshold $\alpha$ and collect the coordinates with high importance scores: $J = \{j \,|\, S(j) \geq \alpha\}$. Define the dimension reduction map $T:\Omega^m \to  \Omega^{|J|}$ by
\[
(X_1,X_2,\ldots,X_m)\longmapsto (X_j\, | \,j \in J).
\]
Once $\mathcal{S}_\mathrm{lab}$ and any new observations are projected down to $\Omega^{|J|}$, a classifier can be applied for predictions in this reduced space.

The following is a concrete calculation of the Mass Transportation Distance for a small generated dataset.

\subsubsection{Concrete Example}

Suppose a randomly generated labeled dataset is given as in Table \ref{rangenmtd}, with 6 observations from the domain $\{A,B,C\}^5$, with possible labels $\{0,1\}$. 

\begin{table}
\begin{center}
\caption[Randomly generated dataset for MTD Feature Selection]{Randomly generated dataset for MTD Feature Selection, with 6 observations and 5 features.}
\label{rangenmtd}

\vspace{4mm}

\begin{tabular}{c|ccccc|c}
Observation & Feature 1 & Feature 2 & Feature 3 & Feature 4 & Feature 5 & Label\\
\hline
1 &C&B&B&B&C& 0 \\
2 &B&C&A&B&A& 1 \\
3 &C&C&C&A&A& 1 \\
4  &A&B&B&A&A&0 \\
5  &B&A&C&C&A&1 \\
6 &C&C&A&A&B& 1 \\
\end{tabular}
\end{center}
\end{table}
For Feature 1, the empirical probability measures are
\[
{\hat \mu_{0}^1} = (1/2) \delta_{A} + (0) \delta_{B} + (1/2) \delta_{C}
\]
and
\[
{\hat \mu_{1}^1} = (0) \delta_{A} + (2/4) \delta_{B} + (2/4) \delta_{C},
\]
so its Mass Transportation Distance is
\[
{\hat d({\hat \mu_{0}^1},{\hat \mu_{1}^1})}  = |1/2 - 0| + |0 - 2/4| + |1/2 - 2/4| = 1.
\]
The Mass Transportation Distances for Features 2 to 5 can be calculated in similar manners, and all 5 distances are given in Table \ref{mtdvalues}. For MTD Feature Selection, if the importance threshold $\alpha$ is 1.5, then only Features 2 and 3 would be used for classification.

\begin{table}
\begin{center}
\caption[Calculated Mass Transportation Distances]{Calculated Mass Transportation Distances for the randomly generated dataset.}
\label{mtdvalues}

\vspace{4mm}

\begin{tabular}{c|ccccc}
& ${\hat d({\hat \mu_{0}^1},{\hat \mu_{1}^1})}$ &${\hat d({\hat \mu_{0}^2},{\hat \mu_{1}^2})}$ & ${\hat d({\hat \mu_{0}^3},{\hat \mu_{1}^3})}$ & ${\hat d({\hat \mu_{0}^4},{\hat \mu_{1}^4})}$ & ${\hat d({\hat \mu_{0}^5},{\hat \mu_{1}^5})}$ \\
\hline
Distance & 1 & 2 & 2 & 1/2 & 1\\
\end{tabular}
\end{center}
\end{table}

Although the Mass Transportation Distance has been thoroughly studied in mathematics and computer science, it is used in data science as a feature selection method for the very first time in this thesis. Section \ref{mtdlip} provides a brief introduction to the relationship between this distance and a certain class of real-valued classifiers, in order to justify its use in supervised learning theory.

\subsection{MTD and Soft Margin Classification}\label{mtdlip}

This section provides some theory to justify using the Mass Transportation Distance in supervised learning. Recall that a binary classifier is simply any function from the domain $\Omega$ to $\{0,1\}$; consequently, any real-valued function $f:\Omega\to\mathbb{R}$ can be viewed as a classifier $C_f:\Omega\to\{0,1\}$ where
\[
C_f(x) = \begin{cases}
1 &\quad\textnormal{ if $f(x)\geq 0$}\\
0 &\quad\textnormal{ if $f(x) < 0$}.
\end{cases}
\]
Given any metric space $(\Omega,d)$, a function $f:\Omega\to\mathbb{R}$ is {\em 1-Lipschitz} if $|f(x) - f(y)| \leq d(x,y)$, for all $x,y\in \Omega$ \cite{normref}. Such a function can be equivalently viewed as a classifier, known as a {\em 1-Lipschitz classifier}, through the relationship $f\longmapsto C_f$ defined above. 

A certain class of classifiers is the set of bounded 1-Lipschitz classifiers, denoted by $\mathrm{Lip}_1(\Omega)$:
\[
\mathrm{Lip}_1(\Omega) = \{f:\Omega \to\mathbb{R}: \textnormal{$f$ is a bounded 1-Lipschitz classifier}\}.
\]
This class is important in learning theory since 1-Lipschitz functions have small classification variance; in other words, observations which are close to each other are generally assigned the same label. Luxburg and Bousquet in \cite{lipclass} studied this set of classifiers in detail, by isometrically embedding a metric space $(\Omega,d)$ in a Banach space $\mathcal{B}$ and the class of 1-Lipschitz functions into the dual space $\mathcal{B}'$.  Then, classifying observations in the original metric space by 1-Lipschitz functions is equivalent to optimally dividing the embedded observations in $\mathcal{B}$ by some hyperplane, a linear form in $\mathcal{B}'$. 

The advantage of using isometric embeddings is that finding suitable hyperplanes for classification in a Banach space is a well-studied problem \cite{lipclass}, whereas constructing 1-Lipschitz classifiers with good classification performance in a general metric space can be extremely difficult. Therefore, an important problem in supervised learning theory is whether a 1-Lipschitz classifier with good classification performance exists for a given metric space. The Mass Transportation Distance relates to 1-Lipschitz classifiers, thus offering a possible solution to this problem, because it can be used to determine the existence of a high-performing 1-Lipschitz classifier in terms of classification margin explained below.

Let $\epsilon >0$ and $\delta\geq 0$. Call a 1-Lipschitz function $f:\Omega\to\mathbb{R}$ a {\em soft $(\epsilon,\delta)$-margin classifier} if
\begin{equation}\label{complexeq}
\mu\{X \in \Omega : (-1)^{T(X)+1}f(X)<\epsilon/2\}\leq \delta,
\end{equation}
where $T(X) = Y$ denotes the true class label for an observation $X\in\Omega$. The expression $(-1)^{T(X)+1}f(X)<\epsilon/2$ in (\ref{complexeq}) is equivalent to 
\[
T(X) = 1 \textnormal{ and } f(X) < \epsilon/2, \textnormal{ or } T(X) = 0 \textnormal{ and } f(X) > -\epsilon/2,
\]
which refers to observations wrongly, or almost wrongly, classified by $f$. Denote $\mu_0$ and $\mu_1$ as the theoretical distributions of observations $X\in\Omega$ with labels $Y = 0$ and $Y = 1$, respectively, and assume that the two class distributions are equally balanced: $\mu_0 = \mu_1 = 1/2$. 
The relationship between the Mass Transportation Distance and 1-Lipschitz classifiers is first given by the following theorem, the famous result regarding this distance by Kantorovich and Rubinstein, e.g. found in \cite{mtdmain}. 

\begin{theo}[Kantorovich Optimality Criterion \cite{mtdmain}]\label{mainmtd}
Let $(\Omega,d)$ be a metric space, then the Mass Transportation Distance between two finitely-supported probability measures $\mu$ and $\mu'$ can be written as
\[
{\hat d}(\mu,\mu') = \sup_f \int_\Omega f(x) d(\mu-\mu')(x),
\]
where the supremum is taken over all 1-Lipschitz functions $f:\Omega\to\mathbb{R}$.
\end{theo}

An equivalent statement to Theorem \ref{mainmtd} is that the infimum in the definition of the Mass Transportation Distance
\[
{\hat d}(\mu,\mu') = \inf_{\nu}\int_{\Omega\times \Omega} d(x,y) \, d\nu
\]
is achieved by a probability measure $\nu$ on $\Omega\times\Omega$, with marginals $\mu$ and $\mu'$, if and only if there exists a 1-Lipschitz function $f:\Omega\to\mathbb{R}$ such that
\[
f(x) - f(y) = d(x,y)
\]
for all $(x,y)$ in the support of $\nu$, see e.g. \cite{supestov}. As a result, in the context of supervised learning theory, suppose a metric space $(\Omega,d)$ has {\em diameter}
\[
\mathrm{diam}(\Omega) = \sup_{x,x'\in \Omega} d(x,x') \leq 1.
\]
If there exists a soft $(\epsilon,\delta)$-margin classifier, then ${\hat d}(\mu_0,\mu_1) \geq \epsilon (1-\delta)$. The reason is that, by Theorem \ref{mainmtd}, $d(X,X') \geq f(X)-f(X') \geq \epsilon/2 + \epsilon /2 = \epsilon$ for any $X,X'\notin \{X \in \Omega : (-1)^{T(X)+1}f(X)<\epsilon/2\}$ with $X$ having label $1$ and $X'$ label 0. Therefore, as the Mass Transportation Distance is defined by an infimum, where the inside integral is bounded below by $\epsilon(1-\delta)$,
\[
{\hat d}(\mu,\mu') = \inf_{\nu}\int_{\Omega\times \Omega} d(x,y) \, d\nu\geq \epsilon(1-\delta).
\]

Conversely, given $\epsilon >0$, $\gamma\geq0$, and a metric space $(\Omega,d)$, define a new $\epsilon$-bounded distance $d_\epsilon = \min\{d,\epsilon\}$ by
\[
d_\epsilon(x,x') = \begin{cases}
d(x,x') & \quad\textnormal{ if $d(x,x') \leq \epsilon$}\\
\epsilon & \quad\textnormal{ otherwise.}
\end{cases}
\]
If the Mass Transportation Distance with respect to $d_\epsilon$ satisfies
\[
{\hat d_\epsilon}(\mu_0,\mu_1) \geq \gamma,
\]
then by Theorem \ref{mainmtd}, there exists a 1-Lipschitz function $f:(\Omega,d_\epsilon)\to (\mathbb{R},|\cdot - \cdot|)$, where $d^{|-|} = |\cdot-\cdot|$ denotes the regular absolute value distance on $\mathbb{R}$, such that the Mass Transportation Distance of the push-forward measures of $\mu_0$ and $\mu_1$, with respect to $d^{|-|}$, is at least $\gamma$:
\[
{\hat d^{|-|}}(f_*\mu_0, f_*\mu_1)\geq \gamma.
\]
Note that $\mathrm{diam}(f(\Omega)) \leq \epsilon$ since $f$ is a 1-Lipschitz function and that $f_*\mu_0$ and $f_*\mu_1$ are measures defined on $(\mathbb{R},|\cdot - \cdot|)$. As a result, the Mass Transportation Distance has an exact form in this space, see e.g. \cite{mtdline}:
\[
{\hat d^{|-|}}(f_*\mu_0, f_*\mu_1) = \int_0^\epsilon F_0(x) - F_1(x) \, dx \geq \gamma,
\]
where $F_0$ and $F_1$ respectively denote the cumulative distribution functions of $\mu_0$ and $\mu_1$. Consider Figure \ref{mtdline} for some $\delta\geq 0$, and the goal now is to bound $\gamma$ in order to show that this particular $f$ is a soft margin classifier. Based on this figure, which is drawn without any loss of generality regarding $F_0$ and $F_1$, $\gamma$ is a lower bound on the area between $F_0$ and $F_1$. Since this area can be bounded above by three rectangles, the following holds:
\begin{align*}
\gamma &\leq \epsilon/3 \textnormal{\quad\quad\quad\quad\,\,\, (middle rectangle, with height 1)}\\
& + (1-\delta)\epsilon/3 \textnormal{\quad\quad (left rectangle, with height $(1-\delta)$)}\\
& + (1-\delta)\epsilon/3 \textnormal{\quad\quad (right rectangle, with height $(1-\delta)$)}\\
& = \epsilon (1 - 2\delta/3)
\end{align*}
As a result, $2\delta/3 \leq 1-\gamma/\epsilon$ so
\[
\delta \leq 3(1-\gamma/\epsilon)/2.
\]
Consequently, $f$ is a soft $(\epsilon/3,3(1-\gamma/\epsilon)/2)$-margin classifier.

\begin{figure}
\begin{center}
\includegraphics[scale = 0.75]{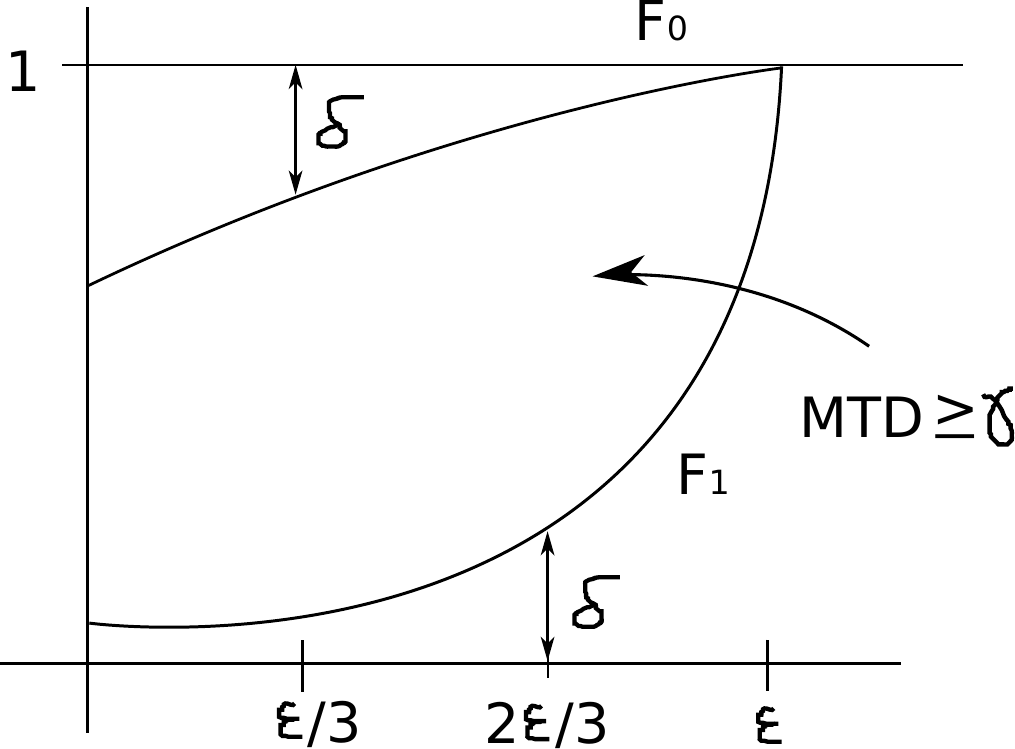}
\caption[Mass Transportation Distance on the real line]{Visualization of the Mass Transportation Distance on the real line, which is the area between the two cumulative distribution functions $F_0$ and $F_1$.}
\label{mtdline}
\end{center}
\end{figure}

In summary of the explanations above, the following theorem, as a corollary from Theorem \ref{mainmtd}, relates the Mass Transportation Distance to 1-Lipschitz functions and their classification margins.

\begin{theo}\label{newmtdthem}
Let $\epsilon>0$ and $\delta,\gamma \geq 0$. Then i) implies ii) and i') implies ii'):
\begin{itemize}
\item [i)] Suppose that $\mathrm{diam}(\Omega)\leq 1$ and that there exists a soft $(\epsilon,\delta)$-margin 1-Lipschitz classifier.
\vspace{-3.5mm}
\item [ii)] ${\hat d}(\mu_0,\mu_1) \geq \epsilon (1-\delta)$
\end{itemize}
\begin{itemize}
\item [i')] ${\hat d_\epsilon}(\mu_0,\mu_1) \geq \gamma$
\vspace{-3.5mm}
\item [ii')] There exists a soft $(\epsilon/3,3(1-\gamma/\epsilon)/2)$-margin classifier.
\end{itemize}
\end{theo}

Given a classification problem in some metric space, the Mass Transportation Distance should be calculated to determine whether 1-Lipschitz classifiers can be used for predictions. If an acceptable soft margin classifier exists by Theorem \ref{newmtdthem}, then the approach of Luxburg and Bousquet should be considered, in order to search for a suitable hyperplane which well separates the embedded observations, in a Banach space, of the two different labels. In real-life applications however, the Mass Transportation Distance can be extremely difficult to compute; therefore, for certain types of domains, the distance can be used as an effective feature selection method by assigning an importance score to each feature or coordinate individually. Only the highly scored features should be kept, and classifiers, including 1-Lipschitz functions, can then be used for predictions according to these features.

Chapter \ref{dimred} has introduced two techniques for dimensionality reduction, a known feature extraction method called Random Projections based on random matrix multiplication, and a new feature selection method called MTD Feature Selection based on the Mass Transportation Distance. The next chapter explains the process of validating a classifier and some measures of predictive performance.

\cleardoublepage

\chapter{Validation of a Classifier}\label{validclass}

This chapter explains the problem of validating a supervised learning classifier. Once a classifier is trained and predicts labels of new observations, possibly under some fixed classification parameters in training, it is important to know how accurate and precise the predictions are. In real-life applications however, actual labels for the new observations are seldom known and cannot be compared with the predictions. Consequently, standard evaluation methods for a learning algorithm are based on dividing the training sample into a training set and an evaluation set, since predictions for observations in the evaluation set can be compared with their actual known labels \cite{dmreview,crossvalid}. At the same time, these methods can also estimate the optimal training parameters, if applicable, that a classifier should use by iterating the evaluation process multiple times for various parameter values.

Section \ref{permeasure} details common measures of evaluating predictions, namely the accuracy, F-Measure, and area under the Receiver Operating Characteristic (ROC) curve. Section \ref{crossvalid} explains two techniques of dividing the training sample to allow for evaluating the performance of a classifier and determining its optimal classification parameters.

\section{Evaluation Measures}\label{permeasure}

The goal of this section is to detail a few measures of comparing and evaluating  predictions of observations from a classifier, against their actual labels. 

Suppose $\mathcal{Y}_\mathrm{pred} = \{{\tilde Y_1},{\tilde Y_2},\ldots,{\tilde Y_n}\}$ is a collection of predicted labels from a classifier, for $n$ observations, and $\mathcal{Y}_\mathrm{true}= \{Y_1,Y_2,\ldots,Y_n\}$ are their actual true labels, where ${\tilde Y_1},{\tilde Y_2},\ldots,{\tilde Y_n},Y_1,Y_2,\ldots,Y_n \in \{0,1\}$. The evaluation of predictions from a classifier is simply the comparison of the labels in $\mathcal{Y}_\mathrm{pred}$ against the respective labels in $\mathcal{Y}_\mathrm{true}$. A {\em confusion (error) matrix} is a $(2\times 2)$ matrix, see e.g. \cite{rpfmeasure}, which counts the occurrences of the 4 possible comparative outcomes:

\begin{minipage}{4cm}
\vspace{4mm}
\begin{center}
\vspace{13.5mm}
\begin{tabular}{lcc}
Predicted: ${\tilde Y_i}$ && Label 1\\
&& Label 0
\end{tabular}
\end{center}
\end{minipage}\begin{minipage}{10cm}
\vspace{4mm}
\begin{center}
\hspace{2mm}Actual: $Y_i$

\hspace{1mm}Label 1\quad\quad\quad\quad\quad\quad\quad Label 0

\vspace{1.5mm}

\begin{tabular}{|c|c|}
\hline
True Positives ($tp$) & False Positives ($fp$)\\
\hline
False Negatives ($fn$) & True Negatives ($tn$)\\
\hline
\end{tabular}
\end{center}
\end{minipage}
\vspace{3.4mm}

\noindent where
\begin{align*}
tp &= \mathrm{card}\{{\tilde Y_i} = 1\textnormal{ and } Y_i = 1\}\\
tn &= \mathrm{card}\{{\tilde Y_i} = 0\textnormal{ and } Y_i = 0\}\\
fp &= \mathrm{card}\{{\tilde Y_i} = 1\textnormal{ and } Y_i = 0\}\\
fn &= \mathrm{card}\{{\tilde Y_i} = 0\textnormal{ and } Y_i = 1\}.
 \end{align*}
As label 1 is often called the positive label and 0 is called the negative label, it is customary to refer to the four outcomes as {\em True Positive}, {\em True Negative}, {\em False Positive}, and {\em False Negative}, as indicated in the confusion matrix \cite{rpfmeasure}.

From the confusion matrix, mathematical scores which evaluate the predictions of a classifier can then be computed. The following subsections explain three common measures of evaluation.

\subsection{Accuracy}

{\em Accuracy} is the most basic and common score to compute from the confusion matrix, see e.g.  \cite{rpfmeasure}. It is simply the ratio of correct predictions and the number of total predictions:
\[
\mathrm{Accuracy} = \frac{\textnormal{Number of correct predictions}}{\textnormal{Number of predictions}} =  \frac{tp + tn}{tp + tn + fn + fp}
\]
The advantage of the accuracy measure is that it is easy to compute and gives a quick first indication of how well a classifier predicts. Although, if the two class sizes are known to not be approximately equal, accuracy values can be skewed by trivial predictions. For instance, suppose 20 observations have true label 1 and 5 observations have label 0. A classifier, which trivially predicts that all 25 observations have label 1, would have an accuracy of 80\%, even though it does not have any predictive powers.

\subsection{The F-Measure}

The {\em F-Measure} is a measure to evaluate predictions which is defined as the harmonic mean of {\em precision} and {\em recall}, e.g. see \cite{f1measure}. Based on the confusion matrix, they are defined as
\begin{align*}
\mathrm{Recall} & = \frac{tp}{tp + fn}\\
\mathrm{Precision} & = \frac{tp}{tp + fp}
\end{align*}
Recall is the proportion of observations with actual label 1 that are indeed labeled as 1, while precision is the proportion of observations with predicted label 1 that actually have label 1. In terms of recall and precision, the F-Measure is defined as
\[
\mathrm{F}\textnormal{-}\mathrm{Measure} = 2\left(\frac{\mathrm{Precision}\cdot\mathrm{Recall}}{\mathrm{Precision}+ \mathrm{Recall}}\right) = \frac{2\cdot tp}{(2\cdot tp) + fn + fp}.
\]

The F-Measure is commonly used in data science, especially in the area of information retrieval, where positive observations of label 1 are important \cite{f1measure}. A disadvantage is that the F-Measure does not take predictions of label 0 into account.

As a quick example, the following confusion matrix
\begin{center}
\begin{tabular}{|c|cc|}
\hline
&Actual label 1 & Actual label 0\\
\hline
Predicted label 1 & 25 & 12\\
Predicted label 0 & 5 & 49\\
\hline
\end{tabular}
\end{center}
would result in an accuracy score of
\[
\frac{25 + 49}{25 + 49 + 12 + 5} = \frac{74}{91}\approx 0.8132 
\]
and a F-Measure score of approximately
\[
\frac{2 \times 0.8333 \times 0.6757}{0.8333 + 0.6757} \approx 0.7463,
\]
since the recall is $25/30 \approx 0.8333$ and the precision is $25/37 = 0.6757$.

\subsection{Area under the Receiver Operating Characteristic Curve}\label{rocarea}

The {\em area under the Receiver Operating Characteristic (ROC) curve} is another measure of evaluating predictions from a classifier, which is based on decision parameter variation. Since most classifiers depend on parameters, either implicitly or explicitly defined, at the classification step, changes to these parameters would in turn affect predictions. Area under the ROC curve measures how well classifiers predict as the parameters are varied. Two references, which Section \ref{rocarea} is based on, regarding the ROC curve and its area are \cite{rocbrad} and \cite{rocorig}.

Formally, suppose a classifier $g = g_\alpha$ depends on a parameter $\alpha\in [a,b]\subseteq \mathbb{R}\cup\{\infty\}$ at the classification stage, which is increasingly varied from $a$ to $b$ in a finite number $w$ of increments:
\[
a = \alpha_1 < \alpha_2 < \ldots <  \alpha_{w-1} < \alpha_w = b.
\]
For each fixed value $\alpha\in\{\alpha_1,\alpha_2,\ldots,\alpha_w\}$, $g_\alpha$ will output the collection 
\[
\mathcal{Y}_{\mathrm{pred},\alpha} = \{{\tilde Y_{1,\alpha}},{\tilde Y_{2,\alpha}},\ldots,{\tilde Y_{n,\alpha}}\}
\]
of predictions for some $n$ observations. Each collection $\mathcal{Y}_{\mathrm{pred},\alpha}$ can be compared to the true labels in $\mathcal{Y}_\mathrm{true}$ of the observations, and a confusion matrix depending on $\alpha$ can be computed:
\[
\left(\begin{array}{cc}
\textnormal{True Positives } (tp_\alpha) & \textnormal{False Positives } (fp_\alpha)\\
\textnormal{False Negatives } (fn_\alpha) & \textnormal{True Negatives } (tn_\alpha)
\end{array}\right).
\] 

\begin{figure}
\begin{center}
\includegraphics[scale = 0.4]{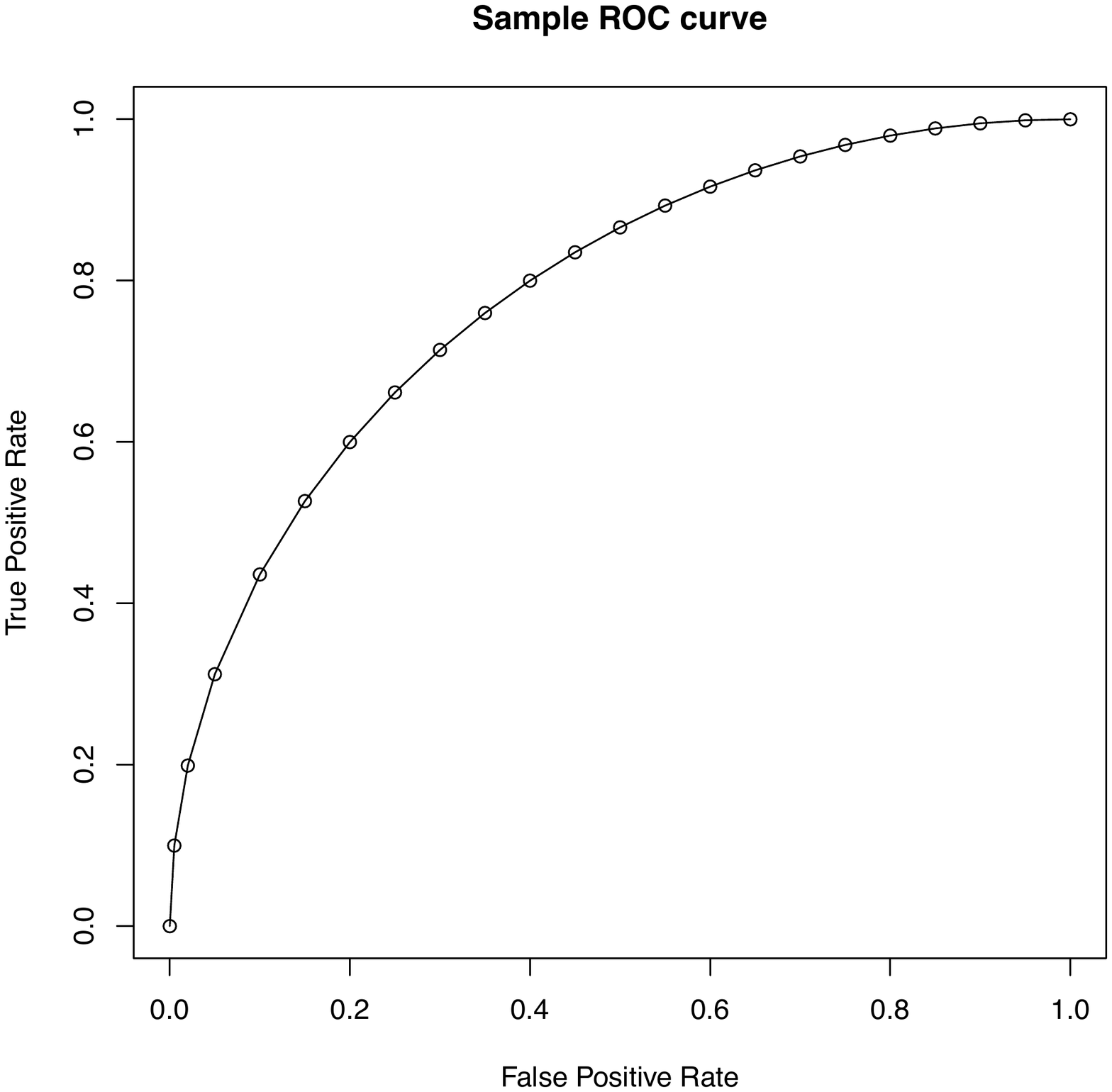}
\caption[Example of a ROC curve]{This is a sample ROC curve, where $\alpha\in[a,b]$ is increasingly varied from $a$ to $b$ in $w = 23$ increments. The pairs of vectors consisting of false and true positive rates are plotted in $\mathbb{R}^2$, and line segments connect these points to form the ROC curve.}
\label{samproc}
\end{center}
\end{figure}

From the confusion matrix, define the {\em true positive} and {\em false positive rates}, denoted by $tpr_\alpha$ and $fpr_\alpha$, as
\begin{align*}
tpr_\alpha  = \mathrm{Recall}_\alpha& = \frac{tp_\alpha}{tp_\alpha + fn_\alpha}\\
fpr_\alpha & = \frac{fp_\alpha}{fp_\alpha + tn_\alpha}.
\end{align*}
Then, the area under ROC curve is defined to be
\[
\mathrm{Area}_\mathrm{ROC} = \left|\sum_{i = 1}^{w-1} (fpr_{\alpha_i} - fpr_{\alpha_{i+1}})tpr_{\alpha_{i+1}}\right|.
\]
The interpretation for $\mathrm{Area}_\mathrm{ROC}$ is the following. Pairs of values $(fpr_\alpha,tpr_\alpha)$ can be plotted in $\mathbb{R}^2$, for $\alpha \in \{\alpha_1,\alpha_2,\ldots,\alpha_w\}$; the Receiver Operating Characteristic (ROC) curve is simply the curve consisting of line segments that join these points; and the area under the ROC curve is the integral, estimated using the rectangle method, of this curve. Figure \ref{samproc} provides an example of a ROC curve, where the area under the curve is approximately $\mathrm{Area}_\mathrm{ROC} = 0.785$.

\subsubsection{Example of Parameter Variation for the ROC Curve}

The following provides an example of using decision parameter variation to generate the ROC curve for the Random Forest classifier. Recall that the Random Forest classifier $g$ predicts the label for a new observation $X$ based on the votes $T_i(X)$ of $t$ Decision Trees, generated by bootstrap samples from the training observations:
\[
\textnormal{label of $X$} = g(X) = \mathrm{mode}\{T_i(X):i = 1,2,\ldots,t\}.
\]
An equivalent definition is
\begin{equation}
\label{rfvotealpha}
g(X) = \begin{cases}
1 &\quad \textnormal{if $\frac{\mathrm{card}\{T_i(X) = 1\}}{t} \geq \frac{1}{2}$};\\
0 &\quad \textnormal{else}.
\end{cases}
\end{equation}
From this definition, it is clear that the decision parameter $\alpha$ for the ROC curve can be the voting threshold, which in (\ref{rfvotealpha}) is $\alpha = 0.5$. In  other words, for $\alpha \in [0,1]$, the parametrized Random Forest classifier $g_\alpha$ is defined by $g_\alpha(X) = 1$ if $\mathrm{card}\{T_i(X) = 1\}/t \geq \alpha$ and $g_\alpha(X) = 0$ otherwise. A similar definition for the parameterized $k$-NN classifier, also based on the voting threshold, exists as well.

Note that if $\alpha = 0$, all predictions from $g_\alpha$ is 1 and if $\alpha = 1$, almost all predictions would be 0. Hence, the true positive and false positive rates would be $(tpr_0,fpr_0) = (1,1)$ and $(tpr_1,fpr_1) \approx (0,0)$ for predictions by Random Forest. The ROC curve, as $\alpha$ ranges from $0$ to $1$, would then be generated in a counterclockwise manner. 

Since the area under the ROC curve evaluates predictions according to parameter variation, one advantage is that the area provides a measure of classification performance for a classifier dependent on its decision parameter. The ROC curve itself also allows for a visual representation of how the true and false positive rates change as the parameter varies. The disadvantage is of course that not all classifiers have parametrical dependencies, or that the parameters are difficult or impossible to determine and alter.

Section \ref{crossvalid} below explains two methods of dividing a training sample to allow for a classifier to train and make predictions; these predictions can then be evaluated based on measures defined above.

\section{Validation on Training Set}\label{crossvalid}

This section discusses two methods of dividing a training set to evaluate a classifier and estimate optimal classification parameters. As usual, the training sample will be denoted as
\[
\mathcal{S}_\mathrm{lab}= \{(X_1,Y_1),(X_2,Y_2),\ldots,(X_n,Y_n)\}.
\]

\subsection{Holdout Method}

The first method, known as the {\em Holdout Method}, of dividing a training sample is to simply randomize the sample and divide it into an evaluation-testing set and an evaluation-training set \cite{crossvalid}:
\[
\mathcal{S}_\mathrm{lab} = \mathcal{S}_\mathrm{test}\cup \mathcal{S}_\mathrm{train},
\]
where
\begin{align*}
\mathcal{S}_\mathrm{test}& = \{(X_1,Y_1),(X_2,Y_2),\ldots,(X_{n'},Y_{n'})\}\\
\mathcal{S}_\mathrm{train}&=\{(X_{n'+1},Y_{n'+1}),(X_{n'+2},Y_{n'+2}),\ldots,(X_{n},Y_{n})\}
\end{align*}
for some $1 < n' < n$. A classifier would train on $\mathcal{S}_\mathrm{train}$  and predict labels for observations $X_{1},X_{2},\ldots,X_{n'}$ in $\mathcal{S}_\mathrm{test}$. Since the true labels for these observations are known, the predictions can be evaluated according to measures in Section \ref{permeasure}. In addition, a learning algorithm often depends on one, or more, classification parameter in training, such as the value $k$ for $k$-NN or the number $t$ of generated Decision Trees for Random Forest. In such a case, once a training sample has been divided, the classifier under different parameter values would be run and evaluated, and the value that results in the best classification performance, according to the same training and testing split, would be selected as optimal.
 
In practice, the number $n'$ of testing observations is usually taken to be approximately 1/3 of the size $n$ of the entire training set. The Holdout Method is often repeated many times and the average of the evaluation measures, across the repetitions, is considered as an estimate for classification performance \cite{crossvalid}.

\subsection{Cross Validation}

The method of {\em $t$-fold cross validation} divides a training sample $\mathcal{S}_\mathrm{lab}$ into $t$ groups and trains and predicts on the groups in a sequential manner, explained as follows \cite{crossvalid}:
\begin{enumerate}
\item Randomize $\mathcal{S}_\mathrm{lab}$ and partition it into $t$ disjoint groups, each of roughly equal size:
\[
\mathcal{S}_\mathrm{lab} = \bigcup_{i = 1}^t \mathcal{S}_i,
\]
where $|\mathcal{S}_i| \approx n/t$.
\item For $i = 1,2,\ldots,t$, train a classifier with all the groups, except for the $i$-th one: $\mathcal{S}_\mathrm{lab}\setminus \mathcal{S}_i$.
\item Allow the trained classifier to predict labels for observations in $\mathcal{S}_i$.
\item Compare the prediction labels against the actual labels and assign a score $s_i$ of performance, e.g. based on the evaluation measures defined in Section \ref{permeasure}.
\end{enumerate}

\begin{figure}
\begin{center}
\includegraphics[scale = 0.5]{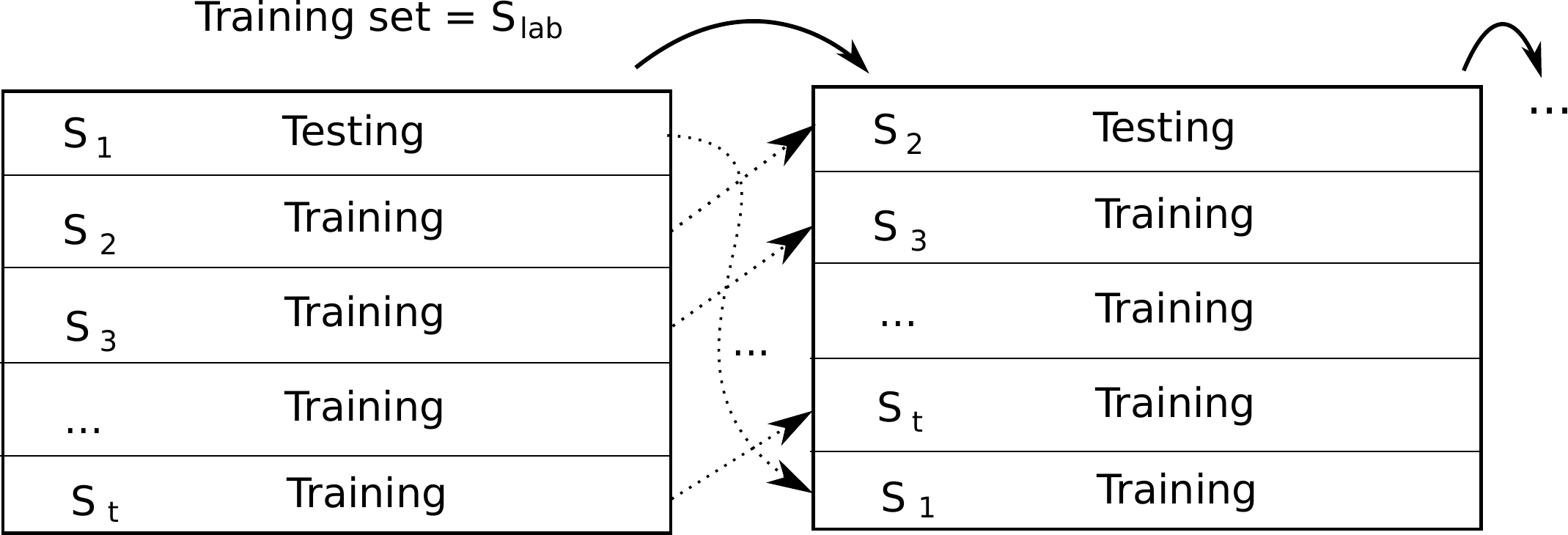}
\caption[Visualization of the cross validation process]{This diagram gives a visualization of the cross validation process. A training set is divided into $t$ groups, or stacks, $\mathcal{S}_1,\mathcal{S}_2,\ldots,\mathcal{S}_t$; after a classifier is tested on $\mathcal{S}_1$, it is put below the other stacks, and $\mathcal{S}_2$ becomes the new testing set. The process iterates until $\mathcal{S}_1$ is on top of the stacks again.}
\label{crossval}
\end{center}
\end{figure}

\noindent The mean ${\bar s}$ of the scores $s_1,s_2,\ldots,s_t$,
\[
{\bar s} = \sum_{i = 1}^t s_i,
\]
is taken to be an estimate of the classifier's performance. 

In particular, the method of $n$-fold cross validation, where $n$ is the total number of observations in $\mathcal{S}_\mathrm{lab}$, is called {\em leave-one-out cross validation}, because in each repetition, all but one instance in the training set are used to train a classifier \cite{crossvalid}. Figure \ref{crossval} is an illustration of the cross validation process. If a learning algorithm depends on some classification parameter in training, once the sample has been partitioned into $t$ disjoint groups, the cross validation process would run multiple times for different parameter values for determining the optimal one for classification.

Both the Holdout Method and $t$-fold cross validation are commonly used in practice for dividing a training set to evaluate the performance of a classifier and to estimate optimal parameters used for classification. The Holdout Method is the simplest but suffers from possibly high variance, as the training set is divided at random and often more than once. Furthermore, not all observations are guaranteed to be in the evaluation-training set. Conversely, $t$-fold cross validation ensures that each observation is trained $t-1$ times and tested once. The disadvantage is that the evaluation process has to be repeated $t$ times and can be time-consuming. For more information on these methods, see \cite{crossvalid} and \cite{dmreview}.

In summary, Chapter \ref{validclass} has introduced three measures of evaluating the predictions of a classifier, namely accuracy, F-Measure, and the area under the ROC curve, and explained two methods, the Holdout Method and $t$-fold cross validation, of dividing a training sample to evaluate a classifier according to these measures. The next chapter explains the methodology of applying these evaluation techniques for the $k$-NN and Random Forest classifiers on the genetic dataset for predicting coronary artery disease.

\cleardoublepage

\chapter{Dataset and Methodology}\label{datasetchap}

Chapter \ref{datasetchap} describes the genetic dataset of Single-Nucleotide Polymorphism (SNP) information considered for this thesis, and the methodology of applying the dimensionality reduction steps and classification algorithms on this dataset. Section \ref{inforaw} explains the dataset that comes from the Ontario Heart Genomics Study. Section \ref{methosect} details the methodology of applying two approaches of dimensionality reduction and classification algorithms, which would be evaluated through the Holdout Method and cross validation, to predict coronary artery disease on this dataset.

\section{Information on Dataset}\label{inforaw}

Directly from the Ontario Heart Genomics Study (OHGS), the dataset, known as the {\em OHGS B2 file}, considered for this thesis contains genotype information for 865688 SNPs across 3907 patients (the observations), each labeled as {\em control} (Non CAD) with label 0 or as {\em case} (CAD) with label 1. These SNPs are based on the patients' DNA information from Chromosomes 1 to 22, except for Chromosome 16 because the file storing SNP genotype information for this chromosome is corrupt. Also, for simplicity, SNPs from the two sex chromosomes X and Y are not considered either.

Table \ref{datasetinfo} provides counts for the number of control and case patients in the OHGS dataset, and Table \ref{snpcount} lists the number of SNPs from each of Chromosomes 1 to 15 and 17 to 22. As mentioned in Section \ref{gwas}, due to DNA microarray genotyping limitations, the genotype for a patient at each SNP cannot be determined exactly; only probabilities for the three possible genotypes at the SNP are provided from the Ontario Heart Genomics Study. Consequently, the dataset in question is stored as
\[
M_\mathrm{raw} = \left(\begin{array}{ccccccccc}
p_{1,1}&q_{1,1}&r_{1,1} && \ldots && p_{1,865688}&q_{1,865688}&r_{1,865688}\\
p_{2,1}&q_{2,1}&r_{2,1} && \ldots && p_{2,865688}&q_{2,865688}&r_{2,865688}\\
\vdots & \vdots& \vdots && \ddots &&\vdots & \vdots& \vdots \\
p_{3907,1}&q_{3907,1}&r_{3907,1} && \ldots && p_{3907,865688}&q_{3907,865688}&r_{3907,865688}\\
\end{array}\right),
\]
where $p_{ij},q_{ij},r_{ij}\in[0,1]$ are respectively the corresponding probabilities of observation $i$ having genotype {\em homozygous major}, {\em heterozygous}, and {\em homozygous minor} for SNP $j$. Because there are only three possible genotypes, $p_{ij} + q_{ij} + r_{ij} = 1$.

\begin{table}
\begin{center}
\caption[OHGS dataset class information]{Information on the number of observations with labels 0 (control) and 1 (case) of the OHGS dataset considered for this thesis.}

\vspace{4mm}

\begin{tabular}{|c|c|c|}
\hline
Label & \# of observations & Dimension of dataset\\
\hline\hline
Control (Non CAD) & 1978 &-\\
\hline
Case (CAD) & 1929&-\\
\hline \hline
Total & 3907 & 865688\\
\hline
\end{tabular}
\label{datasetinfo}
\end{center}
\end{table}

\begin{table}
\begin{center}
\caption[OHGS dataset's SNP count from each chromosome]{Information on the number of SNPs, from each of Chromosomes 1 to 15 and 17 to 22, in the OHGS dataset.}

\vspace{4mm}

\begin{tabular}{|c||ccccccc|}
\hline
Chromosome & 1 & 2 & 3 & 4 & 5 & 6 & 7 \\
\hline
\# of SNPs & 73571 & 75918 & 62268 &57582 &57971&57687&48380\\
\hline \hline
Chromosome & 8 & 9 & 10 & 11 & 12 & 13 & 14\\
\hline
\# of SNPs  &50026&42785 &49600 & 45927 &43802 &34979 &28936\\
\hline \hline
Chromosome &  15 & 17 & 18 & 19 & 20 &21 & 22\\
\hline
\# of SNPs & 26907 & 21319 & 27212 & 12422 & 23488 &12924 & 11984\\
\hline
\end{tabular}
\label{snpcount}
\end{center}
\end{table}

As the entries of $M_\mathrm{raw}$ are all probabilities, belonging to $\mathbb{R}$, the observations reside in a finite dimensional normed vector space $\Omega = (\mathbb{R}^{865688\times 3},||\cdot||)$. As a result, Random Projections and the $k$-Nearest Neighbour classifier can be applied to this dataset directly for the prediction of coronary artery disease. The approach with MTD Feature Selection with Random Forest, on the other hand, would require that $M_\mathrm{raw}$ be transformed to a discrete dataset. Such a transformation can be defined as follows. Let $\Omega = \{\mathrm{HM}, \mathrm{He}, \mathrm{Hm}\}$ with three possible words, corresponding to the three genotypes for each SNP, homozygous major $(\mathrm{HM})$, heterozygous $(\mathrm{He})$, and homozygous minor $(\mathrm{Hm})$, be equipped with the discrete distance, and consider the map
\[
M_\mathrm{raw} \longmapsto M_\mathrm{cat} = \left(\begin{array}{cccc}X_{1,1} &X_{1,2} &\ldots & X_{1,865688}\\
X_{2,1} & X_{2,2}& \ldots & X_{2,865688}\\
\vdots & \vdots &\ddots & \vdots\\
X_{3907,1} & X_{3907,2} & \ldots & X_{3907,865688}
\end{array} \right),
\]
where
\[
X_{ij} = \begin{cases}
\mathrm{HM} &\quad \textnormal{if \quad$p_{ij}  = \max\{p_{ij},q_{ij},r_{ij}\}$}\\
\mathrm{He} &\quad \textnormal{if \quad$q_{ij}  = \max\{p_{ij},q_{ij},r_{ij}\}$}\\
\mathrm{Hm} &\quad \textnormal{if \quad$r_{ij}  = \max\{p_{ij},q_{ij},r_{ij}\}.$}\\
\end{cases}
\]
Then, the mentioned techniques, MTD Feature Selection and Random Forest, on the domain $\Omega^{865688}$ can be applied to $M_\mathrm{cat}$ for classification. Section \ref{methosect} discusses the methodology of applying the first approach, with Random Projections and $k$-NN, on $M_\mathrm{raw}$ and the second approach, with MTD Feature Selection and Random Forest, on $M_\mathrm{cat}$.

\section{Methodology}\label{methosect}

The goal of Section \ref{methosect} is to explain the methodology behind two approaches for the classification of coronary artery disease with the genetic dataset from the Ontario Heart Genomics Study, evaluated with the Holdout Method and cross validation:
\begin{description}
\item[Approach 1] (Random Projections and $k$-NN): Project the genetic dataset to a lower dimensional space with Random Projections and apply the $k$-NN classifier (Section \ref{approach1}).
\item[Approach 2] (MTD Feature Selection and Random Forest): Apply MTD Feature Selection to select important SNPs from the genetic dataset and use the Random Forest classifier (Section \ref{approach2}).
\end{description}

Due to the high-dimensionality of the genetic dataset, high-performance computing resources, namely the Enterprise M9000 servers from High Performance Computing Virtual Laboratory (HPCVL) \cite{hpcvl}, located in Eastern Ontario, Canada, are used to run the two dimensionality reduction methods, Random Projections and MTD Feature Selection, in parallel to save computational costs. The $k$-NN and Random Forest classifiers are then able to run, on the reduced datasets, through a personal computer with the R packages {\tt knnflex} \cite{knnflex} and {\tt randomForest} \cite{randomForest}.

\subsection{Approach 1: Random Projections and $k$-NN}\label{approach1}

For the initial experiment for the prediction of artery disease with the genetic dataset from the Ontario Heart Genomics Study, Approach 1 applies the Random Projections technique along with the $k$-Nearest Neighbour classifier using the real-valued dataset $M_\mathrm{raw}$. As each of the 865688 SNPs in the OHGS dataset requires 3 coordinates in $M_\mathrm{raw}$, for a total of $865688\times 3 = 2597064$ columns, this dataset is extremely high-dimensional. Therefore, for this initial experiment of Approach 1, only the 73571 SNPs from Chromosome 1 are used, so the considered dataset contains 3907 observations in $\Omega = (\mathbb{R}^{220713},||\cdot||)$, each with $73571\times 3 = 220713$ coordinates:
\[
M_\mathrm{raw,Chr 1} = \left(\begin{array}{ccccccccc}
p_{1,1}&q_{1,1}&r_{1,1} && \ldots && p_{1,73571}&q_{1,73571}&r_{1,73571}\\
p_{2,1}&q_{2,1}&r_{2,1} && \ldots && p_{2,73571}&q_{2,73571}&r_{2,73571}\\
\vdots & \vdots& \vdots && \ddots &&\vdots & \vdots& \vdots \\
p_{3907,1}&q_{3907,1}&r_{3907,1} && \ldots && p_{3907,73571}&q_{3907,73571}&r_{3907,73571}\\
\end{array}\right).
\]

From the method of Random Projections, the observations from $M_\mathrm{raw,Chr 1}$ are projected to the reduced space $(\mathbb{R}^{5000}, ||\cdot||)$ with dimension $m' = 5000$, via random matrix multiplication as explained in Section \ref{rpsect}. Due to the high-dimensionality of the dataset, only one value of $m'$ and one randomly generated matrix are considered. The $k$-NN classifier is run on the projected observations using both the $||\cdot||_1$ and $||\cdot||_2$ norms for the classification parameter $k = 1,3,5,7,9,11,13,15,17,19$. For validation of the $k$-NN classifier's performance, the Holdout Method, using 2880 observations for training and 1027 observations for testing (see Table \ref{holdoutinfo} for the class sizes), and  5-fold cross validation on the entire 3907 observations (from Table \ref{datasetinfo}) are run on the reduced dataset. The performance measures of accuracy, the F-Measure and area under the ROC curve are used; Section \ref{resultsknn} includes all the results from this validation process.

The following details a simple parallel framework for Random Projections, which is used for the genetic dataset $M_\mathrm{raw}$ due to its high-dimensional observations. Then, Section \ref{approach2} explains the methodology of classifying coronary artery disease based on Approach 2 with MTD Feature Selection and Random Forest.

\begin{table}
\begin{center}
\caption{OHGS dataset training and testing split for Holdout Method.}

\vspace{4mm}

\begin{tabular}{|c||c|c||c|}
\hline
& \# of controls  & \# of cases & Total\\
\hline\hline
Training & 1549 & 1331&2880\\
\hline
Testing &429&598&1027\\
\hline\hline
Total & 1978 & 1929 & 3907\\
\hline
\end{tabular}
\label{holdoutinfo}
\end{center}
\end{table}

\subsubsection{Parallel Random Projections}

This section describes a general parallel framework for applying Random Projections on a dataset with observations in a high-dimensional domain $\Omega = (\mathbb{R}^m,||\cdot||)$. Suppose such a dataset has $n$ observations, each with $m$ coordinates, represented as a matrix $M$ of size $n\times m$. For a lower dimension $m'<<m$, applying Random Projections to project observations to $(\mathbb{R}^{m'},||\cdot||)$ is of course equivalent to generating a random matrix $M_\mathrm{rand}$ (as in Theorem \ref{probjltheo}), of size $m\times m'$, and performing the matrix multiplication
\[
M_\mathrm{red} = \textnormal{the reduced dataset from $M$} = \frac{1}{\sqrt{m'}}M\cdot M_\mathrm{rand}.
\]

Both the generation of the random matrix $M_\mathrm{rand}$ and the multiplication $M\cdot M_\mathrm{rand}$ can be done in parallel. Fix $z$ as the number of jobs to run in parallel and partition $M$ into $z$ parts by columns, each of roughly equal size:
\[
M = \left[ \begin{array}{cccc}M_1&M_2&\ldots & M_z\end{array}\right],
\]
where each $M_i$ has the same number $n$ of rows as $M$ and has approximately $m/z$ columns, denoted by $\mathrm{col}(M_i)$. Generate $z$ random matrices
\[
M_\mathrm{rand,1},M_\mathrm{rand,2},\ldots,M_\mathrm{rand,z},
\]
each of size $\mathrm{col}(M_i)\times\ m'$. Then, the reduced matrix $M_\mathrm{red}$ is simply the sum 
\[
M_\mathrm{red} = \frac{1}{\sqrt{m'}}M\cdot M_\mathrm{rand}=\frac{1}{\sqrt{m'}} \sum_{i = 1}^{z} M_{i}\cdot M_{\mathrm{rand},i},
\]
where each matrix product $M_i\cdot M_{\mathrm{rand},i}$ can be calculated in parallel. Figure \ref{parallelrp} provides an illustration of this framework for Random Projections. 

\begin{figure}
\begin{center}
\includegraphics[scale = 0.5]{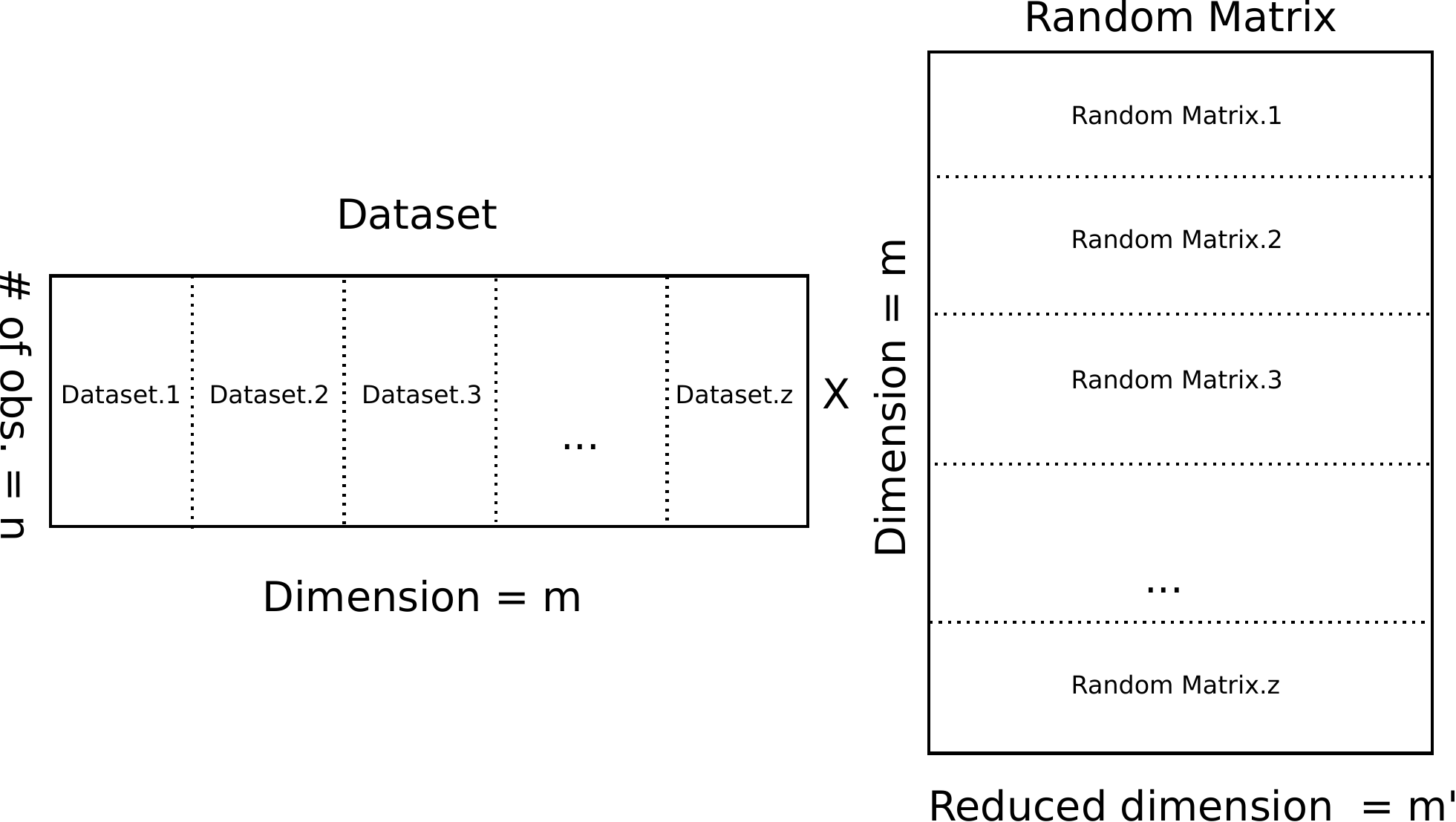}
\vspace{1mm}

{\large =}

 \includegraphics[scale = 0.75]{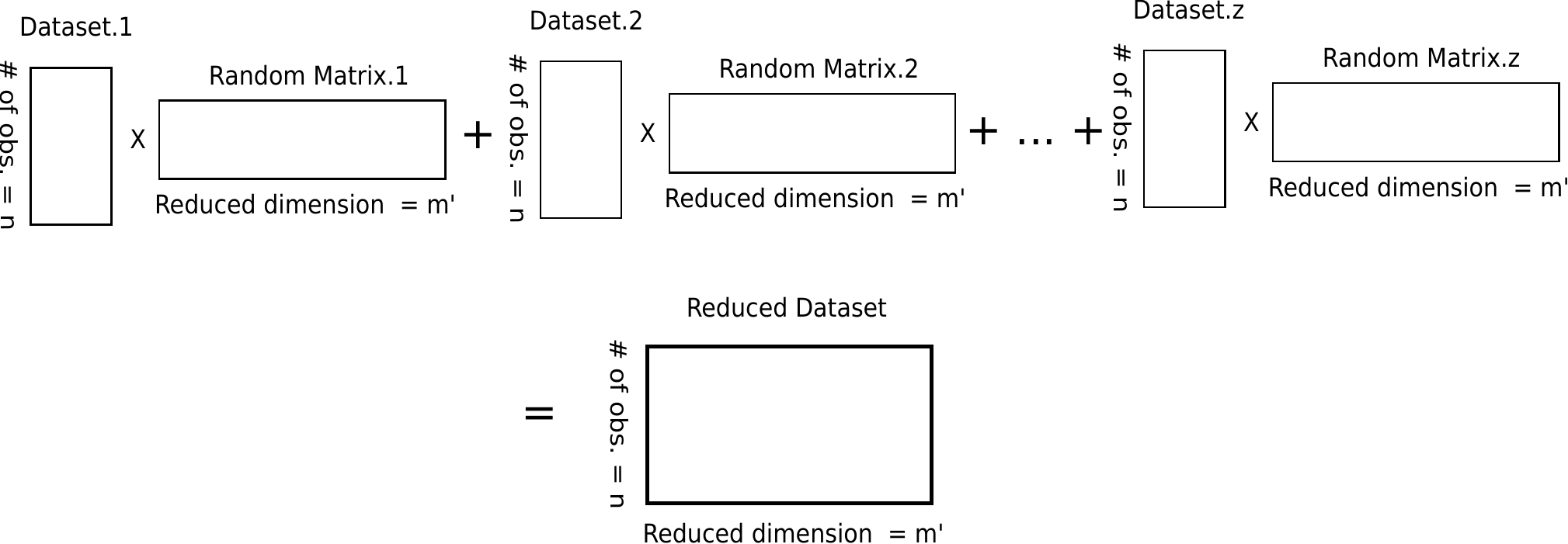}

\caption[Random Projections in parallel diagram]{Illustration of running Random Projections via multiplication by a randomly generated matrix in parallel. The dataset is partitioned into $z$ parts, each part is multiplied by a randomly generated matrix of the appropriate size, and then the products are summed together.}
\label{parallelrp}
\end{center}
\end{figure}

\begin{table}
\begin{center}
\caption[Number of SNPs selected with MTD Feature Selection from Chromosome 1]{Number of important SNPs selected with MTD Feature Selection from Chromosome 1, for the three various values of $\alpha$.}
\label{fsmtdchr1}

\vspace{4mm}

\centerline{\begin{tabular}{|c|ccc|}
\hline
 & $\alpha = 0.2$  & $\alpha = 0.3$ & $\alpha = 0.4$ \\
 \hline
\# of total $\to$ important SNPs & $73571 \to 287$ & $73571 \to 62$ & $73571 \to 12$  \\
\hline
\end{tabular}}
\end{center}
\end{table}

Regarding Approach 1, the dataset $M_\mathrm{raw}$ consists of $n = 3907$ observations, each with $220713$ coordinates. The reduced dimension is $m' = 5000$ and the parallel Random Projections framework is realized with the M9000K servers from HPCVL using $z = 21$ parallel jobs. The parallel computations take approximately 1 month to run.

\subsection{Approach 2: Mass Transportation Distance and Random Forest}\label{approach2}

Two experiments for Approach 2 are run for the prediction of coronary artery disease using the genetic dataset. For the first experiment, as a comparative study against Approach 1, the discrete dataset $M_\mathrm{cat,Chr1}$ (corresponding directly to $M_\mathrm{raw,Chr1}$) containing all 3907 observations with the 73571 SNP information from Chromosome 1 is used:
\[
M_\mathrm{cat,Chr1} = \left(\begin{array}{cccc}X_{1,1} &X_{1,2} &\ldots & X_{73571}\\
X_{2,1} & X_{2,2}& \ldots & X_{2,73571}\\
\vdots & \vdots &\ddots & \vdots\\
X_{3907,1} & X_{3907,2} & \ldots & X_{3907,73571}
\end{array} \right),
\]
where $X_{ij} \in \{\mathrm{HM},\mathrm{He},\mathrm{Hm}\}$. This dataset is similarly  divided into 2880 training observations and 1027 testing observations, as seen in Table \ref{holdoutinfo}, according to the Holdout Method. The MTD Feature Selection technique is run on the 2880 training observations and important SNPs are selected for each threshold value $\alpha = 0.2,0.3,0.4$. The computations take approximately 2 days on the M9000 supercomputers at HPCVL. Table \ref{fsmtdchr1} provides the number of selected SNPs for each value.

Random Forest, with default settings and $t = 500$ generated Decision Trees, is then trained on these 2880 observations using only the important SNPs and evaluated on the 1027 testing observations with these same SNPs. Note that since MTD Feature Selection uses class label information, it is run only on the training observations from the Holdout Method. Due to the high-dimensionality of the dataset and its associated computational costs, 5-fold cross validation is not considered in this experiment for Approach 2.

\begin{table}
\begin{center}
\caption[Number of SNPs selected with MTD Feature Selection from Chromosomes 1-15, 17-22]{Number of important SNPs selected with MTD Feature Selection from Chromosomes 1 to 15 and 17 to 22, for various values of $\alpha$.}
\label{fsmtdchrall}

\vspace{4mm}

\centerline{\begin{tabular}{|c|cccc|}
\hline
 & $\alpha = 0.2$  & $\alpha = 0.3$ & $\alpha = 0.4$ & $\alpha = 0.5$\\
 \hline
\# of total $\to$ important SNPs & $865688 \to 3335$ & $865688 \to 716$ & $865688 \to 99$ & $865688 \to 10$ \\
\hline
\end{tabular}}
\end{center}
\end{table}

For the second experiment of Approach 2, the entire discrete dataset $M_\mathrm{cat}$ of 3907 observations, with information on all 865688 SNPs from Chromosomes 1 to 15 and 17 to 22, is considered. The dataset is again divided in a training set and a testing set according to Table \ref{holdoutinfo}, and the MTD Feature Selection method is run on the training observations at importance thresholds $\alpha = 0.2,0.3,0.4,0.5$. It takes approximately 3 weeks to run the computations on the M9000 servers from HPCVL. Table \ref{fsmtdchrall} records the number of important SNPs selected for each threshold. Random Forest, with default settings and $t = 500$ generated Decision Trees, is trained on the 2880 observations and evaluated on the other 1027 observations, according to the selected SNPs. 

For both experiments of Approach 2, the predictive performance scores of accuracy, the F-Measure, and area under the ROC curve are considered. The results for this approach are found in Section \ref{mtdrfsec}.

\cleardoublepage

\chapter{Results and Discussion}\label{resultschap}

This chapter details the results of comparing the two approaches, as explained previously in Chapter \ref{datasetchap}, of predicting coronary artery disease based on the genetic dataset from the Ontario Heart Genomics Study. Section \ref{resultsknn} provides the performance scores for Approach 1 using Random Projections followed by the $k$-NN classifier, and Section \ref{mtdrfsec} provides the results from the new MTD Feature Selection and Random Forest. Section \ref{discsect} concludes the chapter with a discussion on the results.

\section{Random Projections and $k$-NN Results}\label{resultsknn}

Regarding the approach of applying Random Projections and the $k$-NN classifier on the dataset containing genotype information on 73571 SNPs from Chromosome 1 and 3907 observations, Tables \ref{foldfiveknn} and \ref{holdoutknn} respectively list the classification results from 5-fold cross validation and the Holdout Method on the reduced dataset of dimension $m' = 5000$. Based on 5-fold cross validation, the value $k = 9$ and the $\ell_1$ norm from the $k$-NN classifier resulted in the best accuracy of 0.5554 and the highest F-Measure of 0.6050; the value $k = 19$ with the $\ell_1$ norm resulted in the optimal area under the ROC curve of 0.5174.

According to the Holdout Method, the value $k = 7$ along with the $\ell_2$ norm, from $k$-NN, resulted in the highest accuracy of 0.5492; the value $k = 11$ and the $\ell_1$ norm scored the best F-Measure of 0.5764; and the value $k = 19$ and the $\ell_1$ norm resulted in the best area under the ROC curve of 0.4991. Figure \ref{rocknn} is a sample ROC curve from the Holdout Method with the $k$-NN classifier where $k = 19$ and the $\ell_1$ norm. Note that for $k = 1$, for both the Holdout Method and cross validation, the area under the ROC curve does not have much meaning since there would only be two pairs of true and false positive rates to estimate the area under the curve with.

\begin{table}
\begin{center}
\caption[Results for Approach 1 from 5-fold cross validation]{Accuracy, the F-Measure, and area under the ROC curve scores for the $k$-NN classifier on the Chromosome 1 SNP dataset, reduced to dimension $m' = 5000$, for difference values of $k$ under the $\ell_1$ and $\ell_2$ norms, from 5-fold cross validation.}
\label{foldfiveknn}
\vspace{3mm}

{\footnotesize
\centerline{\begin{tabular}{|c|cccccccccc|}
\hline
{\em Accuracy}&$k = 1$&$k = 3$&$k = 5$&$k = 7$&$k = 9$&$k = 11$&$k = 13$&$k = 15$&$k = 17$&$k = 19$\\
\hline
$\ell_1$ distance &0.5267     & 0.5342     & 0.5462     & 0.5472     & {\bf 0.5554}     & 0.5465     & 0.5467     & 0.5475     & 0.5490     & 0.5482\\
$\ell_2$ distance & 0.5309        & 0.5393        & 0.5372        & 0.5311 & 0.5431        & 0.5436        & {\bf 0.5444}        &  0.5444        & 0.5418        & 0.5424\\
\hline
{\em F-Measure}&$k = 1$&$k = 3$&$k = 5$&$k = 7$&$k = 9$&$k = 11$&$k = 13$&$k = 15$&$k = 17$&$k = 19$\\
\hline
$\ell_1$ distance & 0.5750     & 0.5823     & 0.5929     & 0.5985     & {\bf 0.6050}     & 0.5953     & 0.5988     & 0.6014     & 0.6032     & 0.6016\\
$\ell_2$ distance & 0.5756     & 0.5854     & 0.5833     & 0.5794     & 0.5922     & 0.5947     & 0.5974     & 0.5985     & 0.5968     & {\bf 0.6002}\\
\hline
{\em Area under ROC}&$k = 1$&$k = 3$&$k = 5$&$k = 7$&$k = 9$&$k = 11$&$k = 13$&$k = 15$&$k = 17$&$k = 19$\\
\hline
$\ell_1$ distance &-     & 0.3014     & 0.4094     & 0.4479     & 0.4765     & 0.4887     & 0.4931     & 0.5090     & 0.5122     & {\bf 0.5174} \\
$\ell_2$ distance &  -     & 0.2938     & 0.4016     & 0.4462     & 0.4668     & 0.4851     & 0.4990     & 0.5088     & 0.5125     & {\bf 0.5145}\\
\hline
\end{tabular}}}
\end{center}
\end{table}

\begin{figure}
\begin{center}
\includegraphics[scale = 0.45]{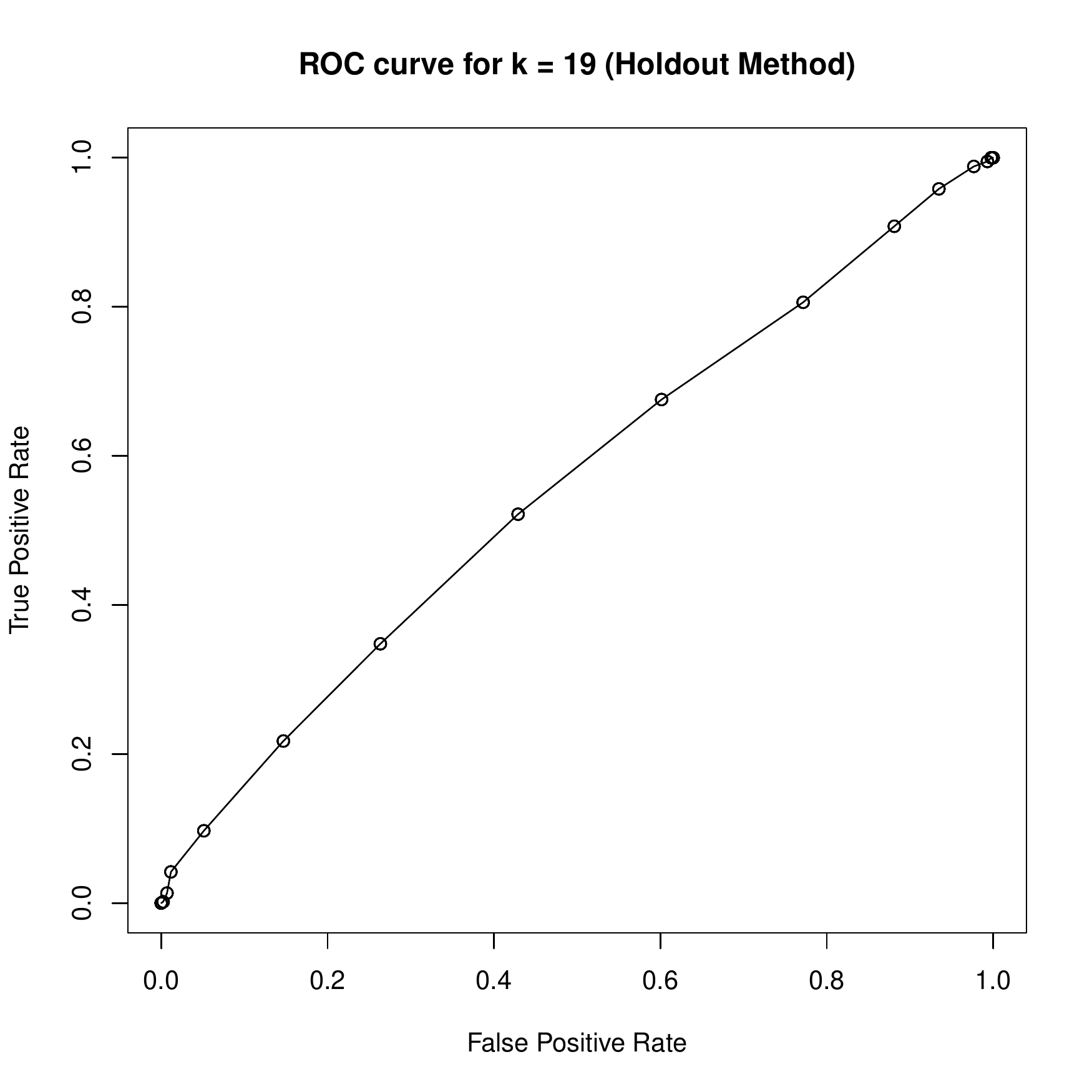}
\caption[Sample ROC curve from Approach 1]{Sample ROC curve, with an estimated area under the curve of 0.50, from the $k$-NN classifier, with $k = 19$ and the $\ell_1$ norm, after Random Projections using the Holdout Method.}
\label{rocknn}
\end{center}
\end{figure}

\begin{table}
\begin{center}
\caption[Results for Approach 1 from the Holdout Method]{Accuracy, the F-Measure, and area under the ROC curve scores for the $k$-NN classifier on the Chromosome 1 SNP dataset, reduced to dimension $m' = 5000$, for difference values of $k$ under the $\ell_1$ and $\ell_2$ norms, from the Holdout Method.}

\label{holdoutknn}
{\footnotesize
\centerline{\begin{tabular}{|c|cccccccccc|}
\hline
{\em Accuracy}&$k = 1$&$k = 3$&$k = 5$&$k = 7$&$k = 9$&$k = 11$&$k = 13$&$k = 15$&$k = 17$&$k = 19$\\
\hline
$\ell_1$ distance & 0.5063        & 0.5180        & 0.5219        & 0.5375& 0.5307        & {\bf 0.5463}        & 0.5433        & 0.5307        & 0.5287        & 0.5424\\
$\ell_2$ distance &0.5112     & 0.5122     & 0.5307     & {\bf 0.5492}     & 0.5346     & 0.5326     & 0.5287     & 0.5316     & 0.5307     & 0.5336\\
\hline
{\em F-Measure}&$k = 1$&$k = 3$&$k = 5$&$k = 7$&$k = 9$&$k = 11$&$k = 13$&$k = 15$&$k = 17$&$k = 19$\\
\hline
$\ell_1$ distance & 0.5445     & 0.5608     & 0.5651     & 0.5709     & 0.5562     & {\bf 0.5764}     & 0.5717     & 0.5642     & 0.5568     & 0.5704\\
$\ell_2$ distance & 0.5420     & 0.5391     & 0.5594     & {\bf 0.5756}     & 0.5582     & 0.5531     & 0.5493     & 0.5567     & 0.5529     & 0.5561\\
\hline
{\em Area under ROC}&$k = 1$&$k = 3$&$k = 5$&$k = 7$&$k = 9$&$k = 11$&$k = 13$&$k = 15$&$k = 17$&$k = 19$\\
\hline
$\ell_1$ distance &-     & 0.2363     & 0.3791     & 0.4306     & 0.4540     & 0.4636     & 0.4743     & 0.4746     & 0.4867     & {\bf 0.4991} \\
$\ell_2$ distance &  -     & 0.2424     & 0.3870     & 0.4450     & 0.4649     & 0.4720     & 0.4703     & 0.4779     & 0.4766     & {\bf 0.4975}\\
\hline
\end{tabular}}}
\end{center}
\end{table}

\section{MTD Feature Selection and Random Forest Results}\label{mtdrfsec}

For Approach 2, MTD Feature Selection and Random Forest were first applied to the dataset of 3907 observations and 73571 SNPs in Chromosome 1. The same dimensionality reduction and classification methods were then applied to the entire dataset of 3907 and 865688 SNPs across Chromosomes 1 to 15 and 17 to 22. Regarding the first experiment, the best results obtained were an accuracy of 0.6592, a F-Measure score of 0.6149, and an area under the ROC curve of 0.8392, where the threshold $\alpha = 0.2$ from MTD Feature Selection resulted in 287 important SNPs selected for classification by Random Forest. Table \ref{mtdrfchr1} records the predictive scores from this experiment with $\alpha = 0.2,0.3,0.4$, and Figure \ref{rocchr1comp} shows the ROC curves produced for this experiment at the three values of $\alpha$ (graph on the left) and compares these curves against the ROC curve from Approach 1, seen in Figure \ref{rocknn}, using the Holdout Method (graph on the right).

Involving the second experiment of Approach 2 with all 865688 SNPs, the best accuracy score of 0.6660 and area under the ROC curve of 0.8562 were obtained with $\alpha = 0.2$ and 3335 important SNPs selected for Random Forest, while the best F-Measure was obtained with $\alpha = 0.4$, resulting from 99 selected SNPs. Table \ref{mtdrfchrall} provides all the predictive performance measures of Random Forest for $\alpha = 0.2,0.3,0.4,0.5$. Figure \ref{rocallcomp} shows two plots of ROC curves: the plot on the left are the ROC curves from the second experiment of Approach 2 at the four various values of $\alpha$; the plot on the right consists of the three best ROC curves, obtained from the initial experiment of Approach 1 and the two experiments of Approach 2, based on the Holdout Method.

\begin{table}
\begin{center}
\caption[Results for Approach 2 using SNPs from Chromosome 1]{Accuracy, the F-Measure, and area under the ROC curve scores for the Random Forest classifier on the Chromosome 1 SNP dataset, for various values of the MTD Feature Selection threshold $\alpha$.}
\label{mtdrfchr1}

\vspace{4mm}

\centerline{\begin{tabular}{|c|ccc|}
\hline
 & $\alpha = 0.2$ & $\alpha = 0.3$ & $\alpha = 0.4$ \\
 \hline\hline
\# of total $\to$ important SNPs & $73571 \to 287$ & $73571 \to 62$ & $73571 \to 12$  \\
\hline\hline
{\em Accuracy} & {\bf 0.6592} & 0.6319 & 0.6338\\
\hline
{\em F-Measure} & {\bf 0.6149} & 0.5909 & 0.6008\\
\hline
{\em Area under ROC} & {\bf 0.8392} & 0.7739 & 0.7195\\
\hline
\end{tabular}}
\end{center}
\end{table}

\begin{figure}
\begin{center}
\centerline{\includegraphics[scale = 0.45]{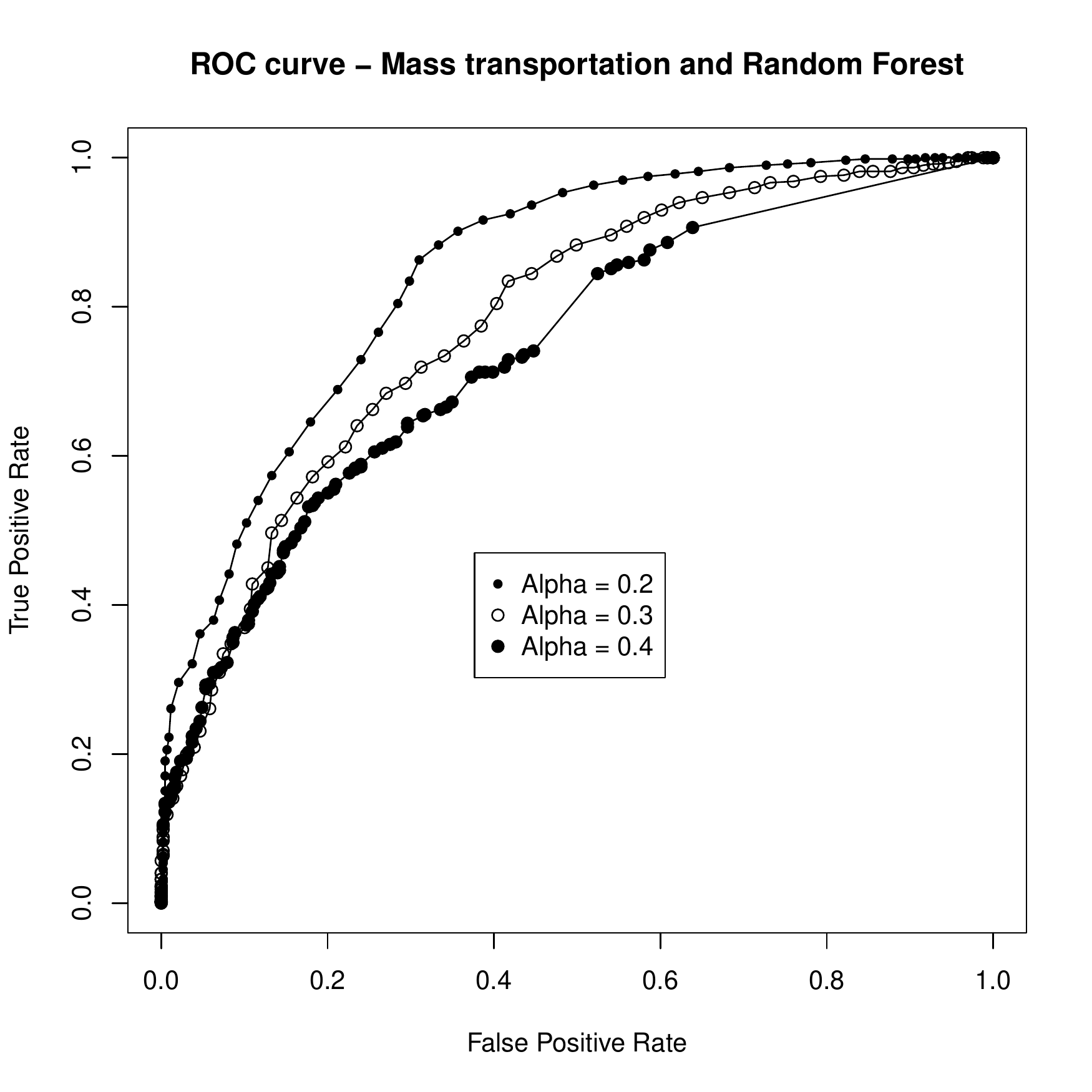}\includegraphics[scale=0.45]{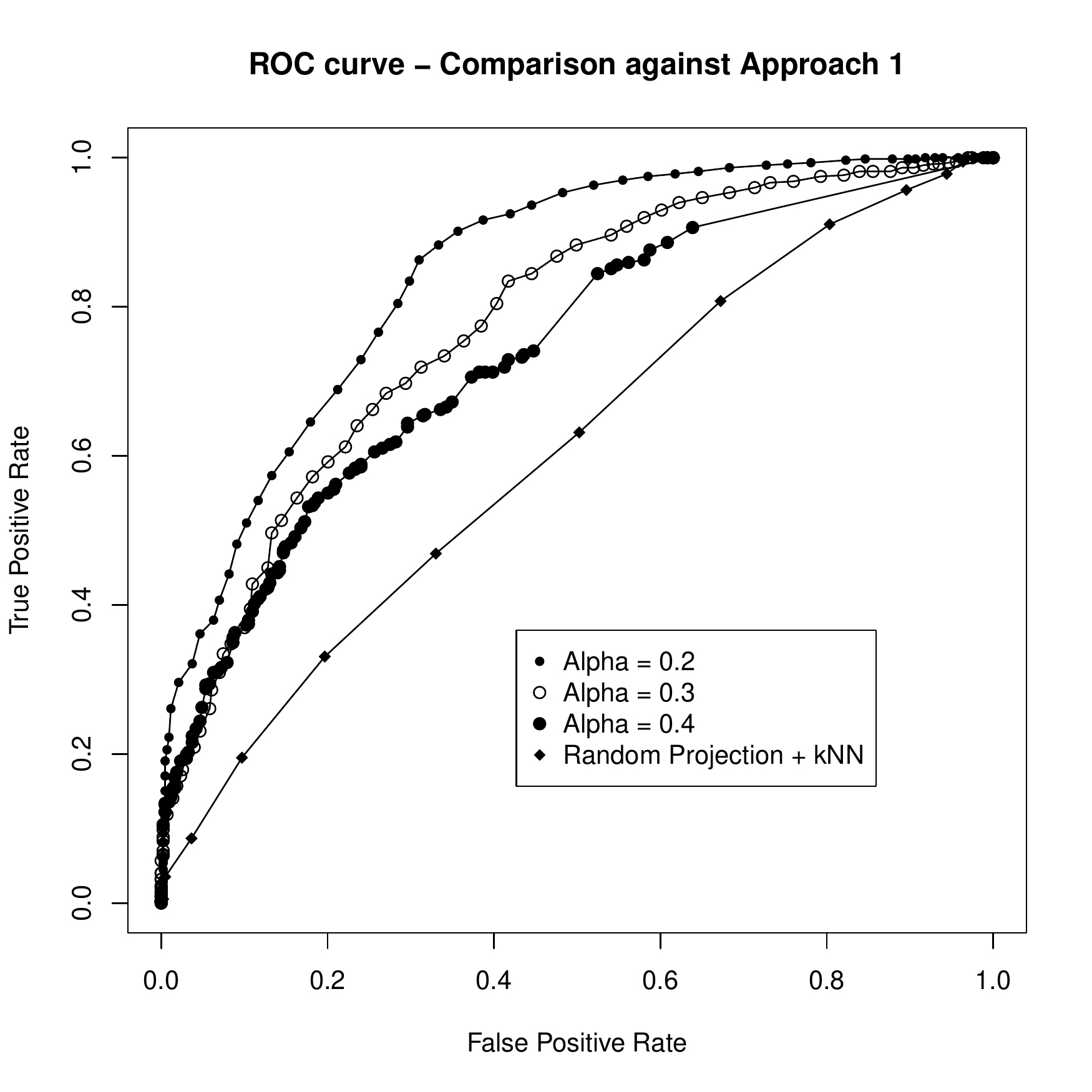}}
\end{center}
\caption[Comparison of ROC curves with SNPs from Chromosome 1]{The figure on the left shows the three ROC curves, based on the three values of $\alpha$, produced from Approach 2 with Random Forest using selected SNPs from Chromosome 1. The figure on the right compares these curves against the ROC curve from Approach 1, with the $k$-NN classifier for $k = 19$ and the $\ell_1$ norm.}
\label{rocchr1comp}
\end{figure}

\begin{table}
\begin{center}
\caption[Results for Approach 2 using SNPs from Chromosomes 1-15, 17-22]{Accuracy, the F-Measure, and area under the ROC curve scores for the Random Forest classifier on the Chromosomes 1 to 15 and 17 to 22 SNP dataset, for various values of the MTD Feature Selection threshold $\alpha$.}
\label{mtdrfchrall}

\vspace{4mm}

\centerline{\begin{tabular}{|c|cccc|}
\hline
 & $\alpha = 0.2$ & $\alpha = 0.3$ & $\alpha = 0.4$ & $\alpha = 0.5$\\
 \hline\hline
\# of total $\to$ important SNPs & $865688 \to 3335$ & $865688 \to 716$ & $865688 \to 99$ & $865688 \to 10$ \\
\hline\hline
{\em Accuracy} &{\bf 0.6660} & 0.6582 & 0.6475 & 0.6271\\
\hline
{\em F-Measure} &0.6202 & 0.6139 & {\bf 0.6253} & 0.5868\\
\hline
{\em Area under ROC} &{\bf 0.8562} & 0.8226 & 0.7767 &0.6137 \\
\hline
\end{tabular}}
\end{center}
\end{table}

\begin{figure}
\begin{center}
\centerline{\includegraphics[scale = 0.45]{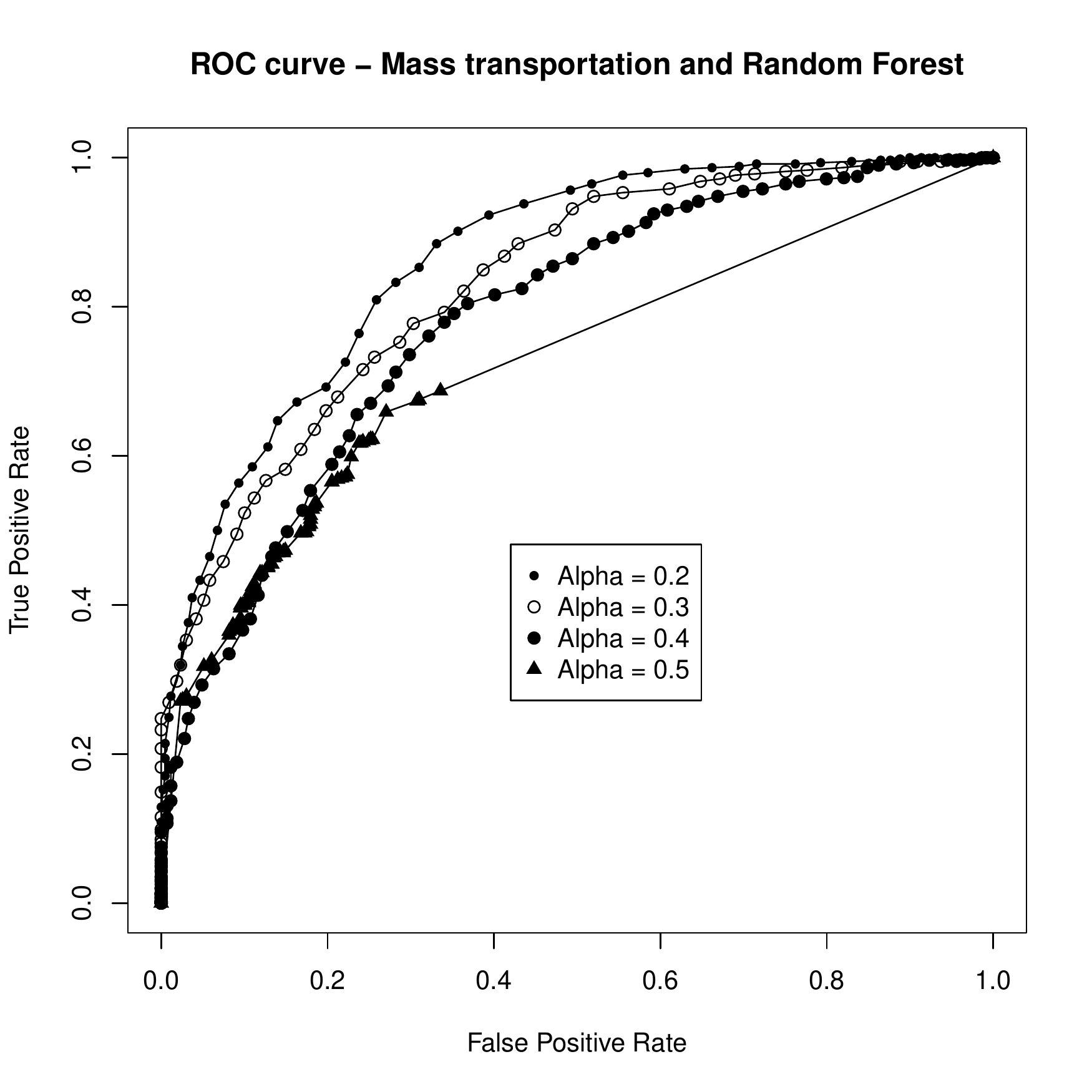}\includegraphics[scale=0.45]{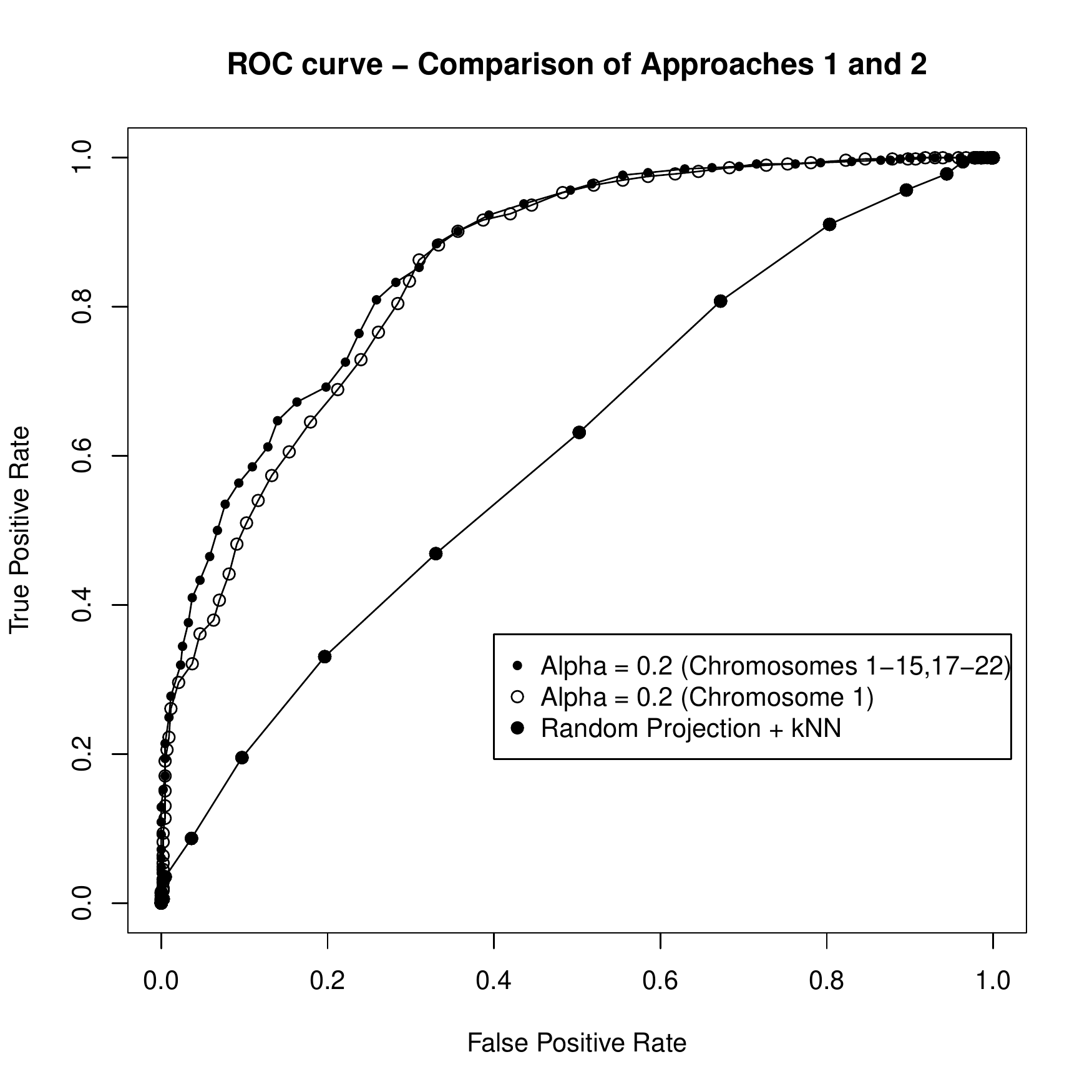}}
\end{center}
\caption[Best three ROC curves obtained]{The figure on the left shows the four ROC curves, based on the four values of $\alpha$, produced from Approach 2 with Random Forest using selected SNPs from Chromosomes 1 to 15 and 17 to 22. The figure on the right compares the three best ROC curves obtained for Approaches 1 and 2: the ROC curve with $k$-NN, for $k = 19$ and the $\ell_1$ norm, and the ROC curves with Random Forest with $\alpha = 0.2$ using selected SNPs from Chromosome 1 and then from Chromosomes 1 to 15 and 17 to 22.}
\label{rocallcomp}
\end{figure}

\section{Discussion of Predictive Results}\label{discsect}

This section discusses the predictive results from Approaches 1 and 2. It is clear that from Tables \ref{foldfiveknn} and \ref{holdoutknn} that the approach of using Random Projections to the lower dimension $m' = 5000$ and then applying $k$-NN was not successful in predicting coronary artery disease with the genetic dataset of 3907 observations with 73571 SNPs represented in $m = 220713$ dimensional space. The best accuracy achieved was only 0.5554 using $k = 9$ and the $\ell_1$ norm, while the highest area under the ROC was 0.5174 with $k = 19$ and also the $\ell_1$ norm, based on 5-fold cross validation. It is worthwhile to mention that the predictive scores were a bit higher from cross validation than from the Holdout Method, since due to the high-complexity of the genetic dataset, only one iteration of the Holdout Method was run.

On the other hand, based on the dataset of 73571 SNPs from the same observations, Approach 2 was able to obtain an accuracy of 0.6592 and an area under the ROC curve of 0.8392, with MTD Feature Selection selecting {\em only} 287 important SNPs for classification by Random Forest. This area under the ROC curve, obtained with only SNPs from Chromosome 1, is considerably higher than the previous high score of 0.608 from \cite{robbie12} on the same genetic dataset.

When all 865688 SNPs from Chromosomes 1 to 15 and 17 to 22 were considered, the best accuracy achieved with Approach 2 increased to 0.6660 and the best area under the ROC curve increased to 0.8562, with 3335 SNPs selected as important by MTD Feature Selection for classification. Even with merely 10 SNPs selected as important, from the threshold $\alpha = 0.5$, Random Forest was obtain to obtain an accuracy of 0.6271 and an area under the ROC curve of 0.6137, which is still better than the score in \cite{robbie12}.

In summary, Chapter \ref{resultschap} has covered the results obtained from experiments based on two approaches of predicting coronary artery disease using techniques from data science. The new MTD Feature Selection method along with the Random Forest classifier were able to obtain the highest predictive accuracies and areas under the ROC curve, which are all considerably higher than scores from the approach of using Random Projections with the $k$-Nearest Neighbour classifier or from any previously used classification algorithms. Chapter \ref{conclude} below concludes this thesis and provides some limitations of these experiments and directions for futures work.

\cleardoublepage

\chapter{Conclusion}\label{conclude}

In short, as its first main objective, this thesis has explained in detail two supervised learning algorithms, the $k$-Nearest Neighbour ($k$-NN) and Random Forest classifiers, and two dimensionality reduction methods, a well-known technique called Random Projections and a novel method termed Mass Transportation Distance (MTD) Feature Selection. The thesis provided a complete proof that the $k$-NN classifier is universally consistent, a highly desirable property for learning algorithms, in any finite dimensional normed vector space. This thesis also justified the use of the Mass Transportation Distance in supervised learning theory, since the distance is closely related to classification margins of 1-Lipschitz functions.

For the second objective, the thesis has compared the approach of applying Random Projections with $k$-NN against the approach of applying MTD Feature Selection and Random Forest on the prediction of coronary artery disease based on a high-dimensional genetic dataset, collected from the Ontario Heart Genomics Study (OHGS). These classification and reduction techniques were all applied on this dataset for the very first time. The comparative study demonstrated that MTD Feature Selection with Random Forest has considerably better predictive abilities than Random Projections with $k$-NN. With 3335 Single-Nucleotide Polymorphisms (SNPs), selected as important for classification by MTD Feature Selection, the Random Forest classifier was able to obtain an accuracy of 0.6660 and an area under the Receiver Operating Characteristic (ROC) curve of 0.8562, which is considerably better than the previous high area of 0.608 from \cite{robbie12}.

However, because this thesis is written from a data science, and not a genetic, perspective, it has two main limitations. First, although the training and evaluating split for MTD Feature Selection and Random Forest was done correctly (as MTD Feature Selection must only be applied to the training part to select important SNPs, and not on the entire genetic dataset), it is important in the genetics community to validate the predictive performance on a completely independent dataset, for instance against a similar genetic dataset from the Wellcome Trust Case-Control Consortium or Cleveland Clinic Foundation (see \cite{robbiethesis} for more details). MTD Feature Selection would ideally be used to select important SNPs from the OHGS genetic dataset and then Random Forest would train using these SNPs on the same dataset. The classifier would then predict labels for observations from an independent dataset using the same SNPs. A concern for this validation process, however, is that selected SNPs from the OHGS dataset may not appear in the independent dataset. Regardless, this is certainly an important direction for future work to improve this thesis.

Second, the concept of Quality Control (QC), which are statistical procedures to determine whether a SNP is of high enough quality to be included in a mathematical analysis, has been ignored in this thesis. Recall that, in the thesis, each SNP for an individual is assigned to the genotype with the highest probability, out of the three probabilities of possible genotypes, found in the raw OHGS dataset. Normally for QC however, a SNP is only assigned to a genotype if the highest probability is greater or equal to 0.9. Otherwise, it is assigned the value {\em no-call}. If a SNP has 10\% or more no-call values, it is removed from a study completely. The surviving SNPs would then be checked to ensure they satisfy the Hardy-Weinberg Equilibrium, a mathematical condition involving the SNP allele and genotype frequencies in a population; non-satisfactory SNPs are again removed. For more information on Quality Control and the Hardy-Weinberg Equilibrium, see \cite{robbiethesis} and \cite{robbie12}.

This thesis did not consider Quality Control to remove SNPs because out of the 865688 starting SNPs from the OHGS dataset, 299388 SNPs, or 34.6\%, would not pass QC. This number corresponds to a considerable amount of coordinates in the dataset to remove without question, especially as some of these removed SNPs may very well be related to coronary artery disease. Experiments using Quality Control, by removing the 299388 SNPs, have in fact demonstrated the MTD Feature Selection method with Random Forest, evaluated exactly as in Section \ref{approach2} with the Holdout Method, could obtain an area under the ROC curve of 0.6230 with 206 selected SNPs that pass QC. The score is much lower than the area of 0.8562 obtained without QC using the same techniques, yet is still higher than the previous high score of 0.608. This result is evidence that some SNPs useful for predicting coronary artery disease would be wrongfully removed as a result of QC. Since dimensionality reduction methods are run prior to classification, it is also not necessary to apply Quality Control to decrease the number of SNPs, or coordinates, for a classifier to run efficiently. 

In any case, Quality Control is an important topic to investigate for future research, especially when learning algorithms with dimensionality reduction methods can handle high-dimensional datasets that may not have the best data quality. The following outlines three additional directions for future research:

\begin{enumerate}
\item The raw OHGS dataset, with probabilities of the three possible genotypes for each SNP, has not been studied prior to this thesis. Although Random Projections and the $k$-NN classifier did not give good predictive results, further studies on the raw dataset with different classifiers, such as Support Vector Machines or Neural Networks, should be done. This dataset has the obvious advantage over the discretized genetic dataset, with assigned values at each SNP, of containing all genotype information from the DNA microarrays. 

\item Feature selection based on the Mass Transportation Distance (MTD) is an extremely promising technique and it should be researched much further. For the finite discrete domain $(\Omega,d)$, the Mass Transportation Distance between two measures simplifies to their $\ell_1$ distance. It is not known what the MTD simplifies to when the domain is a product of two finite discrete spaces: $(\Omega^2,d^2)$ where
\[
d^2[(\omega_1,\omega_2),(\omega_1',\omega_2')] = \mathrm{card}\{i \, | \,\omega_i \neq \omega_i'\}.
\]
Because paired interactions between SNPs are common, by knowing the simplification of the MTD for a product discrete space, one would be able to assign importance scores to, and select, pairs of SNPs for classification.

\item Further regarding the Mass Transportation Distance, as Theorems \ref{mainmtd} and \ref{newmtdthem} provide a theoretical relationship between this distance and the classification margin of 1-Lipschitz functions, it would be extremely interesting to continue developing the theory between the MTD and the predictive abilities of l-Lipschitz classifiers. A reference for studying these types of classifiers and classification margin is \cite{lipclass}.
\end{enumerate}

Despite some limitations, this thesis still provides a new data science perspective on the problem of predicting coronary artery disease with SNP information. It is the hope that by addressing the limitations of this thesis and the suggested future work, one will be able to further improve the predictive abilities of data science algorithms for classifying coronary artery disease. Perhaps one day, an accurate prediction on whether a patient walking into the hospital would have this disease, purely based on a blood sample, could be made.

%\include{intro}
%\cleardoublepage

%\include{discrete_dyn}

%\cleardoublepage

%\include{qualitative}
%\cleardoublepage

%\include{bifurcation}
%\cleardoublepage

%%%%%%%%%%%%%%%%%%%%%%%%%%%%%%%%%%%%%%%%%%%%%%%%%%%%%%%%%%%%%%%%%%%%%%
% If the following line of code is uncomment, then
% the following chapters will be numbered A, B, ...
%%%%%%%%%%%%%%%%%%%%%%%%%%%%%%%%%%%%%%%%%%%%%%%%%%%%%%%%%%%%%%%%%%%%%%
%\appendix

%\include{appendix_A}
%\cleardoublepage

%\include{appendix_B}
\cleardoublepage

%%%%%%%%%%%%%%%%%%%%%%%%%%%%%%%%%%%%%%%%%%%%%%%%%%%%%%%%%%%%%%%%%%%%%%
% We provide two methods to introduce your bibliography.
%
% The hand made bibliography:
% The basic method used the file biblio.tex.  It makes used of
% the standard LaTeX environment \begin{thebibliography}{} and
% \end{thebibliography}.
%
% The bibliography made with BibTeX:
% The second method used the file biblio.bib.  It makes used of
% BibTeX with the commands \bibliography{} and \bibliographystyle{}
% \bibliographystyle{} is defined in the preamble.
%
% For examples on how to use the command  \cite[]{} in the text
% to refer to items of the bibliography, look at the end of the
% section on the "Logistic Equation" in the source file
% qualitative.tex .  The results are dsiplayed at the end of
% Section 2.1 after compilation.
%
% Instead of \cite[]{}, one can use \citet[]{}  and  \citep[]{}.
% We have not illustrated how to use these commands but they are
% used like \cite[]{}.
%
% If BibTeX is used, the compilation of the file template.tex is done
% as follows.
% latex template       (more than ounce if necessary)
% bibtex template
% latex template
% latex template
% 
%%%%%%%%%%%%%%%%%%%%%%%%%%%%%%%%%%%%%%%%%%%%%%%%%%%%%%%%%%%%%%%%%%%%%%
\addcontentsline{toc}{chapter}{Bibliography}
%\bibBasic{biblio}      % Hand made bibliography
\bibTexCite{biblio}     % Bibliography made with BibTeX
                        % Include only the references in the bib file
                        % which are cited in the text.

%\bibTexNocite{biblio}   % Bibliography made with BibTeX
                        % Include all references in the bib file even
                        % if they are not cited in the text.

\end{document}